\documentclass[preprint,12pt]{elsarticle}

\usepackage{amssymb}
\usepackage{amsmath}

\usepackage{amsthm}
\usepackage{bm}
\usepackage{placeins}
\usepackage{graphicx}
\usepackage{epsfig}
\usepackage{epstopdf}
\usepackage{ragged2e}
\usepackage{url}
\usepackage{float}
\usepackage{stfloats}
\usepackage{subfigure}
\usepackage{algorithmic}
\usepackage{algorithm}
\usepackage{threeparttable}
\usepackage{lineno}
\usepackage{multirow}
\usepackage{mathrsfs}

\usepackage[pagebackref,breaklinks=true,colorlinks,citecolor=royalblue,bookmarks=false]{hyperref}
\usepackage{marvosym}

\newtheorem{theorem}{Theorem}[section]

\newtheorem{proposition}[theorem]{Proposition}

\begin{document}

\begin{frontmatter}

\title{Semi-Supervised Multi-Label Feature Selection with Consistent Sparse Graph Learning}
\author[inst1]{Yan~Zhong}
\author[inst2]{Xingyu~Wu}
\author[inst3]{Xinping~Zhao}
\author[inst4]{Li~Zhang}
\author[inst5]{Xinyuan~Song}
\author[inst6]{Lei~Shi}
\author[inst7]{Bingbing~Jiang}

\affiliation[inst1]{organization={School of Mathematical Sciences},
            addressline={Peking University}, 
            city={Beijing},
            postcode={100871}, 
            country={China}}

\affiliation[inst2]{organization={Department of Data Science and Artificial Intelligence},
            addressline={Hong Kong Polytechnic University}, 
            city={Hong Kong},
            postcode={999077},
            country={China}}

\affiliation[inst3]{organization={School of Computer Science and Technology},
            addressline={Harbin Institute of Technology (Shenzhen)}, 
            city={Shenzhen},
            postcode={518055},
            country={China}}

\affiliation[inst4]{organization={School of Computer Science and Technology},
            addressline={University of Science and Technology of China}, 
            city={Hefei},
            postcode={230027},
            country={China}}

\affiliation[inst5]{organization={Department of Computer Science and Engineering},
            addressline={University of Texas at Arlington}, 
            city={Arlington TX},
            postcode={76019}, 
            country={the USA}}

\affiliation[inst6]{organization={State Key Laboratory of Media Convergence and Communication},
            addressline={Communication University of China}, 
            city={Beijing},
            postcode={100024},
            country={China}}

\affiliation[inst7]{organization={School of Information Science and Technology},
            addressline={Hangzhou Normal University}, 
            city={Hangzhou},
            postcode={311121},
            country={China}}

\begin{abstract}
In practical domains, high-dimensional data are usually associated with diverse semantic labels, whereas traditional feature selection methods are designed for single-label data. Moreover, existing multi-label methods encounter \textbf{two main challenges in semi-supervised scenarios: (\romannumeral1)}. Most semi-supervised methods fail to evaluate the label correlations without enough labeled samples, which are the critical information of multi-label feature selection, making label-specific features discarded.
\textbf{(\romannumeral2)}. The similarity graph structure directly derived from the original feature space is suboptimal for multi-label problems in existing graph-based methods, leading to unreliable soft labels and degraded feature selection performance.
To overcome them, we propose a consistent sparse graph learning method for multi-label semi-supervised feature selection (SGMFS), which can enhance the feature selection performance by maintaining space consistency and learning label correlations in semi-supervised scenarios.
Specifically, for Challenge \textbf{(\romannumeral1)}, SGMFS learns a low-dimensional and independent label subspace from the projected features, which can compatibly cross multiple labels and effectively achieve the label correlations. For Challenge \textbf{(\romannumeral2)}, instead of constructing a fixed similarity graph for semi-supervised learning, SGMFS thoroughly explores the intrinsic structure of the data by performing sparse reconstruction of samples in both the label space and the learned subspace simultaneously. In this way, the similarity graph can be adaptively learned to maintain the consistency between label space and the learned subspace, which can promote propagating proper soft labels for unlabeled samples, facilitating the ultimate feature selection.
An effective solution with fast convergence is designed to optimize the objective function. Extensive experiments validate the superiority of SGMFS.
\end{abstract}

\begin{keyword}
Multi-label learning \sep semi-supervised feature selection \sep sparse graph learning \sep subspace learning \sep label correlation 
\end{keyword}

\end{frontmatter}


\section{Introduction}
\label{S1}
Recent years have witnessed a dramatic growth of data dimensionality in various domains, such as text categorization \cite{liu2017deep}, bioinformatics~\cite{wang2022feature} and image processing \cite{guo2022causal}. Except for associating with multiple labels, the high-dimensional data often contains redundant features and noisy dimensions.  As an effective dimension reduction technique, feature selection can identify discriminative features by removing irrelevant features without changing the original feature space, receiving extensive attention \cite{jiang2024multiview}. Due to the high cost of obtaining labeled data, there has been a growing focus on semi-supervised feature selection methods~\cite{tang2014feature}, which can utilize both unlabeled and labeled data to improve performance \cite{liu2023semifree}.

\begin{figure}[h]
	\centering
	\vspace{-0mm}
	\includegraphics[width=1\columnwidth]{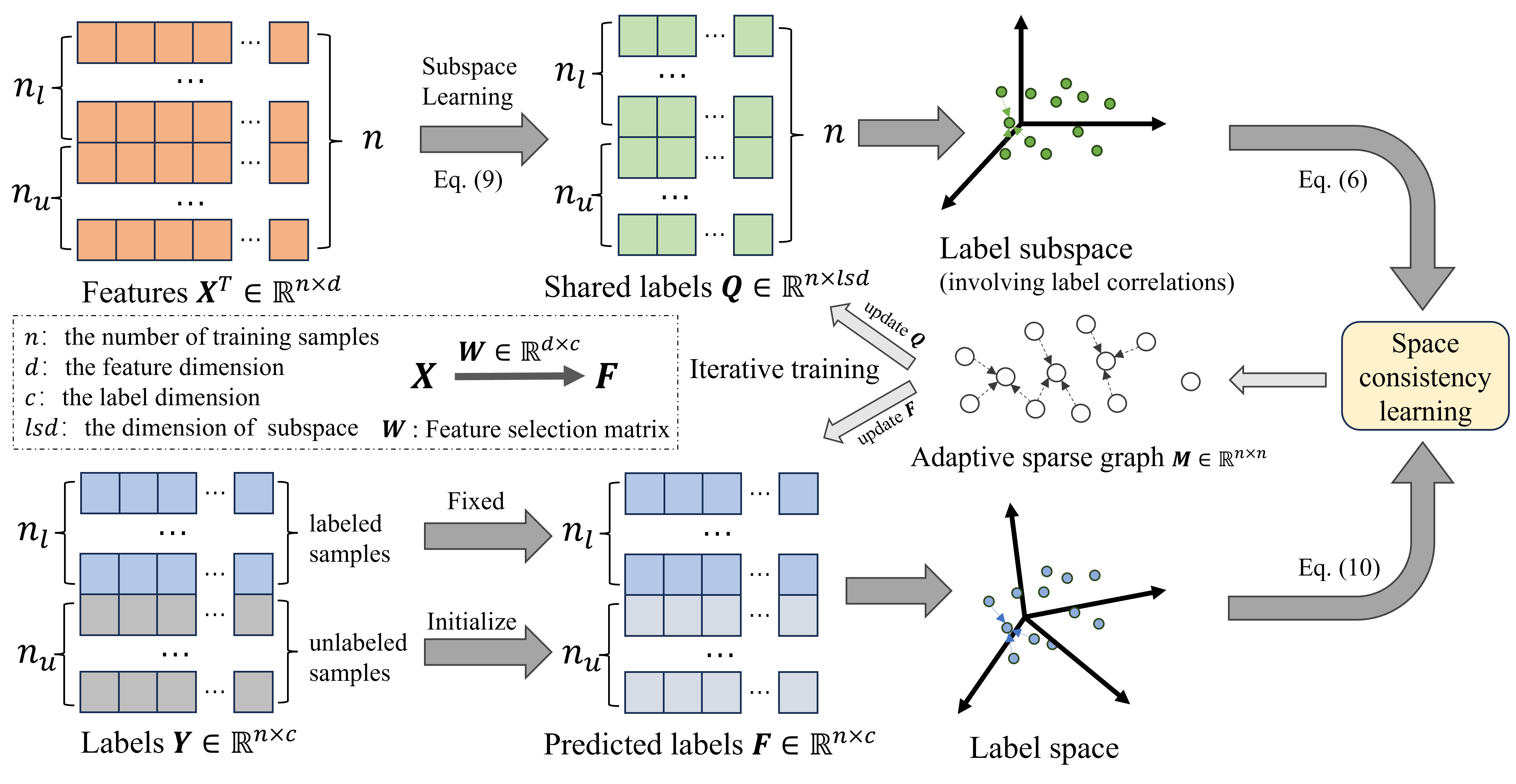}
	\vspace{-0mm}
	\caption{Overview of the consistent Sparse Graph learning for Multi-label semi-supervised Feature Selection (SGMFS): Firstly, the predicted labels ( soft-labels) $\textbf{F}$ and sparse graph matrix $\textbf{M}$ are initialized. Then the shared label subspace can be obtained by subspace learning. Subsequently, the sparse graph matrix $\textbf{M}$ is updated by space consistency learning between label subspace and original label space, and the updated $\textbf{M}$ is used to further update soft-labels $\textbf{F}$ and shared labels $\textbf{Q}$.
    After iterative training until convergence, the optimal feature weight matrix $\textbf{W}$ can be obtained for semi-supervised feature selection.
	}
	\label{fig1}
\end{figure}

Inspired by label propagation (LP) \cite{zhou2004learning}, 
many graph-based feature selection methods have been developed, which often use the pseudo-labels predicted by LP to make up for the lack of label information in the training stage and guide the process of semi-supervised feature selection. For example, Liu et al.~\cite{liu2013efficient} divided the LP and the importance evaluation of features into two independent steps and proposed a noise-insensitive Trace Ratio Criterion for semi-supervised Feature Selection (TRCFS).  Luo et al. \cite{luo2018semi} presented an insensitive sparse regression method (ISR) that directly employs the predicted labels to guide feature selection. Subsequently, a robust sparse feature selection method \cite{sheikhpour2020robust} is proposed based on non-convex $l_{2,p}$-norm to improve the robustness against outliers.  To learn the soft prediction labels, Chen et al. \cite{chen2018semi} imposed the sum-to-one constraint on them and ranked features by regression coefficients.
However, these methods are designed under the assumption that each sample should belong to one of the classes with a probability, limiting them to single-label tasks. In real-world applications, samples are usually associated with multiple semantic labels, thus some researchers began to explore the multi-label feature selection methods \cite{you2021online,shi2023unsupervised}. In this work, we focus on multi-label feature selection in semi-supervised scenarios.
\begin{figure}[t]
	\centering
	\vspace{-0mm}
	\includegraphics[width=0.75\columnwidth]{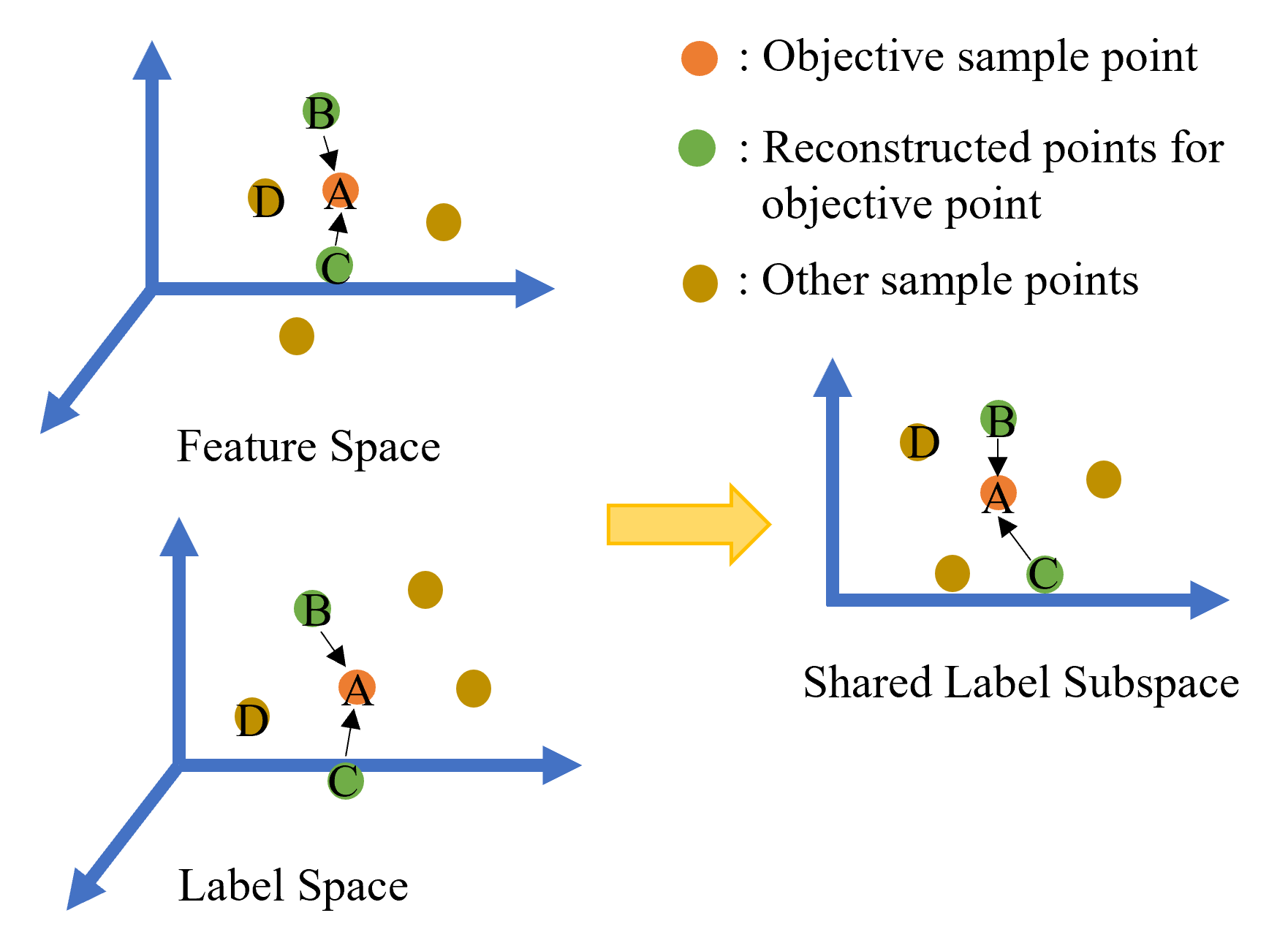}
	\vspace{-0mm}
	\caption{The process of sparse reconstruction of samples space consistency preservation in SGMFS: (i) The shared label subspace is generated by original feature and label spaces. (ii) The most suitable samples $B$ and $C$ are adaptively selected by sparse learning to reconstruct sample $A$. Note that the reconstructed sample points for the same objective sample point in label space and subspace are consistent.
    }
	\label{fig2}
\end{figure}

To deal with multi-label data, one of the most intuitive ways is to directly extend existing single-label feature selection models. For instance, Alalga et al.\cite{alalga2016soft} proposed a soft-constrained Laplacian score to rank features for multi-label data. To properly explore the label correlations, the sparse regression model integrated with feature-label correlation is designed in \cite{lin2023multi}. Thereafter, Yang et al.\cite{yang2023multi} and Zhong et al.~\cite{zhong2021multi} combined global and local label correlations to learn the label-specific features. To learn the structural correlations of labels, the multi-label model via adaptive spectral graph is presented to fit the relation between feature space and label distribution in~\cite{fan2024learning}. In the semi-supervised scenario, Chang et al.~\cite{chang2014semi} and Wang et al.~\cite{wang2017semi} combined multi-label feature selection and semi-supervised learning into a unified framework to learn the low-rank representation of feature weight matrix. Shi et al.~\cite{shi2021binary} introduced label learning with binary hashing to the semi-supervised feature selection. Recently,~\cite{yi2024sfs} proposed a semi-supervised feature selection based on an adaptive graph with global and local constraints, which can maintain the data structures within the selected feature subset. 
However,  existing semi-supervised multi-label feature selection methods still suffer from the following two main challenges: \textbf{(\romannumeral1)} It is difficult to adequately consider label correlations\footnote{Label correlations provide critical information for selecting discriminative features in the multi-label scenario~\cite{zhang2023multi}.}, since the label information is often incomplete in semi-supervised scenarios. \textbf{(\romannumeral2)} 
In LP-based methods, the graph structure derived from the original feature space is untrainable and suboptimal, as the sum-to-one constraint \cite{zhang2013pairwise} in the LP process is unsuitable for multi-label problems. Additionally, these semi-supervised methods overlook the consistency\footnote{The process of space consistency learning enables the dissemination of appropriate labels to unlabeled samples, thereby facilitating feature selection~\cite{xu2018semi}.} of manifold structures between the original label space and the shared label subspace, which degrades the performance of feature selection.

To address these problems, a novel consistent Sparse Graph learning model for Multi-label semi-supervised Feature Selection (SGMFS) is proposed, the overview of which is summarized in Fig.~\ref{fig1}. Specifically speaking, for Challenge \textbf{(\romannumeral1)}, 
we first map the feature information into label space and then the obtained label information is mapped into a low-dimensional reduced space, which is termed the shared label subspace. Since each dimension in this subspace is independent and serves as an integration of original partial related labels, the correlations among the original multiple labels can be effectively captured within the shared label subspace.
For Challenge \textbf{(\romannumeral2)}, instead of constructing a fixed similarity graph for LP-based semi-supervised learning, we proposed to jointly learn an adaptive weight matrix in the original multi-label space and the shared label subspace with the sparse regularization based on $l_1$-norm. This process is performed by sparse reconstruction of samples in both the label space and the shared subspace simultaneously, which can preserve the consistency among these spaces via minimizing the reconstruction error (as shown in Fig.~\ref{fig2}), and can generate reliable soft labels for semi-supervised feature selection. To learn a more accurate sparse graph, the nonnegative and zero-diagonal constraints are used in the objective function of SGMFS. In this way, the proper neighbors for sample reconstruction can be automatically determined by sparse learning, which can effectively improve the robustness of SGMFS. Finally, more reliable features for multi-label learning can be selected by fully considering the intrinsic data structure and label correlations.
The main contributions of our proposed framework in this paper are as follows:
\begin{enumerate}
       \item In this paper, we propose a new multi-label semi-supervised feature selection method named SGMFS. The core goal is to embed label correlation learning and space consistency learning into the process of semi-supervised feature selection, which can generate reliable soft labels for unlabeled samples and promote effective feature selection.
       \item We propose a novel shared subspace learning method to explore label correlations in a semi-supervised setting, enabling effective consideration of label correlations even when label information is incomplete in multi-label datasets, thereby facilitating robust feature selection.
       \item This paper proposes a novel adaptive sparse graph learning method that minimizes sample reconstruction error while maintaining consistency in both the label space and shared label subspace. This approach ensures that the learned soft labels and graph structure are jointly optimized for multi-label feature selection.
      \item  A fast iterative optimization strategy is developed to solve the non-smooth objective function with rigorous theoretical guarantees. Extensive experiments are conducted to validate the effectiveness and superiority of the proposed SGMFS on benchmarks from diverse domains.
\end{enumerate}

The rest of this paper is organized as follows. The related work is reviewed in Section~\ref{S2}, and Section~\ref{S3} presents the SGMFS framework in detail. The optimization procedures are provided in Section~\ref{S4}, then  Section~\ref{S5} validate the advantages of SGMFS for multi-label feature selection. Finally, the conclusion is drawn in Section~\ref{S6}.

\section{Related Work}
\label{S2}

This section briefly reviews the researches related to our work, i.e., multi-label feature selection and graph learning.

\subsection{Multi-label Feature Selection}

In real-world scenarios, multi-label feature selection serves as an effective dimensionality reduction technique to eliminate redundant and noisy features from data with multiple semantic labels. This process enhances performance while reducing computational costs~\cite{zhou2024multi,yuan2024multi}. Existing methods for multi-label feature selection fall into three categories: filter \cite{duda2006pattern}, embedded \cite{jian2016multi}, and wrapper \cite{zhang2009feature} approaches. Among these, embedded methods have gained increasing attention for their ability to integrate feature selection into the learning process, leading to improved performance \cite{tang2014feature, LIAO2024103727}. For instance, MIFS~\cite{jian2016multi} selects features by decomposing multi-label information into a low-dimensional space, while FSNM~\cite{nie2010efficient} employs joint $l_{2,1}$-norm minimization on both the loss function and regularization for feature selection. Huang \emph{et al.}\cite{huang2017joint} proposed a method to learn shared and label-specific features while building a multi-label classifier. Similarly, Ma \emph{et al.}\cite{ma2023discriminative} introduced a discriminative feature selection approach that captures high-order data structures using adaptive graphs.

In practice, label information is often incomplete. While unsupervised feature selection methods operate without label guidance, ground-truth labels remain crucial for model effectiveness. This has led to the development of semi-supervised feature selection methods that leverage both labeled and unlabeled data \cite{sheikhpour2017survey, LI2024103633}. Examples include insensitive sparse regression \cite{luo2018semi} and re-scaled linear square regression \cite{chen2018semi}, which evaluate feature importance by learning feature projections. However, single-label approaches are not directly applicable to multi-label problems, especially when high-dimensional data is associated with multiple labels, as labeling all samples manually can be prohibitively expensive.

To address semi-supervised multi-label feature selection, several methods have been proposed, including CSFS~\cite{chang2014convex}, SFSS~\cite{ma2012discriminating}, and SCFS~\cite{xu2018semi}. Additionally, Chang et al.\cite{chang2014semi} and Wang et al.\cite{wang2017semi} integrated multi-label feature selection with semi-supervised learning to derive low-rank representations of feature weight matrices. Shi et al.\cite{shi2021binary} proposed SFS-BLL, which applies a binary hash constraint on the label propagation process to handle multi-label data. More recently, Yi et al.\cite{yi2024sfs} introduced a semi-supervised multi-label method that learns a similarity graph with global and local constraints to capture data relationships in the selected feature subspace. Similarly, Tan et al.~\cite{tan2024partial} explored compact embedding for partial multi-label learning under semi-supervised conditions, improving efficiency via shared embedding spaces. Despite these advancements, challenges arising from insufficient label information remain. In contrast, this paper focuses on simultaneously minimizing the sample reconstruction error through sparse embedding in both label space and shared label subspace, ensuring that the learned labels and graph are jointly optimized for semi-supervised multi-label feature selection.

\subsection{Graph Learning}
Graph learning is a valuable technique in data mining. Recently, many researchers have focused on learning the graph weights between data samples, since the neighbor structure of unlabeled samples can be explored by graph with better performance.
There are various methods for constructing the graph weight matrix $\textbf{M}$, with the kernel-based $k$ nearest neighbor ($k$NN) \cite{belkin2003laplacian} being a common approach. Let $\textbf{X}=[\textbf{x}_{1},\textbf{x}_{2},...,\textbf{x}_{n}]\in \mathbb{R}^{d \times n}$ represent the feature matrix. Sample $\textbf{x}_{i}$ is connected to sample $\textbf{x}_{j}$ if $\textbf{x}_{i}$ is among the $k$ nearest neighbors of $\textbf{x}_{j}$ in the $k$NN graph.
And the $k$NN graph $\textbf{M}$ is defined as:
\begin{equation}
	m_{i j}=\left\{\begin{array}{l}
		\exp \left(-\frac{\left\|\textbf{x}_{i}-\textbf{x}_{j}\right\|^{2}}{\sigma^{2}}\right), \text { if } \textbf{x}_{i} \text { is linked to } \textbf{x}_{j}  \\
		0, \text { otherwise }
	\end{array}\right.
	\label{eq1}
\end{equation}
Note that $\textbf{x}_{j}$ may not be among the $k$ nearest neighbors of $\textbf{x}_{i}$ when $\textbf{x}_{i}$ is that of $\textbf{x}_{j}$. In existing methods, for the sake of obtaining a symmetrical graph matrix $\tilde{\textbf{M}}$, $\textbf{M}$ can be rewritten as $\tilde{\textbf{M}}=(\textbf{M}+\textbf{M}^{T})/2$. Although the $k$NN graphs are concise and explicit, 
selecting the appropriate kernel parameter $\sigma$ for various datasets is challenging.
Therefore, Nie et al. \cite{nie2019structured} proposed an adaptive $k$NN graph, formulated as:
\begin{equation}
	\begin{aligned}
		\min _{\textbf{M}}\sum_{i,j=1}^{n}\left(\left\|\textbf{x}_{i}-\textbf{x}_{j}\right\|_{2}^{2} m_{i j}+\alpha m_{i j}^{2}\right) \\
		\text { s.t. } \quad \forall i, \sum_{j=1}^{n}m_{i j}=1,0 \leq m_{i j} \leq 1.
	\end{aligned}
	\label{eq2}
\end{equation}
Compared with the kernel-based graph construction method, Eq.~(\ref{eq2}) learns a graph based on the distances between samples in the original feature space. To determine the parameter $\alpha$, we need to manually tune the numbers of the nearest neighbors for each sample. Moreover, solving  Eq.~(\ref{eq2}) involves the optimization of multiple subproblems, requiring higher computational costs.

On the other hand, weight graphs learned by sparsity representation are parameter-free, which can adaptively optimize the nearest neighbors to reconstruct each sample based on the $l_{1}$-norm sparse embedding. Inspired by the dictionary learning \cite{dornaika2019sparse}, sparse graph learning is designed as follows:
\begin{equation}
	\min_{\textbf{M}} \left\| \textbf{M}\right\|_{1}+\|{\textbf{E}}\|_{1}, \text { s.t. } \textbf{X}=\textbf{X}\textbf{M}+\textbf{E}
	\label{eq3}
\end{equation}
where $\textbf{E}$ is the reconstruction error matrix and the $l_1$-norm can achieve sparsity. $\textbf{M}$ in Eq.~(\ref{eq3}) can be considered as a graph matrix since each element in $\textbf{M}$ describes the distance between two samples in the feature space. 

Compared with Eq.~(\ref{eq2}), this way of graph learning in Eq.~(\ref{eq3}) is parameterless and effective. Therefore, inspired by this, we integrate semi-supervised feature selection and sparse graph learning into one framework, which is applicable to large multi-label datasets with effectiveness and robustness.

\section{The proposed algorithm}
\label{S3}
In this section, we first provide Preliminaries and Model Formulation in Section~\ref{S3_1}, followed by our solutions to Challenge \textbf{(\romannumeral2)} in Section~\ref{S3_2} and Challenge \textbf{(\romannumeral1)} in Section~\ref{S3_3}, respectively.

\subsection{Preliminaries and Model Formulation}
\label{S3_1}
As mentioned above, $\textbf{X}=[\textbf{x}_{1},\textbf{x}_{2},...,\textbf{x}_{n}]\in \mathbb{R}^{d \times n}$ denotes the feature matrix of training data, where $n$ and $d$ are the number of samples and features respectively. $\textbf{Y}_{l}=[\textbf{y}_{1};\textbf{y}_{2};...;\textbf{y}_{n_{l}}]\in \mathbb{R}^{n_{l} \times c}$ is the labeled matrix and $c$ is the number of class labels. $\textbf{F}=\left[\begin{array}{l}
	\textbf{F}_{l} \\
	\textbf{F}_{u}
\end{array}\right] \in \mathbb{R}^{n \times c}$ denotes the predicted label matrix. where $\textbf{F}_{l}\in \mathbb{R}^{n_{l} \times c}$ and $\textbf{F}_{u}\in \mathbb{R}^{n_{u} \times c}$, $n=n_{u}+n_{l}$. Inspired by a semi-supervised learning framework CSFS \cite{chang2014convex}, we have the constraint $\textbf{F}_{l}=\textbf{Y}_{l}$ for all the labeled samples in our proposed model, and the soft multi-label matrix $\textbf{F}_{u}$ is initialized to null matrix $\mathbf{0}$. $\textbf{I}$ is the unit matrix with size of $lsd\times lsd$, where $lsd$ is the dimension of label subspace. And $\textbf{Q}=[\textbf{q}_{1};\textbf{q}_{2};...;\textbf{q}_{n}]\in \mathbb{R}^{n \times lsd}$ denotes the shared labels in label subspace.

Let $\textbf{W}\in\mathbb{R}^{d\times c}$ denote the feature weight matrix of all the sample sets, including both labeled and unlabeled samples, which can be estimated by the following framework for semi-supervised feature selection:
\begin{equation}
	\begin{aligned}
		\min _{\textbf{W}, b, \textbf{F}_{l}=\textbf{Y}_{l}}\left\|\textbf{X}^{T} \textbf{W}+\mathbf{1}_{n} \textbf{b}^{T}-\textbf{F}\right\|_{F}^{2}+\gamma\|\textbf{W}\|_{2,1}\\
		\text { s.t. } \quad 0\leq \textbf{F}_{i,j} \leq 1\quad\quad\quad\quad\quad
	\end{aligned}
	\label{eq4}
\end{equation}
where $\|\cdot\|_{F}$ and $\|\cdot\|_{2,1}$ denote the Frobenius and $\ell_{2,1}$ norm of matrix respectively, $\textbf{b} \in \mathbb{R}^{c \times 1}$ is the bias term, and $\mathbf{1}_{n}$ is the column vector with all elements being 1. The first term of Eq.~(\ref{eq4}) is the loss term and the second is the sparse regularization term with parameter $\gamma>0$. And the constraint  $0\leq \textbf{F}_{i,j} \leq 1$ is imposed to avoid improper prediction values of labels. 

\subsection{Adaptive Sparse Graph Learning}
\label{S3_2}
Although the framework of Eq.~(\ref{eq4}) can effectively utilize the tiny proportion of labeled samples to learn a soft multi-label matrix $\textbf{F}$ for feature selection, the space consistency\footnote{Space Consistency: Nearby data points in feature space possess the high probabilities to have the similar labels} and structure consistency\footnote{Structure Consistency: Data points on the same cluster or manifold in feature space possess the high probabilities to have the similar labels.} of unlabeled sample sets are completely ignored, which may lead to limited performance of soft multi-label matrix learning, and impact the reliability of selected features. 

To address this issue, some LP-based methods try to construct a fixed similarity graph for space consistency maintaining and semi-supervised learning, the graph structure derived from the original feature space is suboptimal.
Therefore, we propose a sparse graph learning method inspired by graph-based learning \cite{belkin2003laplacian,nie2019structured}. This method can effectively explore the space consistency and structure consistency of sample points, which is a key aspect of semi-supervised learning problems \cite{zhou2004learning,jiang2022semi,tang2021cross}. 
However, inaccurately considering the inner graph structure of data points leads to unreliable soft label prediction in semi-supervised learning, especially for multi-label datasets. Therefore, instead of using a traditional manifold regularizer based on specifying prior knowledge, we minimize the reconstruction error of samples in both feature space and label space to learn the joint optimal adaptive weight matrix and soft multi-labels. Furthermore, to enhance feature selection performance, we integrate the feature weight matrix into the original feature space. This transformation enables the information contained in high-dimensional feature data to be represented in a lower-dimensional manifold. Then the weight matrix $\textbf{M}$ and soft labels $\textbf{F}$ can be learned by:
\begin{equation}
	\begin{aligned}
		\min _{\textbf{M},\textbf{F}}\left\|\textbf{M} \textbf{X}^{T} \textbf{W}-\textbf{X}^{T} \textbf{W}\right\|_{F}^{2}+\|\textbf{M} \textbf{F}-\textbf{F}\|_{F}^{2}+\|\textbf{M}\|_{F}^{2}\\
		\text { s.t. } \quad 0\leq \textbf{F}_{i,j} \leq 1\quad\quad\quad\quad\quad\quad\quad
	\end{aligned}
	\label{eq5}
\end{equation}
where the third item is the regularization item which is imposed to avoid the trivial solution $\textbf{M}=\textbf{I}_{n\times n}$. However, Frobenius norm-based regularization cannot learn the optimal neighbors for each sample.

Inspired by dictionary learning \cite{dornaika2019sparse}, we replace the Frobenius norm of $\textbf{M}$ with LASSO type minimization to obtain the sparsity of $\textbf{M}$. Additionally, the reconstruction weights $\textbf{M}_{i,j}$ are constrained to be non-negative and symmetric, meanwhile we keep diagonal elements $\textbf{M}_{i,i}=0\;(i=1, \ldots, n)$ in our method, which can adaptively learn the sparse reconstruction neighbors of each data point. Therefore,  $\textbf{M}$ and $\textbf{F}$ are estimated by an adaptive sparse-graph learning process:
\begin{equation}
	\begin{aligned}
		\min _{\textbf{M},\textbf{F}}\left\|\textbf{M} \textbf{X}^{T} \textbf{W}-\textbf{X}^{T} \textbf{W}\right\|_{F}^{2}+\|\textbf{M} \textbf{F}-\textbf{F}\|_{F}^{2}+\gamma\|\textbf{M}\|_{1}\\
		\text { s.t. } \quad 0\leq \textbf{F}_{i,j} \leq 1,\textbf{M}_{i,i}=0,\textbf{M}_{i,j}=\textbf{M}_{j,i}>0 \quad
	\end{aligned}
	\label{eq6}
\end{equation}
where $\|\textbf{M}\|_{1}=\sum_{i, j}\left|\textbf{M}_{i, j}\right|$ and $\gamma$ is a positive parameter that controls the sparsity degree of the reconstruction graph weight matrix $\textbf{M}$.
\subsection{Label Correlations learning}
\label{S3_3}
In multi-label feature selection problems, label correlations provide critical information for selecting discriminative features for multi-label learning. However, existing multi-label feature selection methods cannot effectively consider label correlations due to the lack of label information in the semi-supervised learning scenario.
According to shared feature subspace learning \cite{wang2017semi,ma2012web}, 
there is a low-dimensional label subspace shared by all the original labels, where the shared labels are uncorrelated and fully independent in multi-label learning problems\cite{chang2014semi}. 
Therefore, we propose a label correlations learning method based on subspace learning in the case of missing label information in semi-supervised scenarios.

Concretely speaking, define $\textbf{x}$ denote the original feature vector of a sample data, which pertains to $\textbf{c}$ concepts indicated by the label vector $\textbf{y}\in \mathbb{R}^{c}$, the prediction function $\left\{f_{t}\right\}_{t=1}^{c}$ can be formulated as follows:
\begin{equation}
	f_{t}(\textbf{x})=\textbf{v}_{t}^{T} \textbf{x}+\textbf{p}_ {t}^{T}\textbf{q}
	\label{eq7}
\end{equation}
where $\textbf{p}_{t}$ and $\textbf{v}_{t}$ are weight vectors, and $\textbf{q}$ is the shared label of data $\textbf{x}$ in label subspace.

Assume there are $n$ training samples $\left\{\textbf{x}_{i},\textbf{y}_{i}\right\}_{i=1}^{n}$, in accordance with embedding shared feature subspace uncovering in objective function and defining the mapping process $\textbf{X}^{T}\textbf{W}= \textbf{X}^{T}\textbf{V}+\textbf{Q}\textbf{P}$, where $\textbf{V}=\left[\textbf{v}_{1}, \textbf{v}_{2}, \cdots, \textbf{v}_{c}\right] \in \mathbb{R}^{d \times c}$, $\textbf{P}=\left[\textbf{p}_{1}, \textbf{p}_{2}, \cdots, \textbf{p}_{c}\right] \in \mathbb{R}^{lsd \times c}$, $\textbf{Q} \in \mathbb{R}^{n \times lsd}$  and $lsd$ is the dimension of label subspace, we can consider the label correlations by learning the label subspace with the following framework:
\begin{equation}
	\begin{aligned}
		\min _{f_{t},\textbf{p}_{t},\textbf{v}_{t}} \frac{1}{n} \sum_{i=1}^{n} \operatorname{loss}\left(\textbf{V}^{T} \textbf{x}_i+\textbf{P}^{T}\textbf{q}_i, \textbf{y}_{i}\right)\!+\!\mu \Omega\left(\left\{\textbf{V},\textbf{P}\right\}\right)
		\\\text{s.t.}\quad \textbf{Q}^{T} \textbf{Q}=\textbf{I}\quad\quad\quad\quad\quad\quad\quad\quad\quad
	\end{aligned}
	\label{eq8}
\end{equation}
where $\operatorname{loss}(\cdot)$  and $\mu \Omega(\cdot)$ are the loss function and regularization with $\mu$ as its parameter respectively, and the constraint $\textbf{Q}^{T}\textbf{Q}=\textbf{I}$ is used to make Eq.~(\ref{eq8}) tractable and reduce the impact of label correlations.

By minimizing the distance of label information predicted by original features and the shared labels of samples, we apply the least square loss in the second term of Eq.~(\ref{eq8}), which can arrive at:
\begin{equation}
	\begin{aligned}
		\min _{\textbf{Q},\textbf{P}} \|\textbf{X}^{T}\textbf{W}-\textbf{Q} \textbf{P}\|_{F}^{2},\quad
		\text{s.t.}\quad \textbf{Q}^{T}\textbf{Q}=\textbf{I}\quad\quad\quad\quad
	\end{aligned}
	\label{eq9}
\end{equation}
Therefore, label correlations can be considered to select more discriminating features for multi-label datasets by the process of label subspace learning, which can implicitly encode the sample labels with correlations into the shared labels which are independent of each other.

After label subspace learning, space and structure consistencies learning process, Eq.~(\ref{eq6}) can be improved to:
\begin{equation}
	\begin{aligned}
		\min _{\textbf{M},\textbf{F}}\|\textbf{M} \textbf{F}-\textbf{F}\|_{F}^{2}
		+\|\textbf{M} \textbf{Q}-\textbf{Q}\|_{F}^{2}+\gamma\|\textbf{M}\|_{1}\\
		\text { s.t. } \quad 0\leq \textbf{F}_{i,j} \leq 1,\textbf{M}_{i,i}=0,\textbf{M}_{i,j}=\textbf{M}_{j,i}>0
	\end{aligned}
	\label{eq10}
\end{equation}
\subsection{Objective Function of SGMFS}
Eventually, the  optimization problem of semi-supervised Multi-label Feature Selection with Sparse Graph learning (SGMFS) can be formulated as:
\begin{equation}
	\begin{aligned}
		\min _{\textbf{M}, \textbf{b}, \textbf{F},\textbf{W},\textbf{P},\textbf{Q}}\left\|\textbf{X}^{T} \textbf{W}\!+\mathbf{1}_{n} \textbf{b}^{T}\!-\!\textbf{F}\right\|_{F}^{2}+\alpha\|\textbf{X}^{T}\textbf{W}-\textbf{Q} \textbf{P}\|_{F}^{2}\\ +\beta\left(\|\textbf{M} \textbf{F}-\textbf{F}\|_{F}^{2}+\|\textbf{M} \textbf{Q}-\textbf{Q}\|_{F}^{2}\right)+\gamma\Omega\left(\textbf{W},\textbf{M}\right)\\
		\text { s.t. }  0\leq \textbf{F}_{i,j} \leq 1,\textbf{F}_{l}=\textbf{Y}_{l},\textbf{Q}^{T} \textbf{Q}=\textbf{I},\quad\quad\quad\\
		\textbf{M}_{i,i}=0,\textbf{M}_{j,i}=\textbf{M}_{i,j}>0\quad\quad\quad\quad
	\end{aligned}
	\label{eq11}
\end{equation}
where $\Omega\left(\textbf{W},\textbf{M}\right)=\|\textbf{W}\|_{2,1}+\|\textbf{M}\|_{1}$, $\alpha$ is the positive tolerance parameter and $\beta$ and $\gamma$ are space consistency parameters and sparse regularization parameters, respectively. The objective function of SGMFS Eq.~(\ref{eq11}) is composed of semi-supervised framework Eq.~(\ref{eq4}), the exploration of label correlations Eq.~(\ref{eq9}) and adaptive sparse-graph learning process Eq.~(\ref{eq10}). Therefore, SGMFS aims to enhance feature selection performance by maintaining space consistency and learning label correlations in semi-supervised scenarios.

\section{Optimization}
\label{S4}
On account of the non-smoothness of $l_{2,1}$-norm and $l_{1}$-norm regularization terms, there are many constraints of the proposed objective function Eq.~(\ref{eq11}), it is difficult to obtain the optimal solution. Thus, we propose an efficient and fast convergent algorithm for SGMFS.
\subsection{Learning by Alternating Minimization}
In order to illustrate the joint convexity of the optimization problem of Eq.~(\ref{eq11}) with respect to $\textbf{M}, \textbf{b}, \textbf{F},\textbf{W},\textbf{P}$ and $\textbf{Q}$, we have the following proposition:
\begin{proposition}\label{prop1}
Denote $\textbf{W} \in \mathbb{R}^{d \times c}, \textbf{b} \in \mathbb{R}^{c \times 1},\textbf{F} \in \mathbb{R}^{n \times c},\textbf{M} \in \mathbb{R}^{n \times n},\textbf{P} \in \mathbb{R}^{lsd \times c}$ and $\textbf{Q} \in \mathbb{R}^{n \times lsd}$,
$g\left(\textbf{M}, \textbf{b}, \textbf{F},\textbf{W},\textbf{Q},\textbf{P}\right)=\left\|\textbf{X}^{T} \textbf{W}+\mathbf{1}_{n} \textbf{b}^{T}-\textbf{F}\right\|_{F}^{2}+\alpha\|\textbf{X}^{T}\textbf{W}-\textbf{Q} \textbf{P}\|_{F}^{2}+\gamma\|\textbf{W}\|_{2,1}
+\gamma\|\textbf{M}\|_{1}+\beta\|\textbf{M} \textbf{F}-\textbf{F}\|_{F}^{2}+\beta\|\textbf{M} \textbf{Q}-\textbf{Q}\|_{F}^{2}$.
The minimization of $g\left(\textbf{M}, \textbf{b}, \textbf{F},\textbf{W},\textbf{Q},\textbf{P}\right)$ is jointly convex with respect to $\textbf{M}, \textbf{b}, \textbf{F},\textbf{W},\textbf{P}$ and $\textbf{Q}$.
\end{proposition}
\begin{proof}
	According to the definition of Frobenius norm, $g\left(\textbf{M}, \textbf{b}, \textbf{F},\textbf{W},\textbf{Q},\textbf{P}\right)$ is equal to:\\
	$Tr\left(\left(\textbf{X}^{T} \textbf{W}+\mathbf{1}_{n} \textbf{b}^{T}-\textbf{F}\right)^{T}\left(\textbf{X}^{T} \textbf{W}+\mathbf{1}_{n} \textbf{b}^{T}-\textbf{F}\right)\right)+\gamma\|\textbf{W}\|_{2,1}+\\
	\beta Tr\left(\left(\textbf{M} \textbf{F}-\textbf{F}\right)^{T}\left(\textbf{M} \textbf{F}-\textbf{F}\right)\right)+\beta Tr\left(\left(\textbf{M} \textbf{Q}-\textbf{Q}\right)^{T}\left(\textbf{M} \textbf{Q}-\textbf{Q}\right)\right)\\
	+\alpha Tr\left(\left(\textbf{X}^{T}\textbf{W}-\textbf{Q} \textbf{P}\right)^{T}\left(\textbf{X}^{T}\textbf{W}-\textbf{Q} \textbf{P}\right)\right)+\gamma\|\textbf{M}\|_{1}$~, where each item including $\|\textbf{W}\|_{2,1}$ and $\|\textbf{M}\|_{1}$ is positive semi-definite, so the sum $g\left(\textbf{M}, \textbf{b}, \textbf{F},\textbf{W},\textbf{Q},\textbf{P}\right)$ is jointly convex.
\end{proof}

Therefore, the objective function of SGMFS is minimized by alternate iteration in this paper.

First, according to setting $\textbf{M},\textbf{b},\textbf{F},\textbf{W},\textbf{Q}$ fixed and the derivative of Eq.~(\ref{eq11}) w.r.t. $\textbf{P}$ equal to zero, we have:
\begin{equation}
	2\left(\textbf{Q}^{T}\textbf{X}^{T}\textbf{W}-\textbf{Q}^{T} \textbf{Q} \textbf{P}\right)=0 \quad \!\!\Rightarrow\!\! \quad \textbf{P}=\textbf{Q}^{T}\textbf{X}^{T} \textbf{W}
	\label{eq12}
\end{equation}

Then, we substitute $\textbf{P}$ in Eq.~(\ref{eq11}) with Eq.~(\ref{eq12}) and make $\textbf{M},\textbf{b},\textbf{F},\textbf{W},\textbf{P}$ fixed to optimize $\textbf{Q}$, the constrained minimization problem of Eq.~(\ref{eq11}) w.r.t. $\textbf{Q}$ is equal to:
\begin{equation}
	\begin{aligned}
		\min _{\textbf{Q}}\|\textbf{X}^{T}\textbf{W}-\textbf{Q} \textbf{Q}^{T}\textbf{X}^{T} \textbf{W}\|_{F}^{2}\!+\beta\|\textbf{M} \textbf{Q}-\textbf{Q}\|_{F}^{2}\\
		\text { s.t. } \textbf{Q}^{T} \textbf{Q}=\textbf{I}\quad\quad\quad\quad\quad\quad\
	\end{aligned}
	\label{eq13}
\end{equation}
Note that $\left(\textbf{I}-\textbf{Q} \textbf{Q}^{T}\right)\left(\textbf{I}-\textbf{Q} \textbf{Q}^{T}\right)=\textbf{I}-\textbf{Q} \textbf{Q}^{T}$, then Eq.~(\ref{eq13}) can be rewritten as a spectral clustering optimization problem. The derivation goes as follows:
\begin{equation}
	\begin{aligned}
		\text {Eq.}~(\ref{eq13}) \Leftrightarrow \min _{\textbf{Q}}\left\|\left(\textbf{I}-\textbf{Q} \textbf{Q}^{T}\right) \textbf{X}^{T} \textbf{W}\right\|_{F}+\beta\|(\textbf{M}-\textbf{I}) \textbf{Q}\|_{F} \\
		\Leftrightarrow \min _{\textbf{Q}} \operatorname{Tr}\left(\textbf{W}^{T} \textbf{X}\left(\textbf{I}-\textbf{Q} \textbf{Q}^{T}\right)\left(\textbf{I}-\textbf{Q} \textbf{Q}^{T}\right) \textbf{X}^{T} \textbf{W}\right)\\+\beta \operatorname{Tr}\left(\textbf{Q}^{T}(\textbf{M}-\textbf{I})^{T}(\textbf{M}-\textbf{I}) \textbf{Q}\right)\quad\quad\quad\quad\quad \\
		\Leftrightarrow \min _{\textbf{Q}}-\operatorname{Tr}\left(\textbf{Q}^{T} \textbf{X}^{T} \textbf{W} \textbf{W}^{T} \textbf{X} \textbf{Q}\right)\quad\quad\quad\quad\quad\quad\quad\quad\\+\beta \operatorname{Tr}\left(\textbf{Q}^{T}(\textbf{M}-\textbf{I})^{T}(\textbf{M}-\textbf{I}) \textbf{Q}\right)\quad\quad\quad\quad\quad \\
		\Leftrightarrow \max _{\textbf{Q}} \operatorname{Tr}\left(\textbf{Q}^{T} \textbf{C} \textbf{Q}\right)\quad\quad\quad\quad\quad\quad\quad\quad\quad\quad\quad\quad\quad
		\\
		\text { s.t. } \textbf{Q}^{T} \textbf{Q}=\textbf{I}\quad\quad\quad\quad\quad\quad\quad\quad\quad\quad\quad\quad\quad\quad
	\end{aligned}
	\label{eq14}
\end{equation}
where $\textbf{C}=\textbf{X}^{T} \textbf{W} \textbf{W}^{T} \textbf{X}-\beta(\textbf{M}-\textbf{I})^{T}(\textbf{M}-\textbf{I})$.

\begin{proposition}\label{prop2}
The optimal solution of $\textbf{Q}$ in Eq.~(\ref{eq14}) can be obtained by the eigen-decomposition of $\textbf{C}$, and the optimal $\textbf{Q}$ is the matrix that is composed of the eigenvectors corresponding to the largest $lsd$ eigenvalues of $\textbf{C}$.
\end{proposition}
\begin{proof}
	Since $\textbf{Q}$ satisfies the property of $\textbf{Q}^{T} \textbf{Q}=\textbf{I}$, then $\left(\textbf{Q}^{T}\textbf{IQ}\right)=\textbf{I}$, the optimization problem Eq.~(\ref{eq14}) for $\textbf{Q}$ is equivalent to:\\
	\begin{equation}
		\begin{aligned}
			\max _{\textbf{Q}} \operatorname{Tr}\left(\left(\textbf{Q}^{T}\textbf{GQ}\right)^{-1}\textbf{Q}^{T}\textbf{RQ}\right)\quad\quad	\text { s.t. } \textbf{Q}^{T} \textbf{Q}=\textbf{I}\quad
		\end{aligned}
		\label{eq15}
	\end{equation}
	where $\textbf{G} = \textbf{I}, \textbf{R} = \textbf{C}$, According to \cite{chang2014semi} and \cite{ma2012web}, the optimal $\textbf{Q}$ in  Eq.~(\ref{eq15}) can be obtained by conducting the eigen-decomposition of $\textbf{G}^{-1}\textbf{R}$. Since $\textbf{G}^{-1}\textbf{R}=\textbf{C}$, the optimal $\textbf{Q}$ is equal to the matrix composed of the eigenvectors corresponding to the largest $lsd$ eigenvalues of $\textbf{C}$.
\end{proof}

By setting the derivative of Eq.~(\ref{eq11}) w.r.t. $\textbf{b}$ to 0, we have:
\begin{equation}
	\textbf{b}=\frac{1}{n}\left(\textbf{F}^{T} \mathbf{1}_{n}-\textbf{W}^{T} \textbf{X} \mathbf{1}_{n}\right)
	\label{eq16}
\end{equation}

To overcome the non-smoothness of $l_{2,1}$ norm, we rewrite $\|\textbf{W}\|_{2,1}$ in Eq.~(\ref{eq11}) as $ Tr\left(\textbf{W}^{T}\textbf{DW}\right)$, where $\textbf{D}$ is a dialogue matrix with $n \times n$ size, and its dialogue elements are $\frac{1}{2\left\|\textbf{W}_{i,:}\right\|_{2}},i=1,\dots,n$. Moreover, according to substituting Eq.~(\ref{eq16}) into Eq.~(\ref{eq11}) and fixing  $\textbf{M},\textbf{F},\textbf{Q}$, the optimal problem for $\textbf{W}$ can be written as:
\begin{equation}
	\begin{aligned}
		\min _{\textbf{W}}Tr\left(\left(\textbf{HX}^{T} \textbf{W}\!\!-\!\textbf{HF}\right)^{T}\!\!\left(\textbf{HX}^{T} \textbf{W}\!\!-\!\textbf{HF}\right)\!\right)\!+\!\gamma Tr\left(\textbf{W}^{T}\textbf{DW}\right)\!\\+\alpha Tr\left(\left(\textbf{X}^{T}\textbf{W}\!-\!\textbf{QQ}^{T}\textbf{X}^{T} \textbf{W}\right)^{T}\!\left(\textbf{X}^{T}\textbf{W}\!-\!\textbf{QQ}^{T}\textbf{X}^{T} \textbf{W}\right)\right)
	\end{aligned}
	\label{eq17}
\end{equation}
where $\textbf{H}=\textbf{I}-\frac{1}{n}\mathbf{1}_{n}^{T}\mathbf{1}_{n}$.

Note that $\textbf{M}=\textbf{M}^{T}$ and $\textbf{H}=\textbf{H}^{T}=\textbf{H}^{2}$, then we set the derivative of Eq.~(\ref{eq17}) w.r.t. $\textbf{W}$ equal to zero, to obtain the optimal $\textbf{W}$:
\begin{equation}
	\begin{array}{l}
		2\left(\textbf{S}-\alpha \textbf{XQQ}^{T} \textbf{X}^{T}\right) \textbf{W}=2 \textbf{XHF} \\
		\Rightarrow \textbf{W}=\left(\textbf{S}-\alpha \textbf{XQQ}^{T} \textbf{X}^{T}\right)^{-1} \textbf{XHF}\quad\quad
	\end{array}
	\label{eq18}
\end{equation}
where $\textbf{S}=\textbf{XHX}^{T}+\gamma \textbf{D}+\alpha \textbf{X} \textbf{X}^{T}$.

According to setting the derivative of Eq.~(\ref{eq11}) w.r.t. $\textbf{F}$ equal to zero, we predict the soft multi-label matrix $\textbf{F}$ with $\textbf{M}$ and the updated $\textbf{W}$ and $\textbf{b}$ in Eq.~(\ref{eq18}) and Eq.~(\ref{eq16}) respectively. Since there are constraints of $ \quad 0\leq \textbf{F}_{i,j} \leq 1$ and $\textbf{F}_{l}=\textbf{Y}_{l}$ in objective function Eq.~(\ref{eq11}), $\textbf{F}$ is limited to the range 0-1 inspired by CSFS \cite{chang2014convex}:
\begin{equation}
	\begin{array}{l}
		\textbf{F}=\textbf{K}^{-1}\left(\textbf{X}^{T} \textbf{W}+\mathbf{1}_{n} \textbf{b}^{T}\right) \\
		\textbf{F}_{i, j}=\left\{\begin{array}{ll}
			\textbf{Y}_{i, j}, & \text { if } i \leq n_{l} \\
			0, & \text { if } \textbf{F}_{i, j} \leq 0 \text { and } i>n_{l} \\
			\textbf{F}_{i, j}, & \text { if } 0 \leq \textbf{F}_{i, j} \leq 1 \text { and } i>n_{l} \\
			1, & \text { if } 1 \leq \textbf{F}_{i, j} \text { and } i>n_{l}
		\end{array}\right.
	\end{array}
	\label{eq19}
\end{equation}
where $\textbf{K} = \textbf{I}+\beta\left(\textbf{M}-\textbf{I}\right)^{T}\left(\textbf{M}-\textbf{I}\right)$.
\begin{algorithm}[H]
	\caption{
Optimization procedure of the proposed SGMFS}
	\label{Alg1}
	\begin{algorithmic}[1] 
		\REQUIRE
		The dimension number of label subspace $lsd$;Parameters $\alpha$, $\beta$, $\gamma$; \\
		\quad \quad Training data feature matrix $\textbf{X} \in \mathbb{R}^{d \times n}$;\\
		\quad \quad Label matrix of partly labeled data $\textbf{Y}\in \mathbb{R}^{n_{l} \times c}$.
		\ENSURE The selected features that are top-ranked.
		\STATE Randomly initialize $\textbf{W}_{1}\in\mathbb{R}^{d \times c}$ and initialize $\textbf{M}_{1}\in\mathbb{R}^{n \times n}$ by the symmetric version of Eq.~(\ref{eq1}) ($\textbf{M}_{1}$ satisfies $[\textbf{M}_{1}]_{i,i}=0,[\textbf{M}_{1}]_{i,j}=[\textbf{M}_{1}]_{j,i}>0$). \\Set $\textbf{F}_{1}=\left[\begin{array}{l}
			\textbf{Y} \\
			\mathbf{0}
		\end{array}\right] \in \mathbb{R}^{n \times c}$ and $t=1$;
		\STATE \textbf{Repeat:}
		\label{code:fram:Repeat}
		\STATE \quad Compute the dialogue matrix $\textbf{D}_{t+1}$ by $\textbf{d}_{t}^{i}=\frac{1}{2\left\|[\textbf{W}_{t}]_{i,:}\right\|_{2}}, i=1,\ldots,d$;
		\label{code:fram:update_D}
		\STATE \quad Compute $\textbf{C}_{t+1}=\textbf{X}^{T} \textbf{W}_{t} \textbf{W}_{t}^{T} \textbf{X}-\beta(\textbf{M}_{t}-\textbf{I})^{T}(\textbf{M}_{t}-\textbf{I})$;
		\label{code:fram:derivation_W}
		\STATE \quad Obtain $\textbf{Q}_{t+1}$ according to the eigen-decomposition of $\textbf{C}_{t+1}$ ($\textbf{Q}_{t+1}$ is constituted by the eigenvectors of the largest $lsd$ eigenvalues);
		\label{ code:fram:derivation_APQ}
		\STATE \quad Compute $\textbf{S}_{t}=\textbf{XHX}^{T}+\gamma \textbf{D}_{t}+\alpha \textbf{XX}^{T}$ ;
		\label{code:fram:stepsizes}
		\STATE \quad Compute $\textbf{W}_{t+1}=\left(\textbf{S}_{t}-\alpha \textbf{XQ}_{t} \textbf{Q}_{t}^{T} \textbf{X}^{T}\right)^{-1} \textbf{XHF}_{t}$;
		\label{code:fram:update}
		\STATE \quad Compute $\textbf{b}_{t+1}=\frac{1}{n}\left(\textbf{F}_{t}^{T} \mathbf{1}_{n}-\textbf{W}_{t}^{T} X \mathbf{1}_{n}\right)$;
		\label{code:fram:update}
		\STATE \quad Compute $\textbf{K} = \textbf{I}+\beta\left(\textbf{M}-\textbf{I}\right)^{T}\left(\textbf{M}-\textbf{I}\right)$;
		\label{code:fram:update}
		\STATE \quad Compute $\tilde{\textbf{F}}_{t+1}=\textbf{K}^{-1}\left(\textbf{X}^{T} \textbf{W}_{t+1}+\mathbf{1}_{n} \textbf{b}_{t+1}^{T}\right)$;
		\label{code:fram:update}
		\STATE \quad Compute $\textbf{F}_{t+1}$ according to:\\
		$\left[\textbf{F}_{t+1}\right]_{i, j}=\!\! \left\{\begin{array}{ll}
			\textbf{Y}_{i, j},  \quad\quad \text { if } i \leq n_{l} \\
			0,  \quad\quad\quad\text { if } \left[\tilde{\textbf{F}}_{t+1}\right]_{i, j} \leq 0 \text { and } i>n_{l} \\
			\!\!\left[\tilde{\textbf{F}}_{t+1}\right]_{i, j}, \!\! \text { if } 0 \leq \left[\tilde{\textbf{F}}_{t+1}\right]_{i, j} \leq 1 \text { and } i>n_{l} \\
			1,  \quad\quad\quad\text { if } 1 \leq \left[\tilde{\textbf{F}}_{t+1}\right]_{i, j} \text { and } i>n_{l}
		\end{array}\right.$
		\label{code:fram:update}
		\STATE \quad Compute $\textbf{A}_{t+1}^{+}=\textbf{F}_{t+1}\textbf{F}_{t+1}^{T}+\left(\textbf{Q}_{t+1}\textbf{Q}_{t+1}^{T}\right)^{+}$,$\textbf{B}_{t+1}^{+}=\textbf{A}_{t+1}^{+}$, \\
		\vspace{2mm} \! \!$\textbf{A} _{t+1}^{-}=\left(\textbf{Q}_{t+1}\textbf{Q}_{t+1}^{T}\right)^{-}$, $\textbf{B}_{t+1}^{-}=\left(\textbf{Q}_{t+1}\textbf{Q}_{t+1}^{T}\right)^{-}+\frac{\gamma}{2\beta} \textbf{E}$;
		\label{code:fram:update}
		\STATE \quad Update $\textbf{M}_{t+1}$ according to:\\
		$\quad \left[\textbf{M}_{t+1}\right]_{i, j}=\left[\textbf{M}_{t}\right]_{i, j} \sqrt{\frac{\left(\textbf{M}_{t} \textbf{A}_{t+1}^{-}+\textbf{A}_{t+1}^{-} \textbf{M}_{t}\right)_{i, j}+2 \left[\textbf{B}_{t+1}\right]_{i, j}^{+}}{\left(\textbf{M}_{t} \textbf{A}_{t+1}^{+}+\textbf{A}_{t+1}^{+} \textbf{M}_{t}\right)_{i, j}+2 \left[\textbf{B}_{t+1}\right]_{i, j}^{-}}}\quad\quad\quad\quad\quad$
		\label{code:fram:update}
		\STATE \quad Compute $t=t+1$;
		\STATE \textbf{Until} convergence;
		\label{code:fram:convergence}
		\RETURN Feature weight matrix $\textbf{W}^{*}$, soft multi-label matrix $\textbf{F}^{*}$ and sparse graph matrix $\textbf{M}^{*}$; 
		\label{code:fram:return}
		\STATE Rank features by $\left\|\textbf{W}^{*}_{i,:}\right\|_{2},i=1,\ldots,d$ in a descending order and return top-ranked features.
		\label{code:fram:convergence}
	\end{algorithmic}
\end{algorithm}

Finally, to optimize the sparse graph matrix $\textbf{M}$, Eq.~(\ref{eq11}) can be rewritten as the following sparse optimization problem with $\textbf{Q},\textbf{W},\textbf{F}$ fixed:
\begin{equation}
	\begin{aligned}
		\min _{\textbf{M}}\quad\gamma\|\textbf{M}\|_{1}\!\!+	\beta Tr\left(\left(\textbf{M} \textbf{F}-\textbf{F}\right)^{T}\left(\textbf{M} \textbf{F}-\textbf{F}\right)\right)\\+\beta Tr\left(\left(\textbf{M} \textbf{Q}-\textbf{Q}\right)^{T}\left(\textbf{M} \textbf{Q}-\textbf{Q}\right)\right)\quad\quad\quad\\
		\text { s.t. } \textbf{M}_{i,i}=0,\textbf{M}_{j,i}=\textbf{M}_{i,j}>0\quad\quad\quad\quad\quad
	\end{aligned}
	\label{eq20}
\end{equation}
The objective function in Eq.~(\ref{eq20}) can be rewritten as:
\begin{equation}
	\begin{aligned}
		J(\textbf{M})=& Tr\left\{\textbf{M}\left[\beta\left(\textbf{FF}^{T}+\textbf{QQ}^{T}\right)\right]\textbf{M}^{T}\right.\\
		&\left.-2\left[\beta\left(\textbf{FF}^{T}+\textbf{QQ}^{T}\right)-\frac{\gamma}{2} \textbf{E}\right] \textbf{M}^{T}\right.\\
		&\left.+\left[\textbf{F}^{T}\textbf{F}+\textbf{I}\right]\right\}
	\end{aligned}
	\label{eq21}
\end{equation}
where $\textbf{E}=\mathbf{1}_{n}\mathbf{1}_{n}^{T}$. For simplicity and following the method in \cite{chen2014similarity}, we can rewrite Eq.~(\ref{eq21}) as:
\begin{equation}
	\begin{aligned}
		J(\textbf{M})=& Tr\left(\textbf{MAM}^{T}-2\textbf{BM}^{T}\right)\\
		=& Tr\left(\textbf{MA}^{+}\textbf{M}^{T}-\textbf{MA}^{-}\textbf{M}^{T}-2\textbf{B}^{+}\textbf{M}^{T}+2\textbf{B}^{-}\textbf{M}^{T}\right)
	\end{aligned}
	\label{eq22}
\end{equation}
where $\textbf{A}=\textbf{FF}^{T}+\textbf{QQ}^{T},\textbf{B}=\textbf{FF}^{T}+\textbf{QQ}^{T}-\frac{\gamma}{2\beta} \textbf{E}$. $\textbf{A}^{+},\textbf{A}^{-},\textbf{B}^{+},\textbf{B}^{-}$ are two pair matrices satisfying $\textbf{A}=\textbf{A}^{+}-\textbf{A}^{-},\textbf{B}=\textbf{B}^{+}-\textbf{B}^{-}$, and all the elements in them are nonnegative. Therefore, we can set $\textbf{A}_{i j}^{+}=\left(\left|\textbf{A}_{i j}\right|+\textbf{A}_{i j}\right) / 2+\textbf{Z}_{i j}, \textbf{A}_{i j}^{-}=\left(\left|\textbf{A}_{i j}\right|-\textbf{A}_{i j}\right) / 2+\textbf{Z}_{i j}$ for any nonnegative entries $\textbf{Z}_{i,j}$, which is the same with $\textbf{B}$.

In our method, we set $\textbf{A}^{+}=\textbf{B}^{+}=\textbf{FF}^{T}+\left(\textbf{QQ}^{T}\right)^{+},\textbf{A}^{-}=\left(\textbf{QQ}^{T}\right)^{-}$ and $\textbf{B}^{-}=\left(\textbf{QQ}^{T}\right)^{-}+\frac{\gamma}{2\beta} \textbf{E}$. Note that $\textbf{A},\textbf{A}^{+},\textbf{A}^{-},\textbf{B},\textbf{B}^{+},\textbf{B}^{-}$ are all the symmetric matrices. Because of the constraints of $\textbf{M}$, $\textbf{M}$ is updated by fixed point iteration for each element of $\textbf{M}$:
\begin{equation}
	\begin{aligned}
		\frac{\partial J(\textbf{M})}{\partial \textbf{M}_{i, j}}=(\textbf{MA}+\textbf{AM})_{i, j}-2 \textbf{B}_{i, j}\quad\quad\quad\quad\quad\quad\quad\quad\quad\quad\quad\\= \left(\textbf{MA}^{+}+\textbf{A}^{+} \textbf{M}\right)_{i, j}+2 \textbf{B}_{i, j}^{-}-\left(\textbf{MA}^{-}+\textbf{A}^{-} \textbf{M}\right)_{i, j}\\-2 \textbf{B}_{i, j}^{+}=0\quad\quad\quad\quad\quad\quad\quad\quad\quad\quad\quad\quad\quad\quad\quad\quad\\
		\Leftrightarrow \textbf{M}_{i, j}^{2}\left[\left(\textbf{M} \textbf{A}^{+}+\textbf{A}^{+} \textbf{M}\right)_{i, j}+2 \textbf{B}_{i, j}^{-}\right]\quad\quad\quad\quad\quad\quad\quad\quad\quad\\
		=\textbf{M}_{i, j}^{2}\left[\left(\textbf{MA}^{-}+\textbf{A}^{-} \textbf{M}\right)_{i, j}+2 \textbf{B}_{i, j}^{+}\right]\quad\quad\quad\quad\quad\quad \\
		\Leftrightarrow \textbf{M}_{i, j}=\textbf{M}_{i, j} \sqrt{\frac{\left(\textbf{MA}^{-}+\textbf{A}^{-} \textbf{M}\right)_{i, j}+2 \textbf{B}_{i, j}^{+}}{\left(\textbf{M} \textbf{A}^{+}+\textbf{A}^{+} \textbf{M}\right)_{i, j}+2 \textbf{B}_{i, j}^{-}}}\quad\quad\quad\quad\quad
	\end{aligned}
	\label{eq23}
\end{equation}

If the initialized $\textbf{M}$ satisfies $\textbf{M}_{i,i}=0,\textbf{M}_{i,j}=\textbf{M}_{j,i}>0$, the properties of $\textbf{M}$ always hold in the iteration process of Eq.~(\ref{eq23}). As a result, we can iteratively obtain the optimal values of each variable. During the semi-supervised optimization process, label correlations are efficiently and fully explored with the feature space, label space, and shared label subspace kept consistent. Eventually, the optimal feature weight matrix $\textbf{W}$ is used to select features with high confidence for multi-label learning. The optimization process of our proposed SGMFS is summarized in Algorithm~\ref{Alg1}.

\subsection{Convergence Analysis}
In order to prove the convergence of Algorithm~\ref{Alg1}, we need to prove that $\textbf{J}\left(\textbf{M}\right)$ in Eq.~(\ref{eq21}) and Eq.~(\ref{eq22}) is monotonically decreasing with $\textbf{M}$ updated by Eq.~(\ref{eq23}), before that we need to prove the following propositions first.

\begin{proposition}\label{prop3}
Given the objective function Eq.~(\ref{eq22}), where $\textbf{A}^{+},\textbf{A}^{-},\textbf{B}^{+}$, $\textbf{B}^{-},\textbf{M}$ are all nonnegative and symmetry, the auxiliary function \cite{ding2008convex} of Eq.~(\ref{eq22}) is defined as follows:
\begin{equation}
	\begin{aligned}
		\mathscr{P}\left(\textbf{M}, \textbf{M}^{\prime}\right)=&-\sum_{i k} 2 \textbf{B}_{i k}^{+} \textbf{M}_{i k}^{\prime}\left(1+\log \frac{\textbf{M}_{i k}}{\textbf{M}_{i k}^{\prime}}\right)\\&+\sum_{i k} \textbf{B}_{i k}^{-} \frac{\textbf{M}_{i k}^{2}+\textbf{M}_{i k}^{\prime 2}}{\textbf{M}_{i k}^{\prime}}+\sum_{i k} \frac{\left(\textbf{A}^{+} \textbf{M}^{\prime}\right)_{i k} \textbf{M}_{i k}^{2}}{\textbf{M}_{i k}^{\prime}}\\&-\sum_{i k \ell} \textbf{A}_{k \ell}^{-} \textbf{M}_{i k}^{\prime} \textbf{M}_{i \ell}^{\prime}\left(1+\log \frac{\textbf{M}_{i k} \textbf{M}_{i \ell}}{\textbf{M}_{i k}^{\prime} \textbf{M}_{i \ell}^{\prime}}\right)
	\end{aligned}
	\label{eq24}
\end{equation}
which is a convex function and  the global minimum of Eq.~(\ref{eq24}) with respect to $\textbf{M}$ is:
\begin{equation}
	\begin{aligned}
		\textbf{M}_{i k} &=\arg \min _{\textbf{M}} \mathscr{P}\left(\textbf{M}, \textbf{M}^{\prime}\right) \\
		&=\textbf{M}_{i k}^{\prime} \sqrt{\frac{\left(\textbf{M}^{\prime} \textbf{A}^{-}\right)_{i k}+\left(\textbf{M}^{\prime} \textbf{A}^{-}\right)_{k i}+2\textbf{B}_{i k}^{+}}{\left(\textbf{M}^{\prime} \textbf{A}^{+}\right)_{i k}+\left(\textbf{M}^{\prime} \textbf{A}^{+}\right)_{k i}+2\textbf{B}_{i k}^{-}}} .
	\end{aligned}
	\label{eq25}
\end{equation}
\end{proposition}
\begin{proof}
	Due to the symmetry of matrices $\textbf{B}^{+},\textbf{B}^{-}$ and $M$, we the first derivatives of $\mathscr{P}\left(\textbf{M}, \textbf{M}^{\prime}\right)$:
	\begin{equation}
		\begin{aligned}
			\frac{\partial \mathscr{P}\left(\textbf{M}, \textbf{M}^{\prime}\right)}{\partial \textbf{M}_{i k}}=& \frac{2\left(\textbf{M}^{\prime} \textbf{A}^{+}\right)_{i k} \textbf{M}_{i k}}{\textbf{M}_{i k}^{\prime}}+\frac{2\left(\textbf{M}^{\prime} \textbf{A}^{+}\right)_{k i} \textbf{M}_{i k}}{\textbf{M}_{i k}^{\prime}} \\
			&-\frac{2\left(\textbf{M}^{\prime} \textbf{A}^{-}\right)_{i k} \textbf{M}_{i k}^{\prime}}{\textbf{M}_{i k}}-\frac{2\left(\textbf{M}^{\prime} \textbf{A}^{-}\right)_{k i} \textbf{M}_{i k}^{\prime}}{\textbf{M}_{i k}} \\
			&-4 \textbf{B}_{i k}^{+} \frac{\textbf{M}_{i k}^{\prime}}{\textbf{M}_{i k}}+4 \textbf{B}_{i k}^{-} \frac{\textbf{M}_{i k}}{\textbf{M}_{i k}^{\prime}}
		\end{aligned}
		\label{eq26}
	\end{equation}
	Based on Eq.~(\ref{eq26}), we can get the second derivatives:
	\begin{equation}
		\frac{\partial^{2} \mathscr{P} \left(\textbf{M}, \textbf{M}^{\prime}\right)}{\partial \textbf{M}_{i k} \partial \textbf{M}_{j l}}=\delta_{i j} \delta_{k l} \Gamma_{i k}
		\label{eq27}
	\end{equation}
	where $\delta_{i j}=0, i \neq j$ and $\delta_{l,l}=1$, and
	\begin{equation}
		\begin{aligned}
			\Gamma_{i k}=&2\textbf{M}_{i k}^{\prime} \frac{\left[\left(\textbf{M}^{\prime} \textbf{A}^{-}\right)_{i k}+\left(\textbf{M}^{\prime} \textbf{A}^{-}\right)_{k i}+2\textbf{B}_{i k}^{+}\right]}{\textbf{M}_{i k}^{2}}  \\
			+& 2 \frac{\left(\textbf{M}^{\prime} \textbf{A}^{+}\right)_{i k}+\left(\textbf{M}^{\prime} \textbf{A}^{+}\right)_{k i}+2\textbf{B}_{i k}^{-}}{\textbf{M}_{i k}^{\prime}}>0.
		\end{aligned}
		\label{eq28}
	\end{equation}
	Therefore, the Hessian matrix is a diagonal matrix with positive diagonal elements. Thus Eq.~(\ref{eq24}) is a convex function and we can set Eq.~(\ref{eq24}) to zero to find the minimum, that is
	\begin{equation}
		\begin{aligned}
			4 \textbf{B}_{i k}^{-} \frac{\textbf{M}_{i k}}{\textbf{M}_{i k}^{\prime}}+\frac{2\left(\textbf{M}^{\prime} \textbf{A}^{+}\right)_{i k} \textbf{M}_{i k}}{\textbf{M}_{i k}^{\prime}}+\frac{2\left(\textbf{M}^{\prime} \textbf{A}^{+}\right)_{k i} \textbf{M}_{i k}}{\textbf{M}_{i k}^{\prime}} \\
			=\frac{2\left(\textbf{M}^{\prime} \textbf{A}^{-}\right)_{i k} \textbf{M}_{i k}^{\prime}}{\textbf{M}_{i k}}+\frac{2\left(\textbf{M}^{\prime} \textbf{A}^{-}\right)_{k i} \textbf{M}_{i k}^{\prime}}{\textbf{M}_{i k}}+4 \textbf{B}_{i k}^{+} \frac{\textbf{M}_{i k}^{\prime}}{\textbf{M}_{i k}}
		\end{aligned}
		\label{eq29}
	\end{equation}
	which is equivalent to Eq.~(\ref{eq25}).
\end{proof}

Based on Proposition~\ref{prop3}, we have the following theorem:
\begin{theorem}\label{theo1}
The objective function defined in Eq.~(\ref{eq22}) satisfy $J\left(\textbf{M}_{t}\right) \geq J\left(\textbf{M}_{t+1}\right)$ under the update rule of Eq.~(\ref{eq23}).
\end{theorem}
\begin{proof}
See \ref{apen2} for the proof of Theorem~\ref{theo1}.
\end{proof}

Then we prove another proposition as follows, after which we propose and prove the convergence of Algorithm~\ref{Alg1}.

\begin{proposition}\label{prop4}
The limiting solution of $\textbf{M}$ updated by rule Eq.~(\ref{eq23}) satisfies the KKT condition \cite{boyd2004convex} of the optimization problems Eq.~(\ref{eq21}) and Eq.~(\ref{eq22}).
\end{proposition}
\begin{proof}
	As illustrated above, the optimization problem Eq.~(\ref{eq21}) is equal to Eq.~(\ref{eq22}). Thus we only need to prove this proposition is true for Eq.~(\ref{eq22}).
	
	Since $\textbf{M}$ must satisfy $\textbf{M}_{i,j}>0$ in Eq.~(\ref{eq22}), the Lagrangian function can be constructed as:
	\begin{equation}
		L(\textbf{M})=J(\textbf{M})-T r\left(\Lambda^{T} \textbf{M}\right)
		\label{eq36}
	\end{equation}
	where $\Lambda^{T}$ is the Lagrangian multiplier. According to setting the derivative of $J\left(\textbf{M}\right)$ to zero, we have:
	\begin{equation}
		\frac{\partial L(\textbf{M})}{\partial \textbf{M}}=\frac{\partial J(\textbf{M})}{\partial \textbf{M}}-\Lambda=0
		\label{eq37}
	\end{equation}
	According to the the complementary slackness condition $\Lambda_{i j} \textbf{M}_{i j}=0$, Eq.~(\ref{eq37}) can be written as:
	\begin{equation}
		\frac{\partial J(\textbf{M})}{\partial \textbf{M}_{i j}} \textbf{M}_{i j}=\Lambda_{i j} \textbf{M}_{i j}=0
		\label{eq38}
	\end{equation}
	Notice that Eq.~(\ref{eq38}) is a fixed point equation, and the optimal solution of $\textbf{M}$ must satisfy this equation at convergence.
	
	In another aspect, since $\textbf{M}_{i,j}=0$ is equivalent to $\textbf{M}_{i,j}^{2}=0$, if Eq.~(\ref{eq23}) holds, Eq.~(\ref{eq38}) also holds and vice versa. Therefore, the limiting $\textbf{M}$ updated by Eq.~(\ref{eq23}) satisfies the fixed point equation.
\end{proof}
\begin{theorem}\label{theo2}
The objective function of the proposed SGMFS in Eq.~(\ref{eq11}) decreases monotonically with the optimization process in Algorithm~\ref{Alg1}.
\end{theorem}
\begin{proof}
See \ref{apen3} for the proof of Theorem~\ref{theo2}.
\end{proof}

\subsection{Computational Complexity Analysis}
In this subsection, we briefly analyze the computational complexity of SGMFS. In each iteration of the training phase, the main computational complexity of SGMFS locates in three parts: calculating the inverse of matrix $\textbf{S}-\alpha \textbf{XQQ}^{T}\textbf{X}^{T}$ for obtaining the new feature weight matrix $\textbf{W}$, computing the shared label matrix $\textbf{Q}$ in label subspace and updating the adaptive sparse graph matrix $\textbf{M}$.

The computational complexity of calculating the inverse of matrices for $\textbf{W}$ is $\mathcal{O}\left(nd*\text{min}\left\{n,d\right\}\right)$, the detailed analysis of which is provided in \ref{apen1}. And updating $\textbf{Q}$ requires the eigen-decomposition operation of $\textbf{C}$ in Algorithm~\ref{Alg1}, which needs $\mathcal{O}\left((lsd)*d^{2}\right)$, since only the eigenvectors corresponding to the largest $lsd$ eigenvalues of $\textbf{C}$ are required.

In order to compute sparse graph $\textbf{M}$ in Eq.~(\ref{eq23}), $\textbf{A}^{+}$, $\textbf{A}^{-}$, $\textbf{B}^{+}$ and $\textbf{B}^{-}$ are calculated first, and these matrices including $\textbf{M}$ are all symmetric, thus $\left(\textbf{MA}^{-}\right)^{T}=\textbf{A}^{-}\textbf{M}$ and $\left(\textbf{MA}^{+}\right)^{T}=\textbf{A}^{+}\textbf{M}$. Hence, only two straightforward matrix multiplication operations of n-by-n matrices are needed, with a complexity of $\mathcal{O}\left(n^{3}\right)$.

However, in other semi-supervised feature selection methods that rely on $k$NN, the parameter $k$ is hard to tune to obtain a proper similarity structure and competitive performance, which require huge computational costs. And learning the structured optimal graph in Eq.~(\ref{eq2}) also needs higher computational cost when solving relevant optimization subproblems.

\section{Experiments}
\label{S5}
In this section, we perform comprehensive experiments to compare SGMFS with other state-of-the-art methods, followed by an in-depth analysis of SGMFS.
\subsection{Datasets and Baselines}
We utilize seven multi-label benchmarks spanning five distinct domains, including four small-scale datasets and three large-scale datasets, all available at \url{http://mulan.sourceforge.net/datasets-mlc.html} and \url{http://www.uco.es/kdis/mllresources/}. To ensure experimental reliability, we randomly select suitable training and testing samples. Details of these datasets are presented in Table~\ref{tab1}, where the attribute \emph{cardinality} represents the average number of labels per sample, and \emph{density} refers to the normalized \emph{cardinality}, calculated as \emph{cardinality} divided by the total number of class labels. It is worth noting that the class labels in these datasets are binary.

\begin{table}[t]
	\centering
	\renewcommand\arraystretch{1.2}
	\setlength\tabcolsep{1.3pt}
	\footnotesize
	\caption{Detailed characteristics of multi-label datasets used in our experiments}
	\begin{tabular}{|c|c|c|c|c|c|c|c|}
		\hline  
		~&Domain&Features&Labels&Training&Test&\emph{Cardinality}&\emph{Density}\\
		\hline  
		
		Emotions&music&72&6&400&100&1.869&0.311\\
		\hline 
        EUR-Lex&text&5000&201&10000&1000&2.213&0.011\\
		\hline
        Mediamill&video&120&101&10000&1000&4.376&0.043\\
		\hline
        NUS-WIDE&image&500&81&10000&1000&1.869&0.023\\
		\hline
		Scene&image&294&6&1000&500&1.074&0.179\\
		\hline
		Yeast&biology&103&14&1500&500&4.237&0.303\\
		\hline
        Plant&biology&440&12&685&293&1.079&0.090\\
		\hline
	\end{tabular}
	\label{tab1}
\end{table}

To assess the effectiveness of the proposed SGMFS, we compare it with seven advanced feature selection methods: MIFS \cite{jian2016multi}, FSNM \cite{nie2010efficient}, CSFS \cite{chang2014convex}, LSDF \cite{zhao2008locality}, SFSS \cite{ma2012discriminating}, SCFS \cite{xu2018semi}, and SFS-BLL \cite{shi2021binary}. 
The first two algorithms are fully supervised methods, while the latter five are semi-supervised approaches.
Subsequently, the representative multi-label classification algorithm ML-$k$NN \cite{zhang2007ml} is employed to complete the multi-label learning process using the selected features. 

We adopt two example-based metrics, $Hamming$ $Loss$ and $Ranking$ $Loss$, along with two label-based metrics, $F_{Macro}$ and $F_{Micro}$ (macro-averaged and micro-averaged F1-scores), to assess the performance of the feature selection methods. Lower values for the first two metrics indicate better performance, while higher values for the latter two signify improved results~\cite{tang2014feature,zhong2021multi}.

\begin{figure*}[p]
	\setlength{\abovecaptionskip}{0pt}
	\centering
    \vspace{-3cm}
	\subfigure[]{
		\label{fig:subfig:a} 
		\includegraphics[width=0.25\columnwidth]{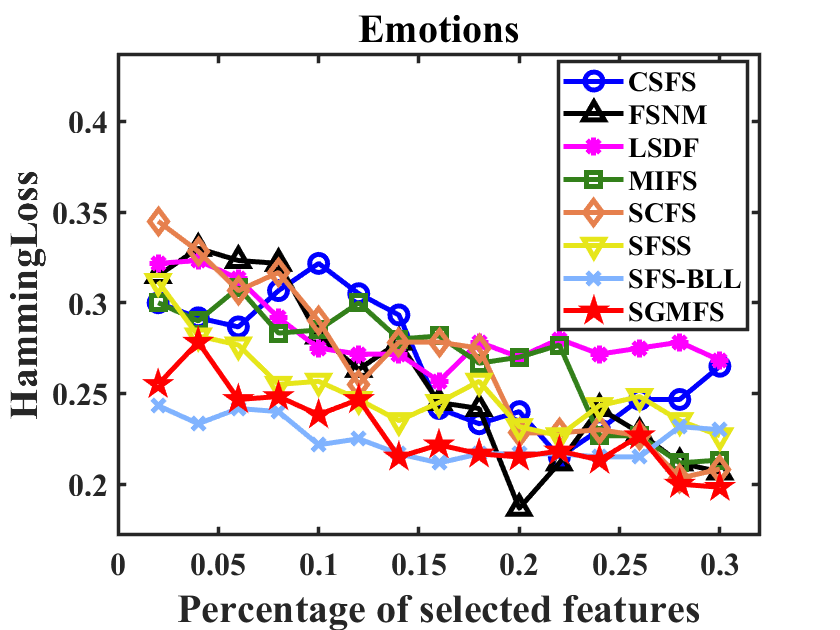}}
	\hspace{-3.1mm}
	\subfigure[]{
		\label{fig:subfig:b} 
		\includegraphics[width=0.25\columnwidth]{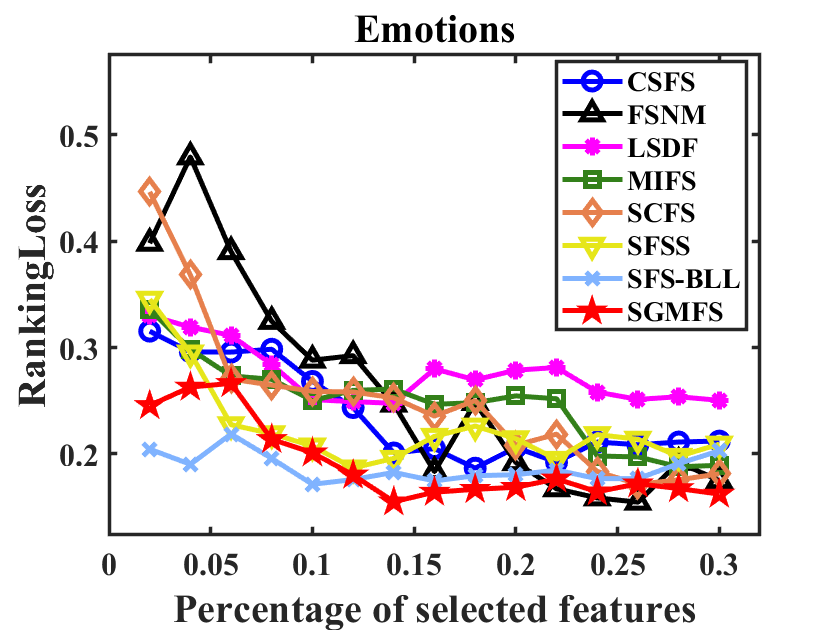}}
	\hspace{-3.1mm}
	\subfigure[]{
		\label{fig:subfig:c} 
		\includegraphics[width=0.25\columnwidth]{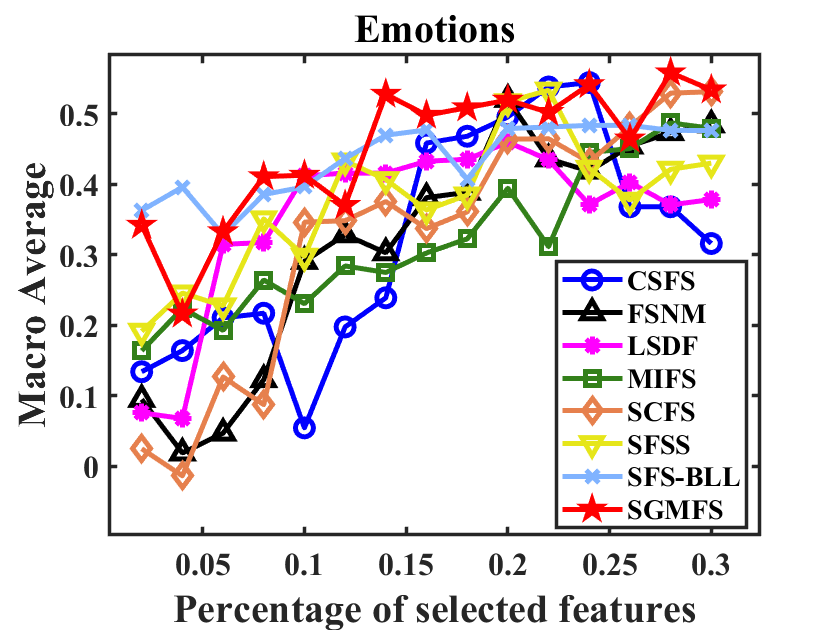}}
	\hspace{-3.1mm}
	\subfigure[]{
		\label{fig:subfig:d} 
		\includegraphics[width=0.25\columnwidth]{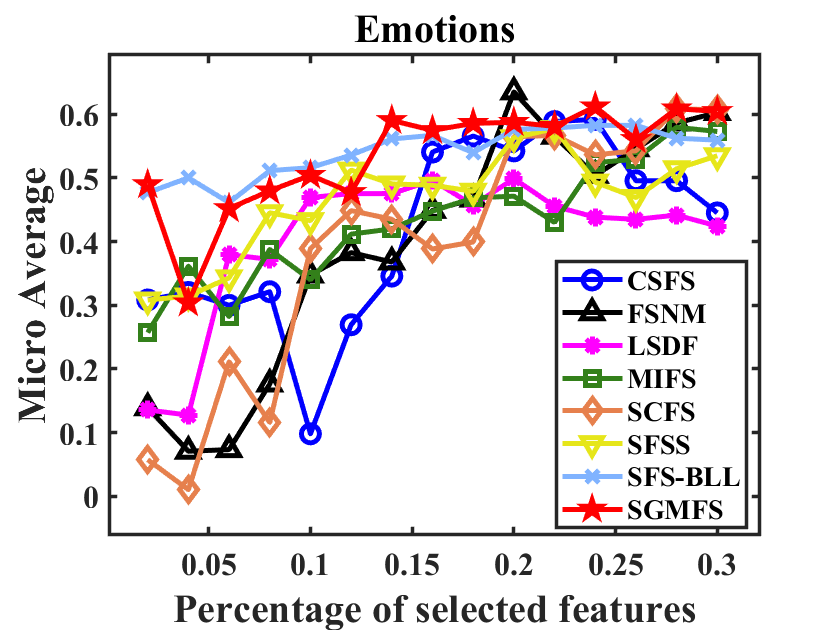}}
	\hspace{-3.1mm}
	\subfigure[]{
		\label{fig:subfig:e} 
		\includegraphics[width=0.25\columnwidth]{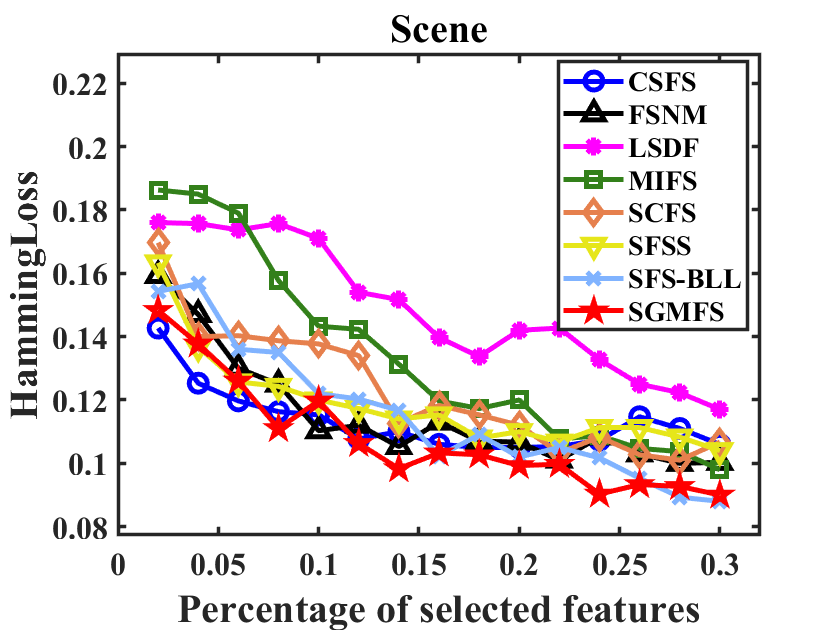}}
	\hspace{-3.1mm}
	\vspace{-0mm}
	\subfigure[]{
		\label{fig:subfig:f} 
		\includegraphics[width=0.25\columnwidth]{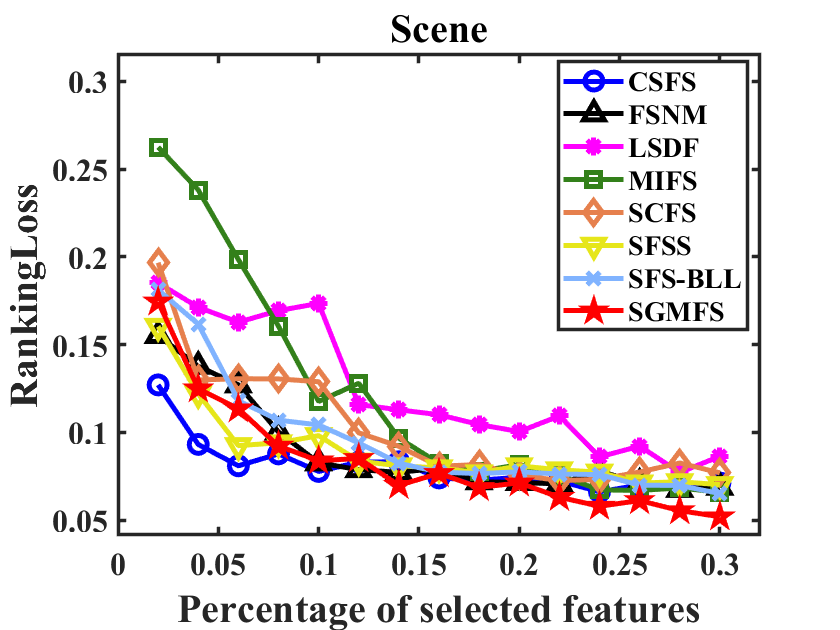}}
	\hspace{-3.1mm}
	\subfigure[]{
		\label{fig:subfig:g} 
		\includegraphics[width=0.25\columnwidth]{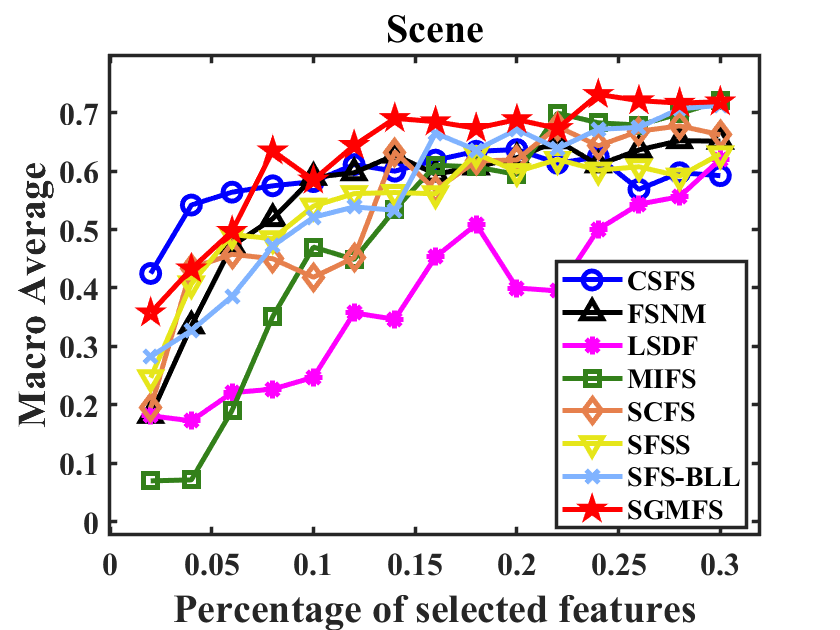}}
	\hspace{-3.1mm}
	\subfigure[]{
		\label{fig:subfig:h} 
		\includegraphics[width=0.25\columnwidth]{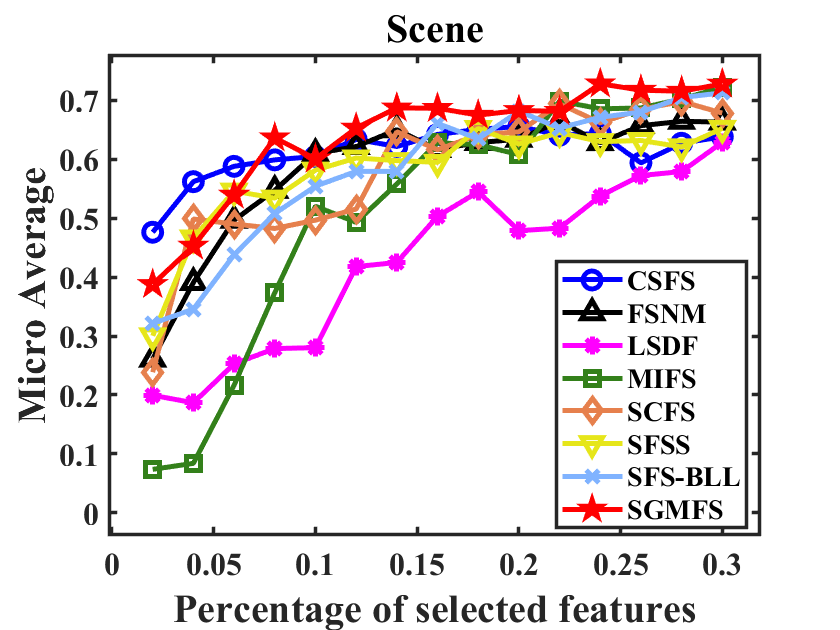}}
	\hspace{-3.1mm}
	\subfigure[]{
		\label{fig:subfig:i} 
		\includegraphics[width=0.25\columnwidth]{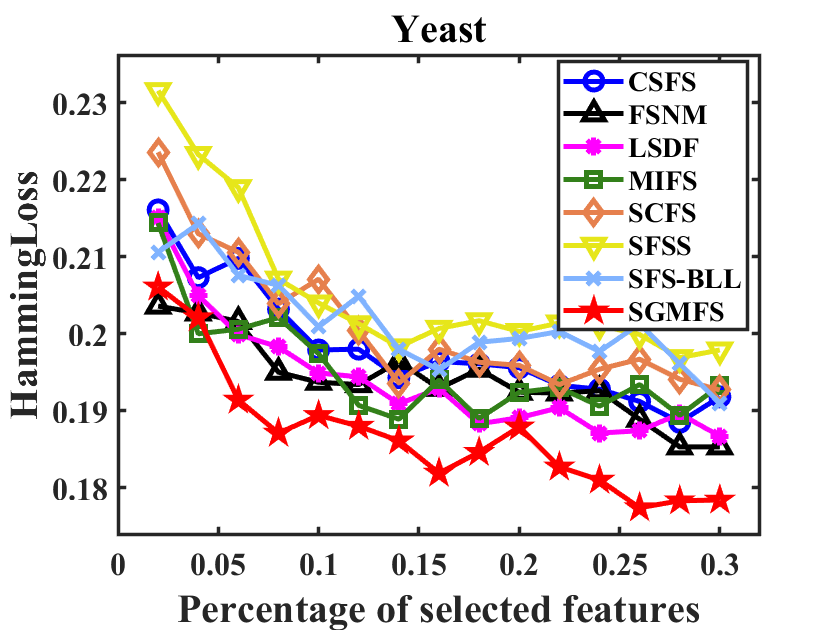}}
	\hspace{-3.1mm}
	\subfigure[]{
		\label{fig:subfig:j} 
		\includegraphics[width=0.25\columnwidth]{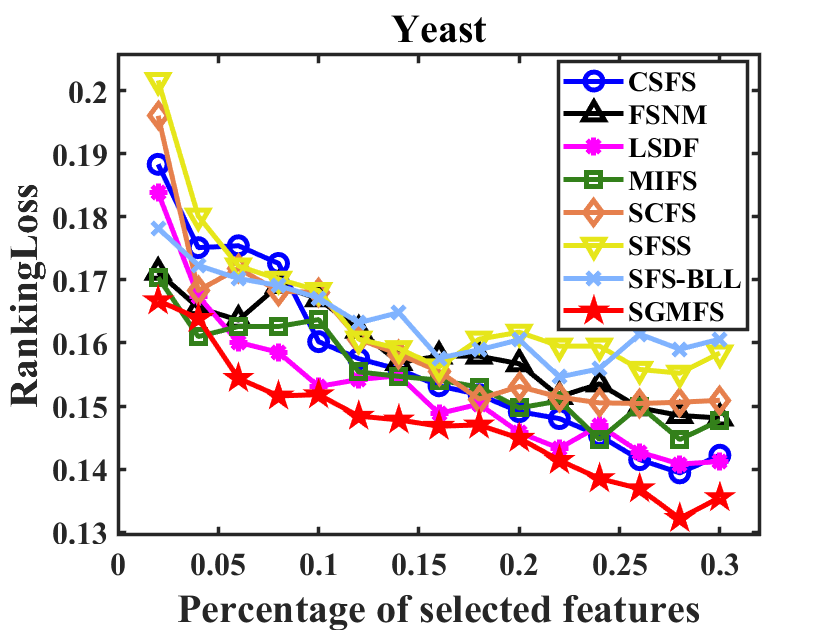}}
	\hspace{-3.1mm}
	\vspace{-0mm}
	\subfigure[]{
		\label{fig:subfig:k} 
		\includegraphics[width=0.25\columnwidth]{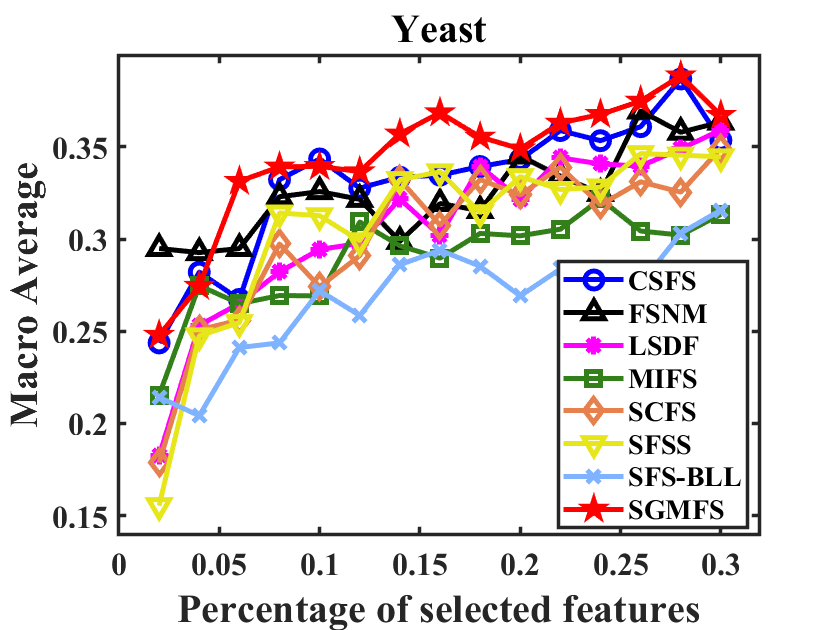}}
	\hspace{-3.1mm}
	\subfigure[]{
		\label{fig:subfig:l} 
		\includegraphics[width=0.25\columnwidth]{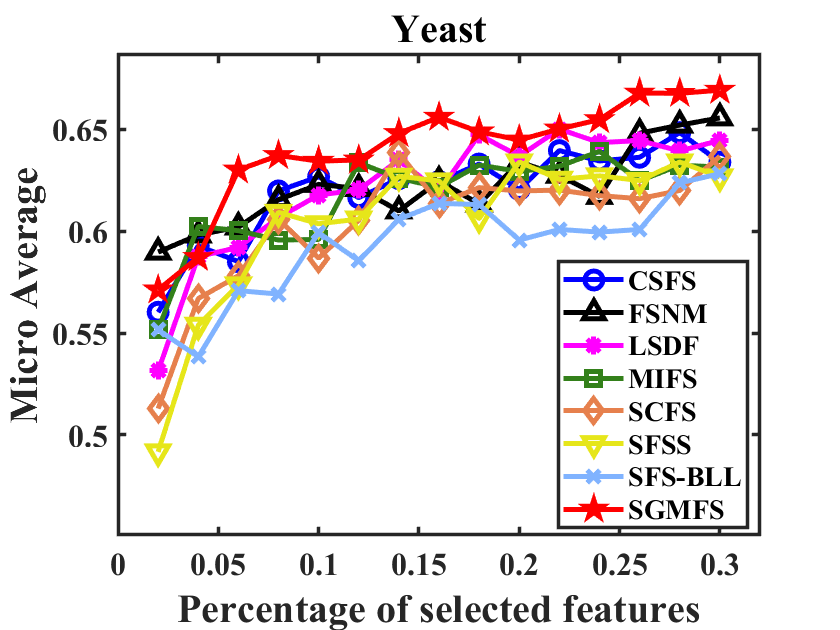}}
	\hspace{-3.1mm}
	\subfigure[]{
		\label{fig:subfig:m} 
		\includegraphics[width=0.25\columnwidth]{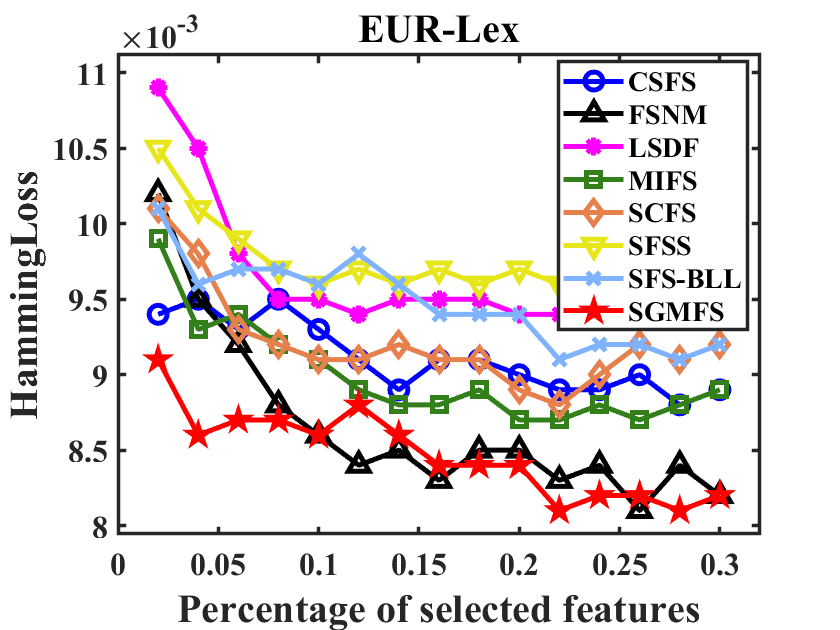}}
	\hspace{-3.1mm}
	\subfigure[]{
		\label{fig:subfig:n} 
		\includegraphics[width=0.25\columnwidth]{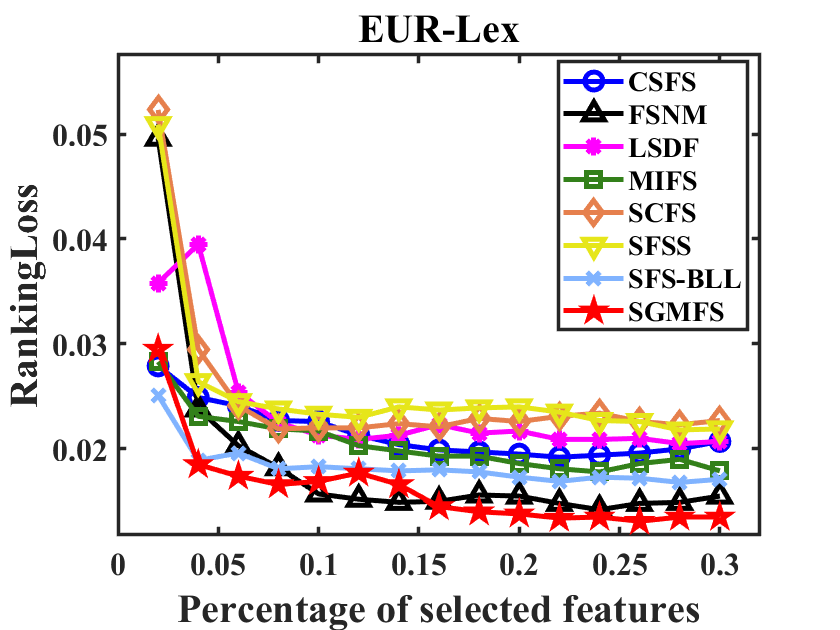}}
	\hspace{-3.1mm}
	\subfigure[]{
		\label{fig:subfig:o} 
		\includegraphics[width=0.25\columnwidth]{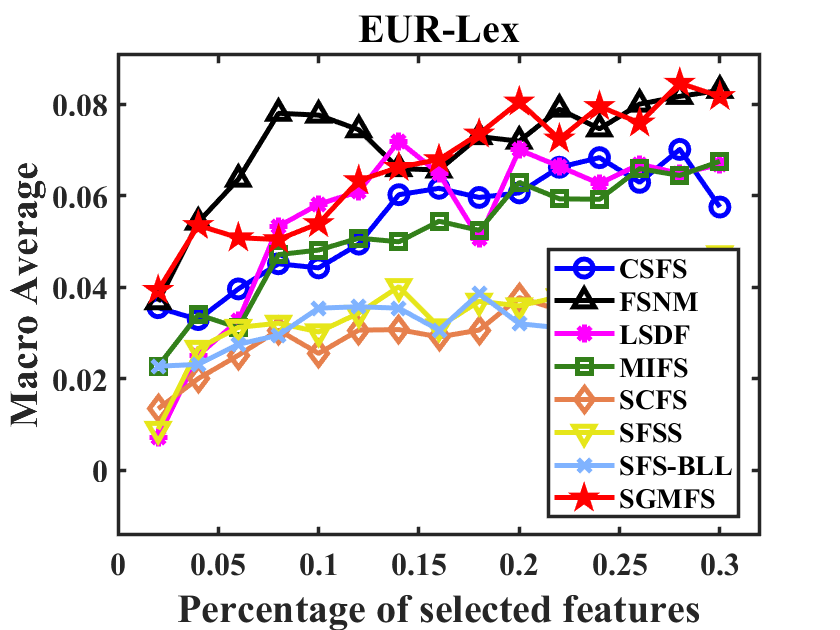}}
	\hspace{-3.1mm}
	\subfigure[]{
		\label{fig:subfig:p} 
		\includegraphics[width=0.25\columnwidth]{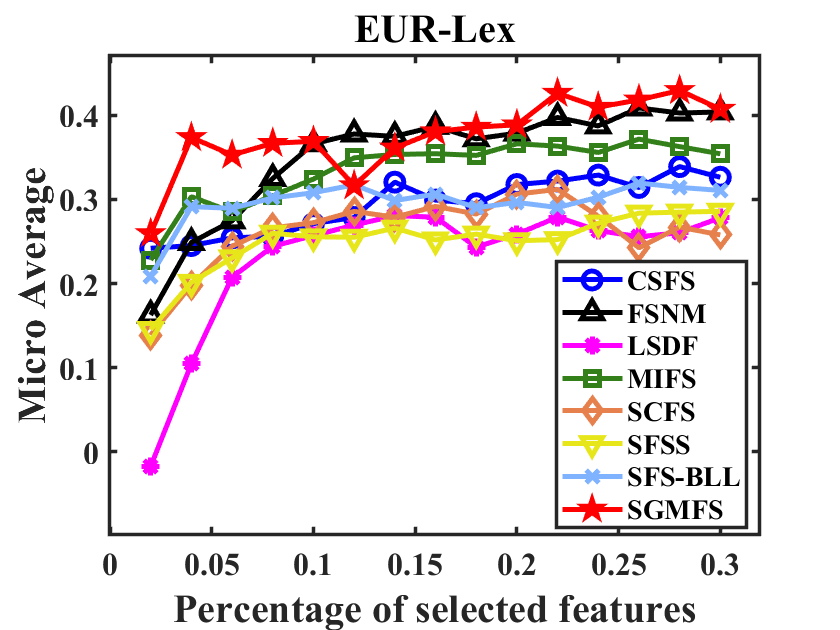}}
	\hspace{-3.1mm}
	\subfigure[]{
		\label{fig:subfig:q} 
		\includegraphics[width=0.25\columnwidth]{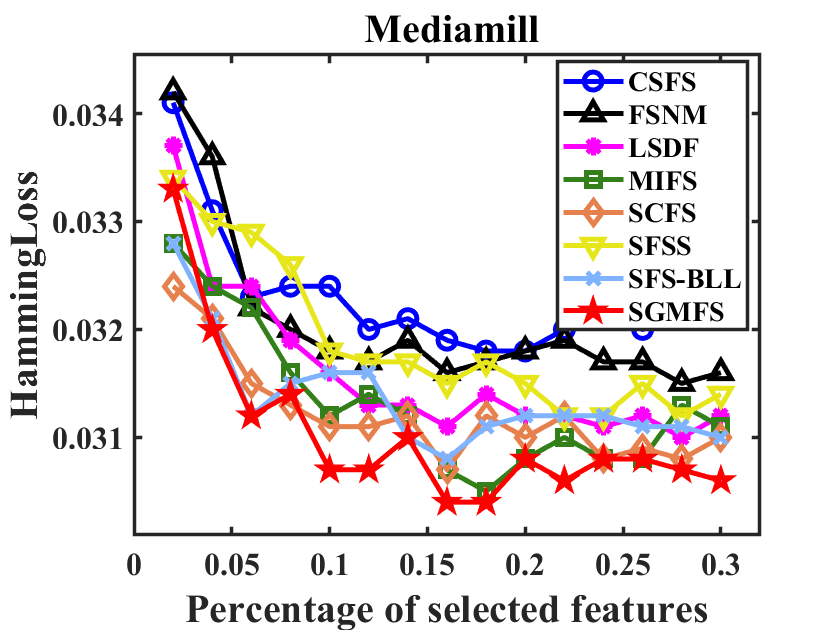}}
	\hspace{-3.1mm}
	\subfigure[]{
		\label{fig:subfig:r} 
		\includegraphics[width=0.25\columnwidth]{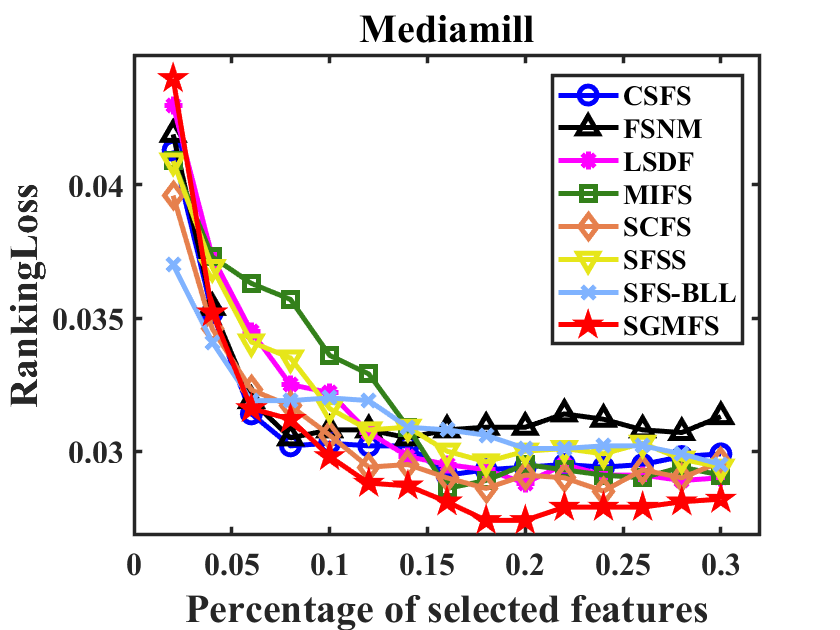}}
	\hspace{-3.1mm}
	\subfigure[]{
		\label{fig:subfig:s} 
		\includegraphics[width=0.25\columnwidth]{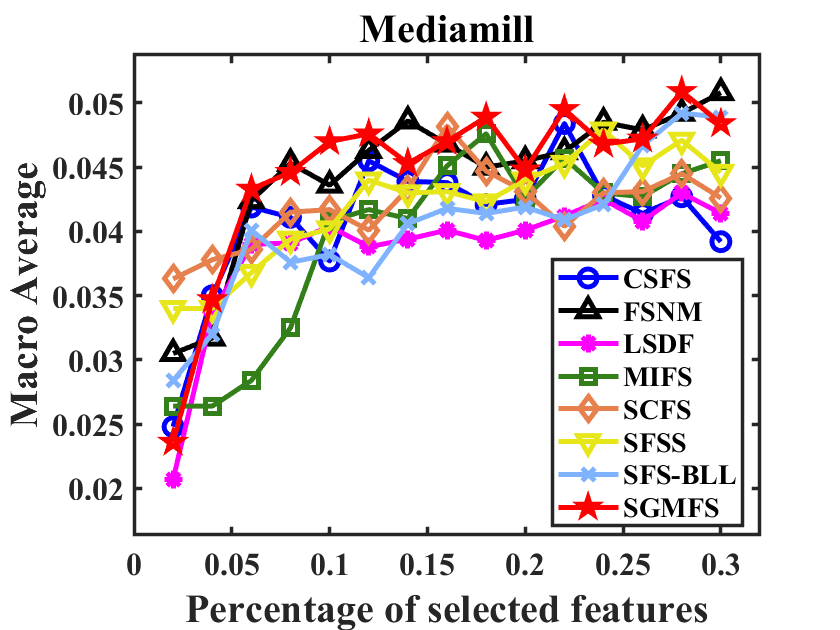}}
	\hspace{-3.1mm}
	\subfigure[]{
		\label{fig:subfig:t} 
		\includegraphics[width=0.25\columnwidth]{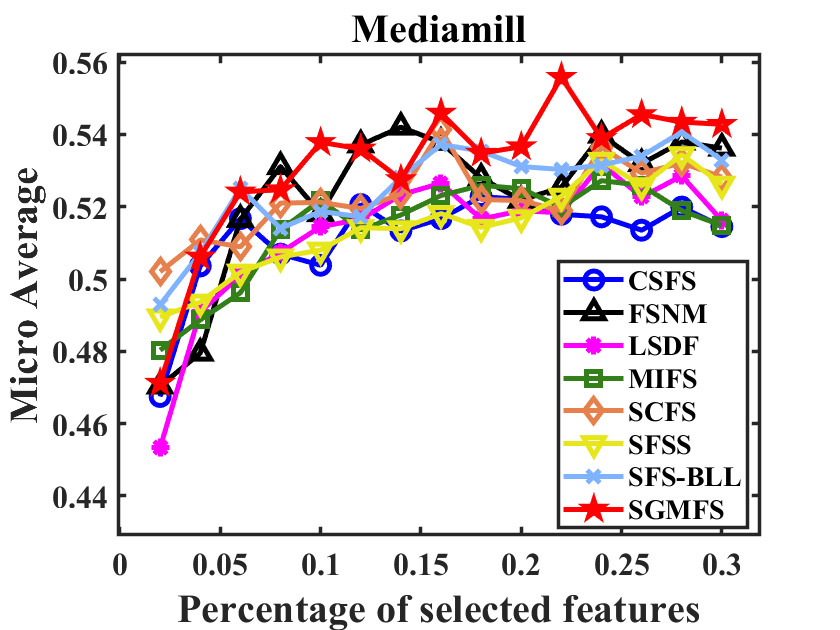}}
	\hspace{-3.1mm}
	\subfigure[]{
		\label{fig:subfig:s} 
		\includegraphics[width=0.25\columnwidth]{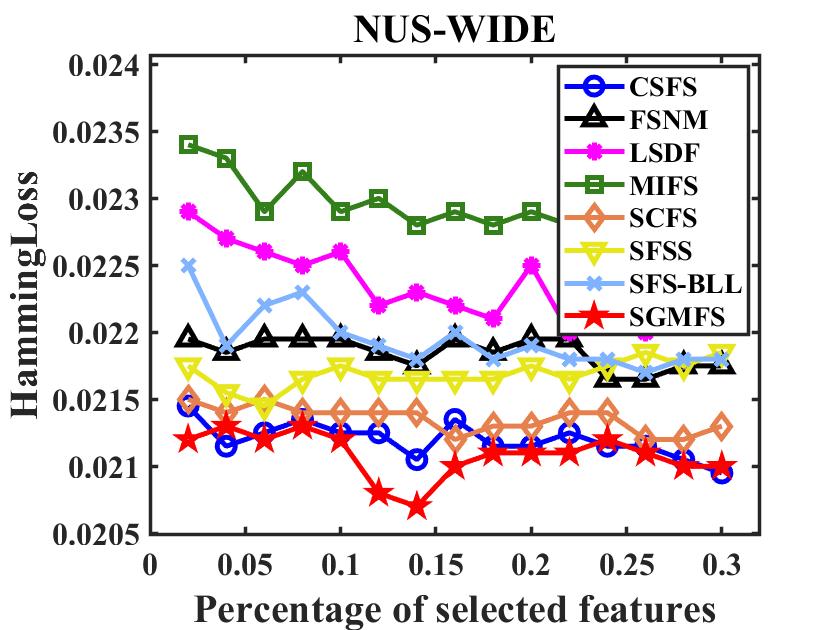}}
	\hspace{-3.1mm}
	\subfigure[]{
		\label{fig:subfig:s} 
		\includegraphics[width=0.25\columnwidth]{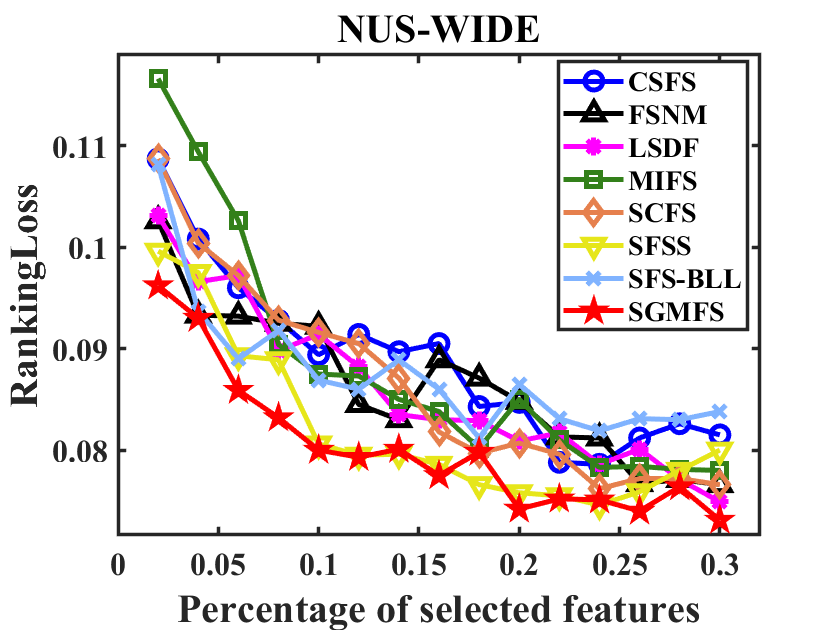}}
	\hspace{-3.1mm}
	\subfigure[]{
		\label{fig:subfig:s} 
		\includegraphics[width=0.25\columnwidth]{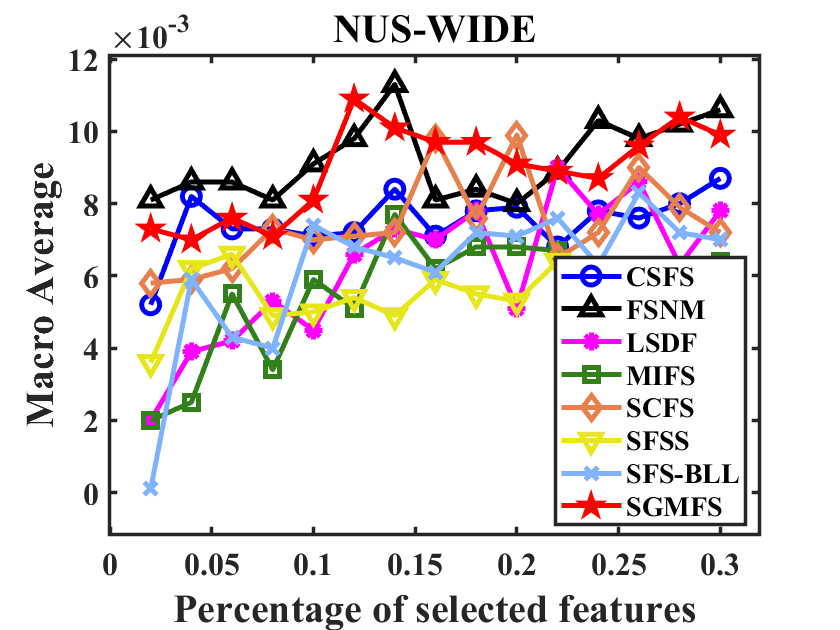}}
	\hspace{-3.1mm}
	\subfigure[]{
		\label{fig:subfig:s} 
		\includegraphics[width=0.25\columnwidth]{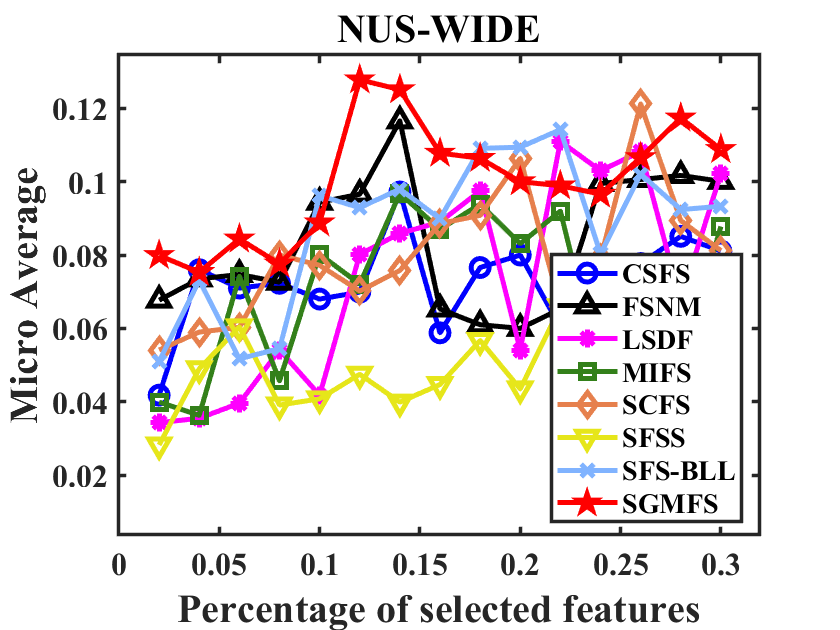}}
	\caption{
    Comparison of SGMFS with other feature selection algorithms (CSFS, FSNM, LSDF, MIFS, SCFS, SFSS, and SFS-BLL) on using four metrics: \emph{Hamming Loss}, \emph{Ranking Loss}, \emph{Macro Average}, and \emph{Micro Average}.
    }
	\label{fig3} 
\end{figure*}

\subsection{Experiment Setup and Performance Comparison}

In our experiments, ML-$k$NN is applied with the selected features for multi-label learning to compare the performance of various feature selection methods. For fairness, the parameters of ML-$k$NN are kept consistent for each feature selection method, with the number of nearest neighbors $k$ set to $10$ and the smoothing parameter fixed at $1$. We evaluate classification accuracy using feature selection proportions ranging from 2\% to 30\% in 2\% increments. Each regularization parameter is individually tuned from the set $\left\{10^{-3}, 10^{-2}, 10^{-1}, 1, 10, 10^{2}, 10^{3}\right\}$, and the optimal parameter is chosen for each method in every experiment. For the proposed SGMFS, the sparsity parameter $\gamma$ is set to 1, and the label subspace dimension $lsd$ is set to half of the initial number of labels ($c/2$). To ensure reliability, ML-$k$NN is run 10 times with the best-selected parameters for each method and feature proportion. In each run, training and test samples are randomly selected (non-overlapping) from the dataset. The average values of each metric across the 10 runs are then calculated to evaluate the performance of each method.

Due to the space limitation, we use the 10\% labeled and 90\% unlabeled training samples to select features according to SGMFS and other semi-supervised methods (CSFS, LSDF, SFSS, and SCFS), and only the 10\% labeled training samples are used for supervised feature selection methods (MIFS and FSNM). Finally, the comparison results of $Hamming$ $Loss$, $Ranking$ $Loss$, $F_{Macro}$ and $F_{Micro}$ between SGMFS and other methods are displayed in Fig.~\ref{fig3}, from which we can observe that:

(\romannumeral1). SGMFS has a better performance than other feature selection methods in most of the percentage of selected features, especially for the selection proportion ranging from 15\% to 30\%. The primary reason is that when the number of features is too limited, the accurate sparse graph structure of labeled and unlabeled data points cannot be adequately captured. This scenario underscores the critical significance of the space consistency and structural coherence during semi-supervised multi-label training.

(\romannumeral2). Not only small datasets, SGMFS can also perform better than other semi-supervised methods in three large datasets in general, which illustrates that SGMFS is able to handle large incompletely labeled datasets with more class labels because label correlations are considered in SGMFS. Although the best $Macro$ $Average$ value of SGMFS is worse than that of supervised method FSNM in three large datasets, it is reasonable since few methods can simultaneously perform best on all the multi-label metrics.

(\romannumeral3). For small datasets, $Hamming$ $Loss$ and $Ranking$ $Loss$ reduce as the percentage of features increases, while $F_{Macro}$ and $F_{Micro}$ are the opposite. However, they have no significant variation with feature proportion ranging from 20\% to 30\% in large datasets, which is reasonable due to the enough training samples.

In conclusion, our proposed semi-supervised multi-label feature selection method has a significant competitiveness with some state-of-the-art feature selection algorithms.

\begin{figure*}[t]
	\setlength{\abovecaptionskip}{0pt}
	\centering
	\subfigure[$Hamming$ $Loss$]{
		\label{fig:subfig:a} 
		\includegraphics[width=0.48\columnwidth]{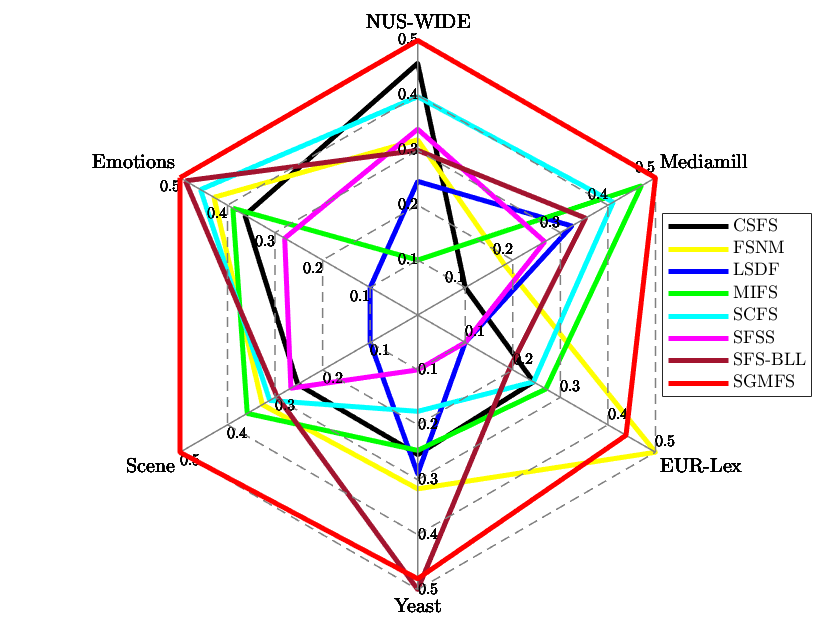}}
	\hspace{-3.1mm}
	\subfigure[$Ranking$ $Loss$]{
		\label{fig:subfig:b} 
		\includegraphics[width=0.48\columnwidth]{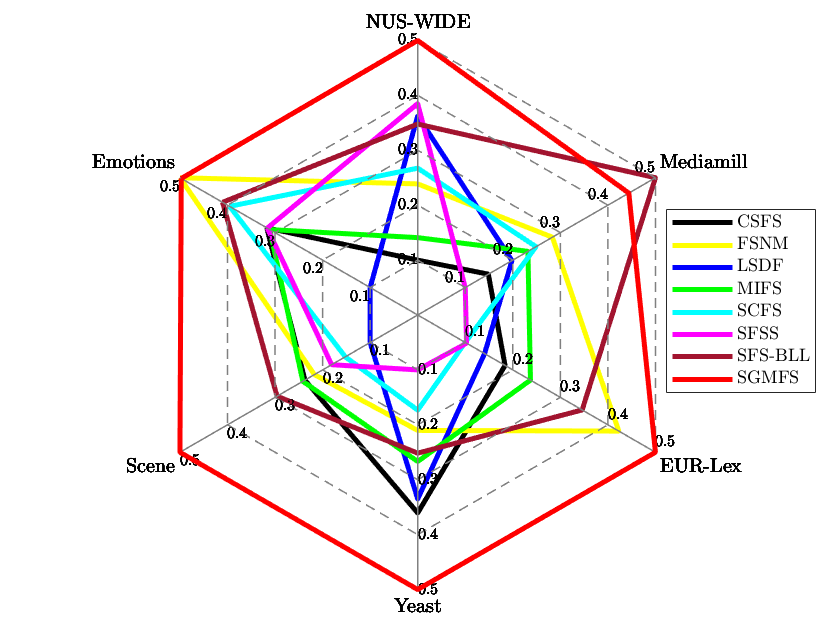}}
	\hspace{-3.1mm}
	\subfigure[$Macro$ $Average$]{
		\label{fig:subfig:c} 
		\includegraphics[width=0.48\columnwidth]{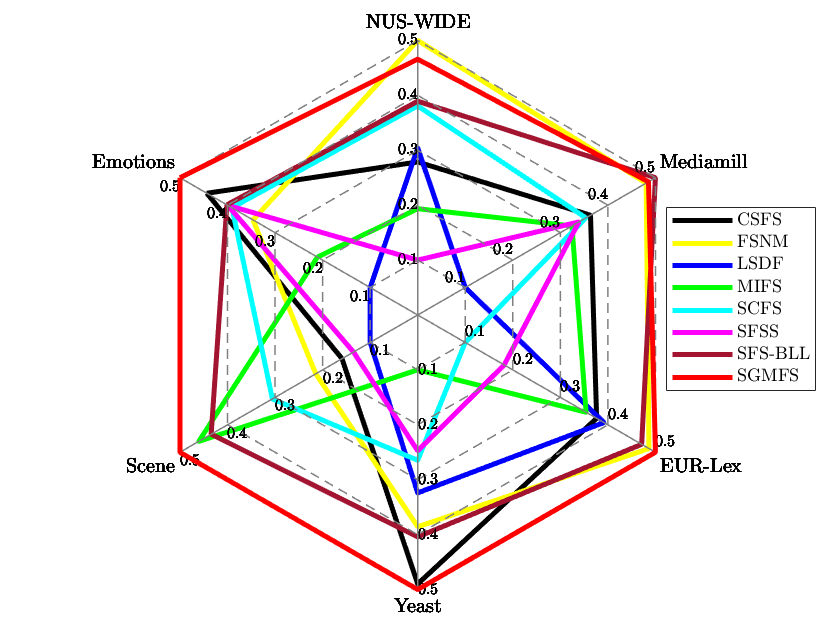}}
	\hspace{-3.1mm}
	\subfigure[$Micro$ $Average$]{
		\label{fig:subfig:d} 
		\includegraphics[width=0.48\columnwidth]{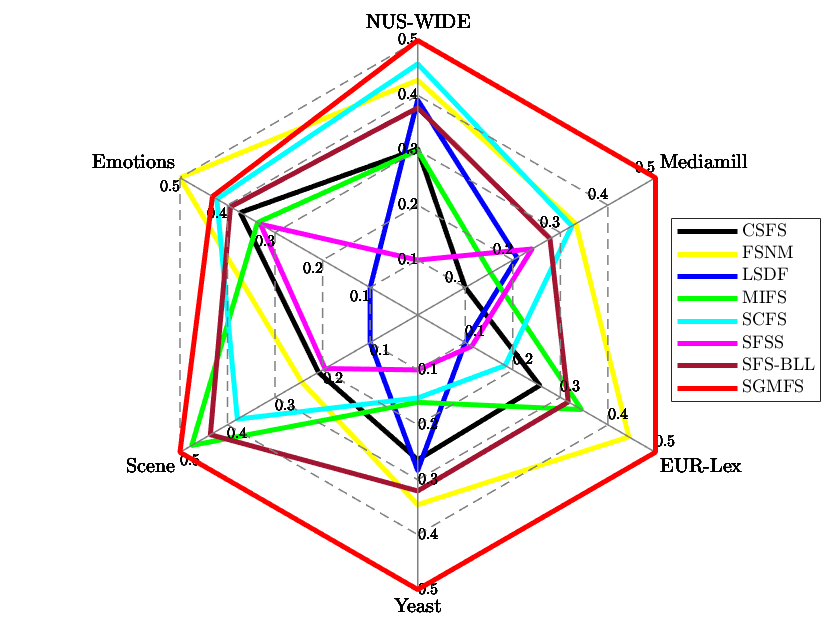}}
        \hspace{-3.1mm}
	\caption{
    Spider web diagrams comparing the stability of SGMFS (red hexagon) with seven state-of-the-art feature selection algorithms across seven multi-label datasets, evaluated using four metrics: $Hamming$ $Loss$, $Ranking$ $Loss$, $Macro$ $Average$, and $Micro$ $Average$.
    }
	\label{fig4} 
\end{figure*}

\begin{figure*}[p]
    \vspace{-2cm}
	\setlength{\abovecaptionskip}{0pt}
	\centering
	\subfigure[Predicted labels $F$]{
		\includegraphics[width=0.45\columnwidth]{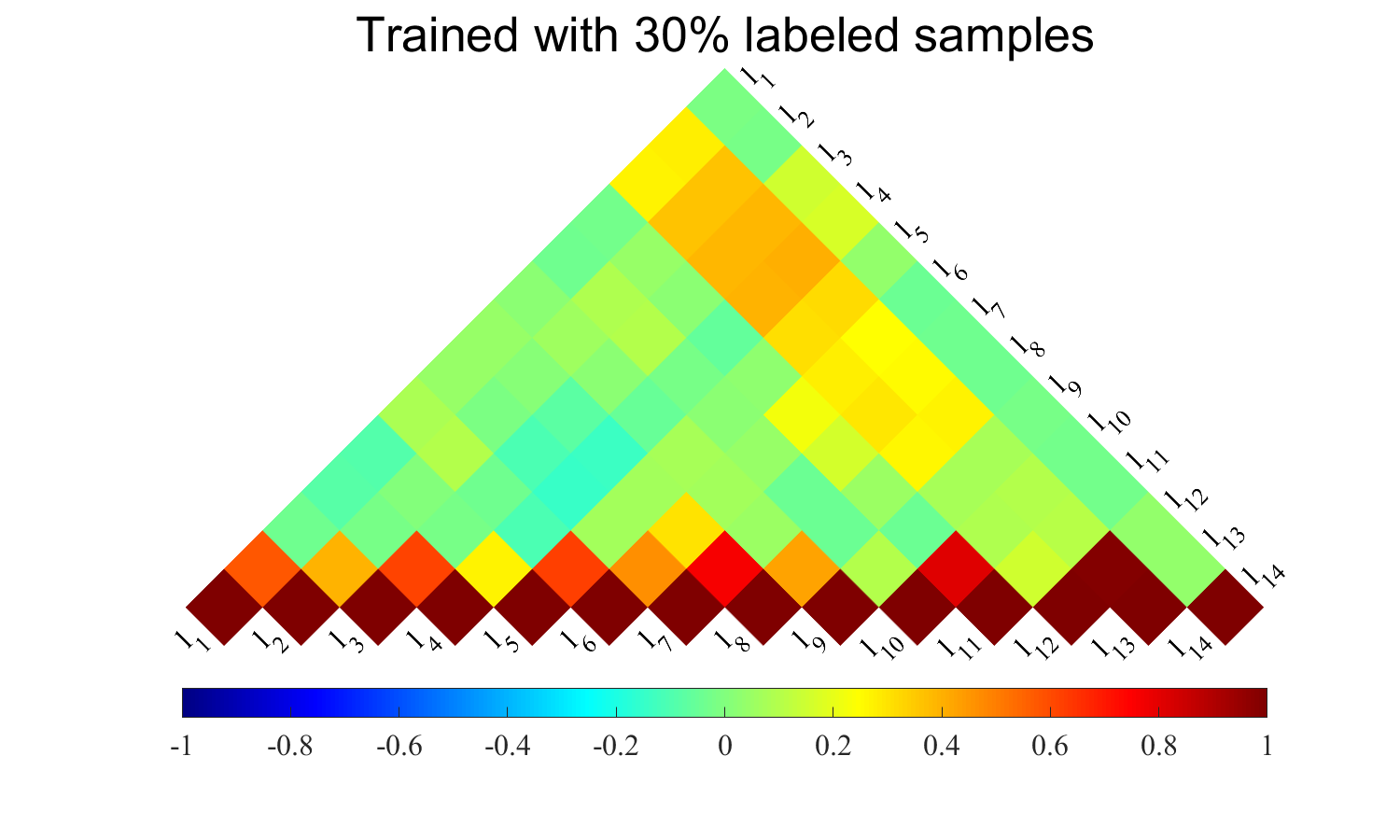}}
	\hspace{-3.1mm}
	\subfigure[Weight matrix $W$]{
		\label{fig:subfig:b} 
		\includegraphics[width=0.45\columnwidth]{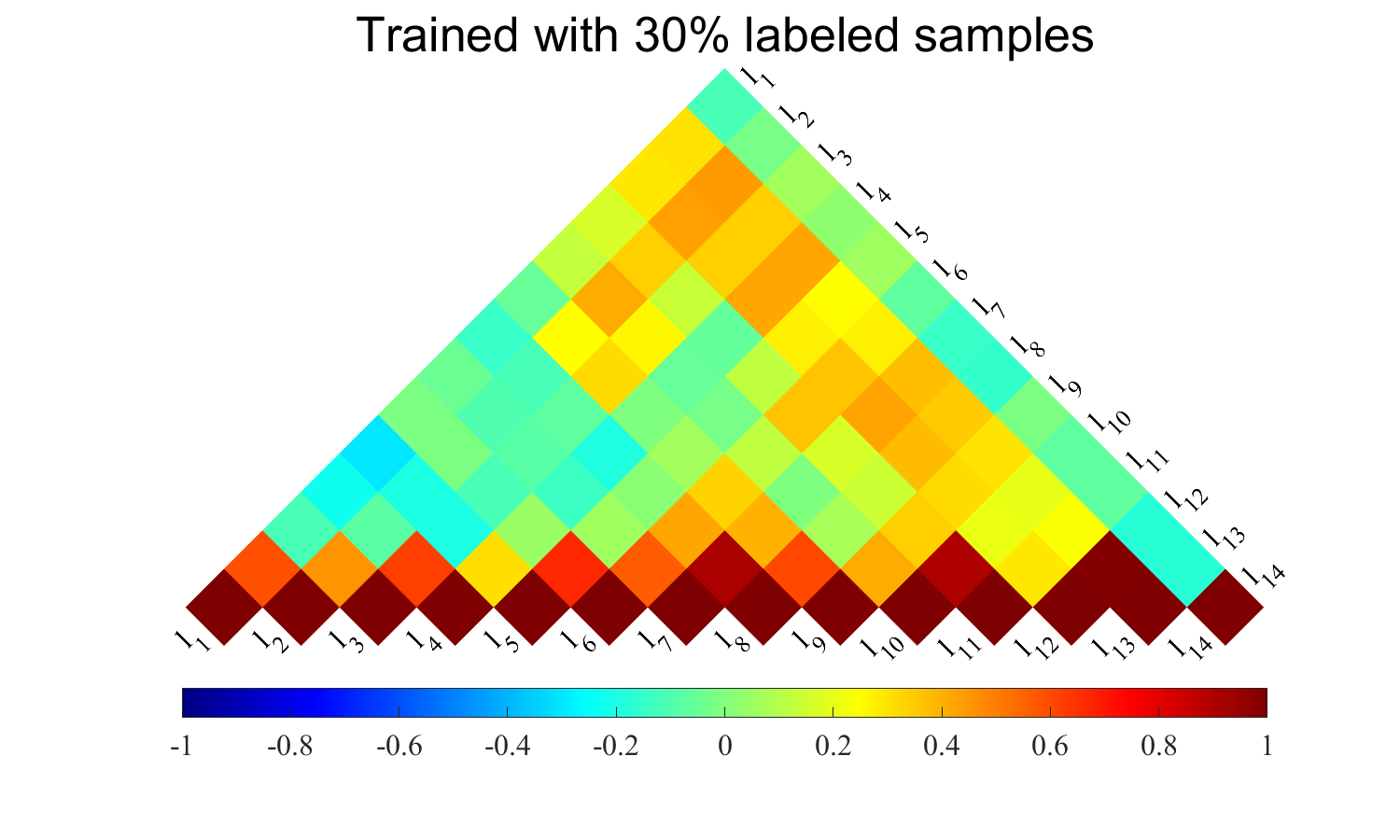}}
	\hspace{-3.1mm}
	\subfigure[GT labels $Y$ of labeled samples]{
		\label{fig:subfig:c} 
		\includegraphics[width=0.45\columnwidth]{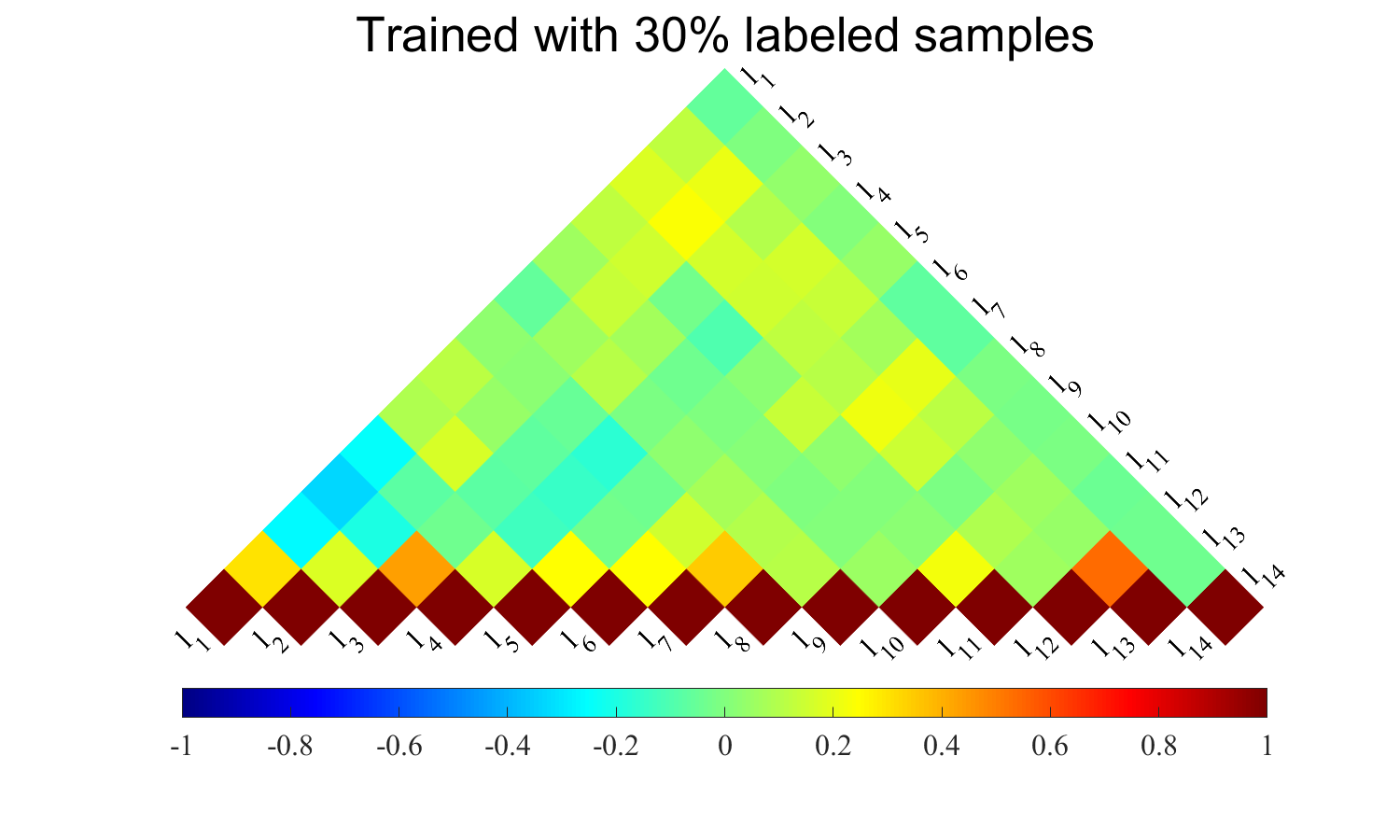}}
	\hspace{-3.1mm}
	\subfigure[Shared labels $Q$]{
		\label{fig:subfig:d} 
		\includegraphics[width=0.45\columnwidth]{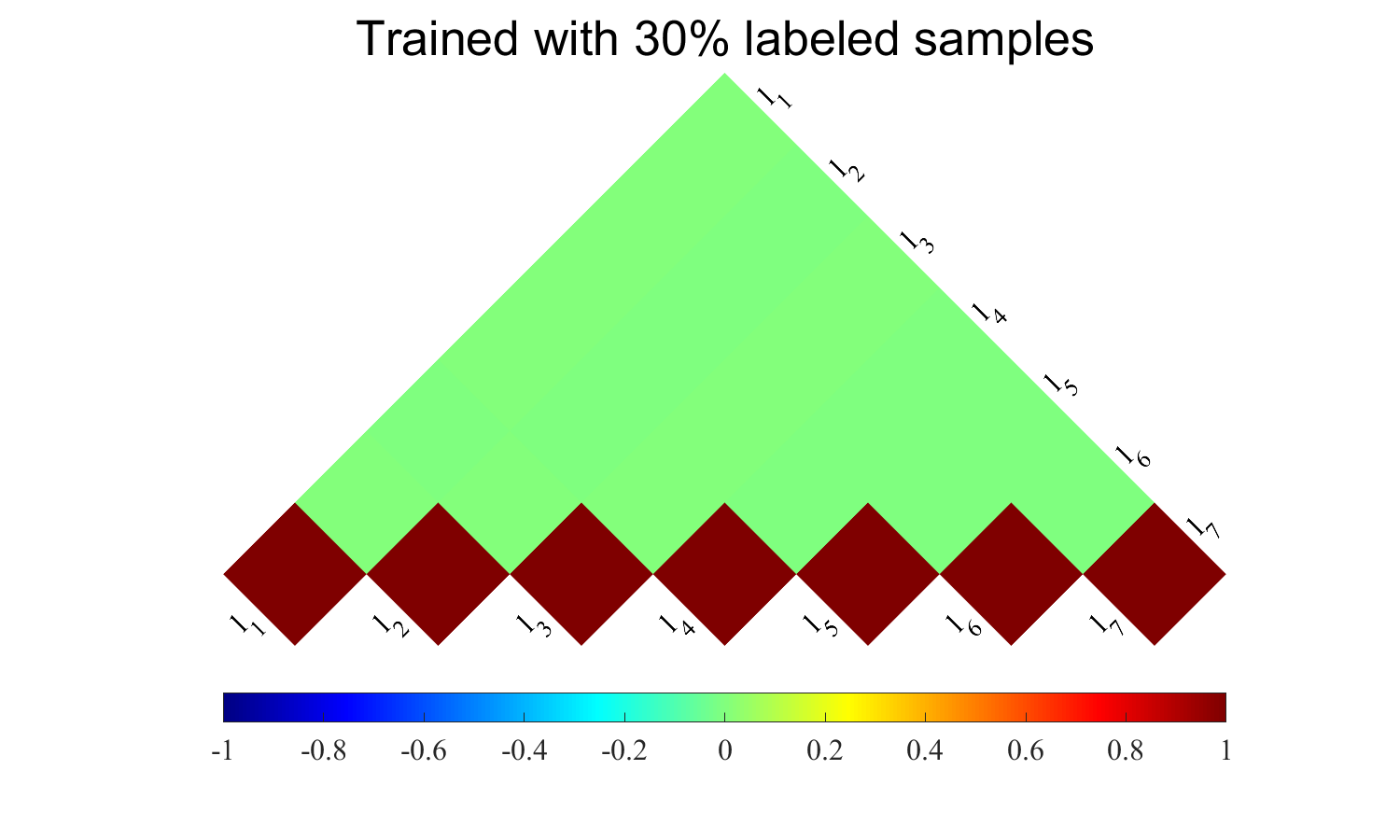}}
    \hspace{-3.1mm}
    \subfigure[Predicted labels $F$]{
		\label{fig:subfig:c} 
		\includegraphics[width=0.45\columnwidth]{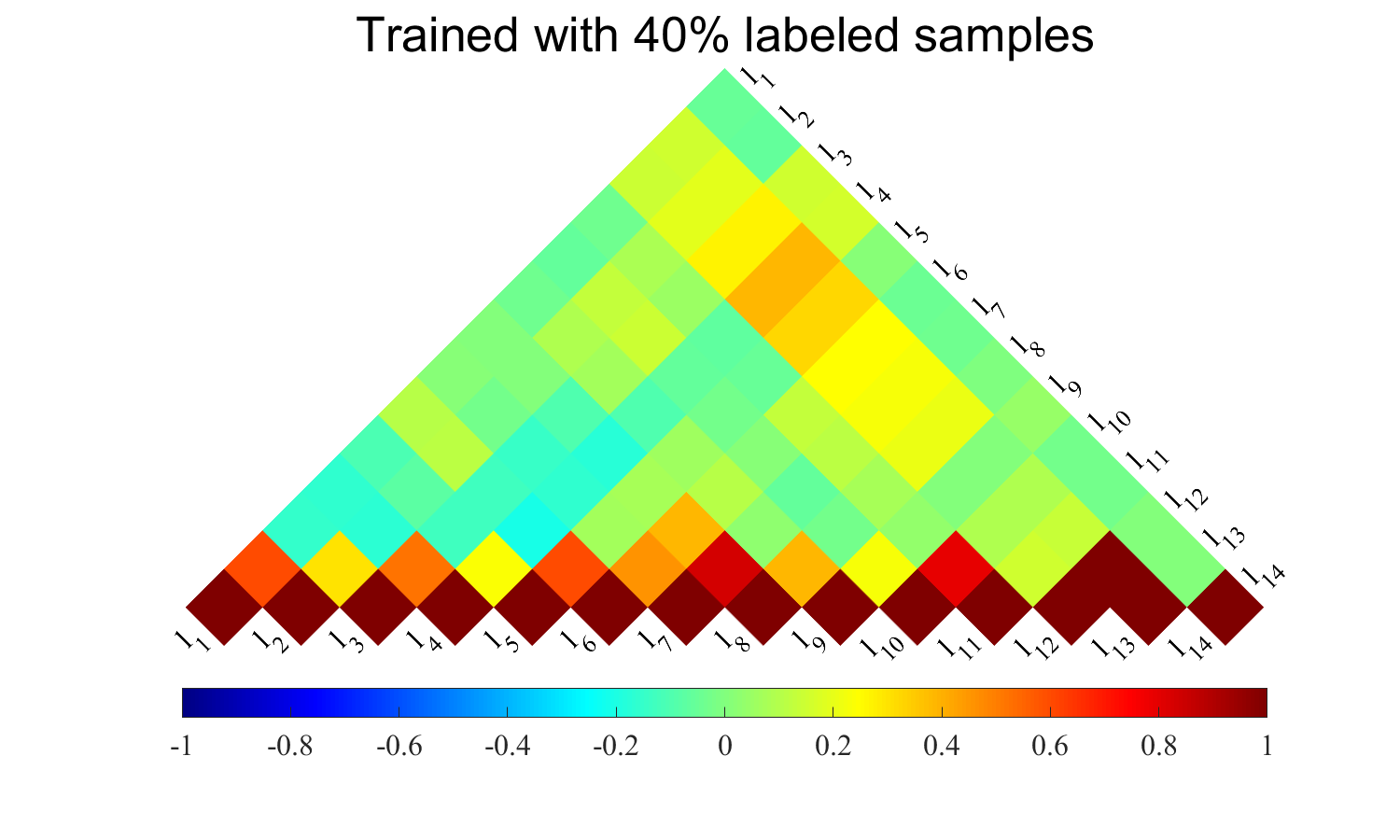}}
	\hspace{-3.1mm}
	\subfigure[Weight matrix $W$]{
		\label{fig:subfig:d} 
		\includegraphics[width=0.45\columnwidth]{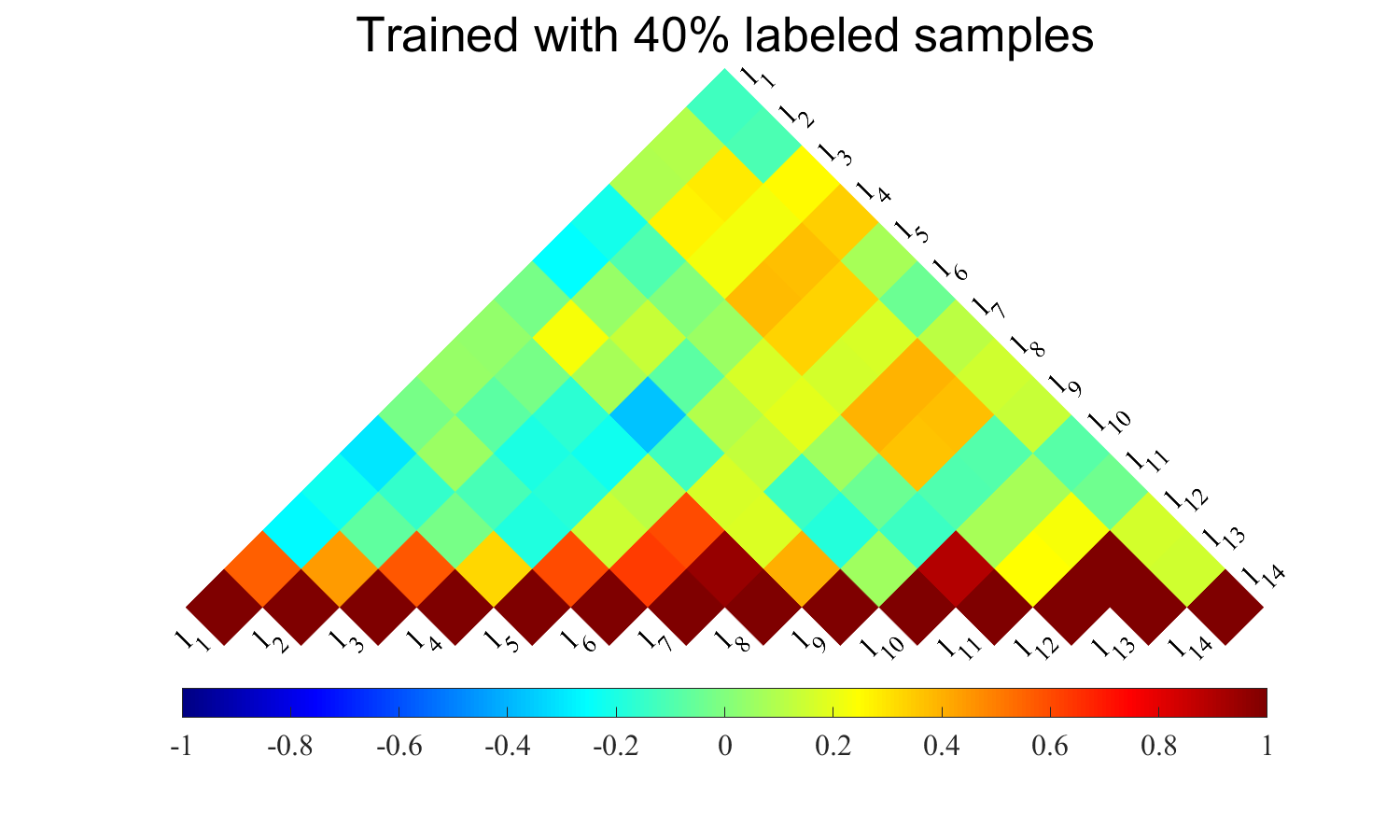}}
	\hspace{-3.1mm}
	\subfigure[GT labels $Y$ of labeled samples]{
		\label{fig:subfig:c} 
		\includegraphics[width=0.45\columnwidth]{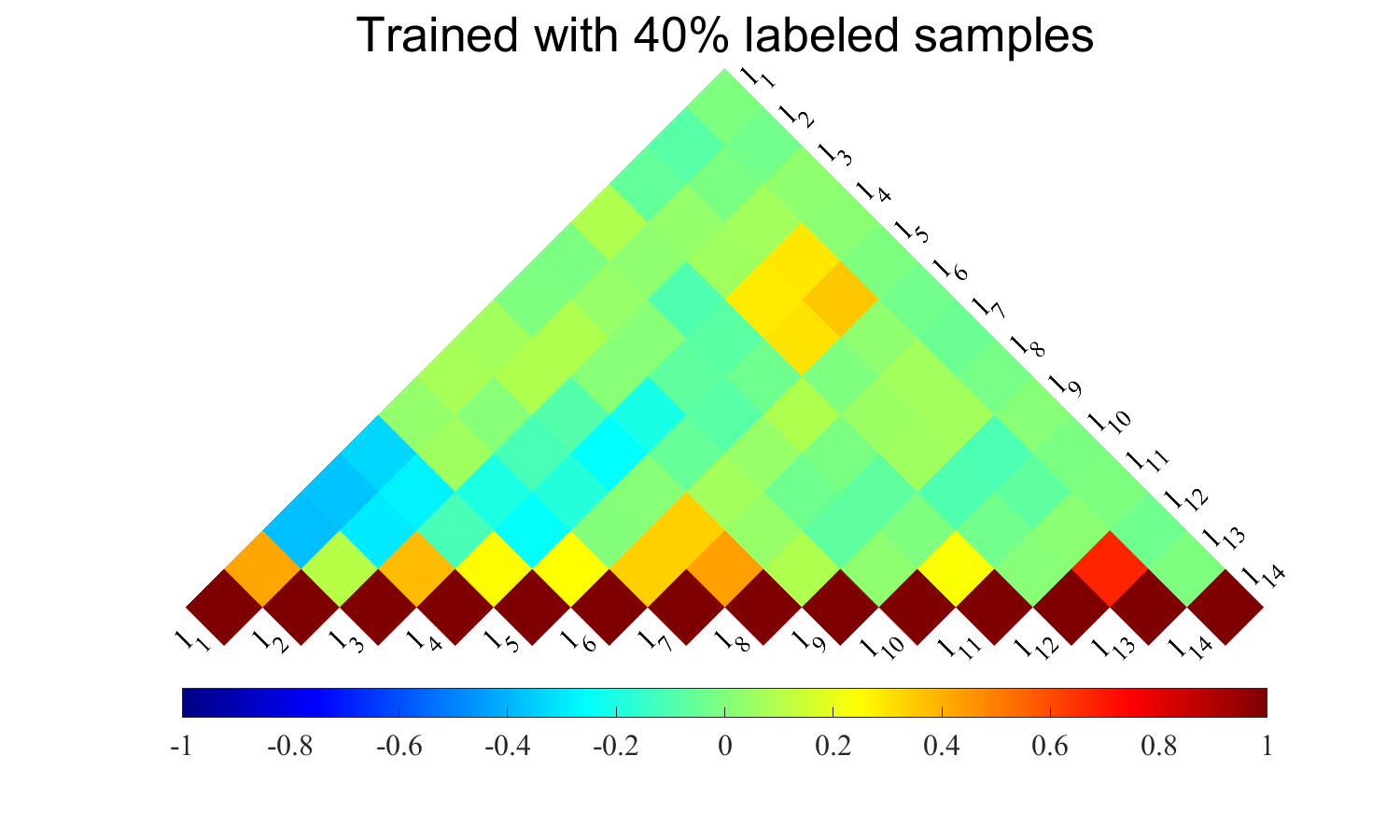}}
	\hspace{-3.1mm}
	\subfigure[Shared labels $Q$]{
		\label{fig:subfig:d} 
		\includegraphics[width=0.45\columnwidth]{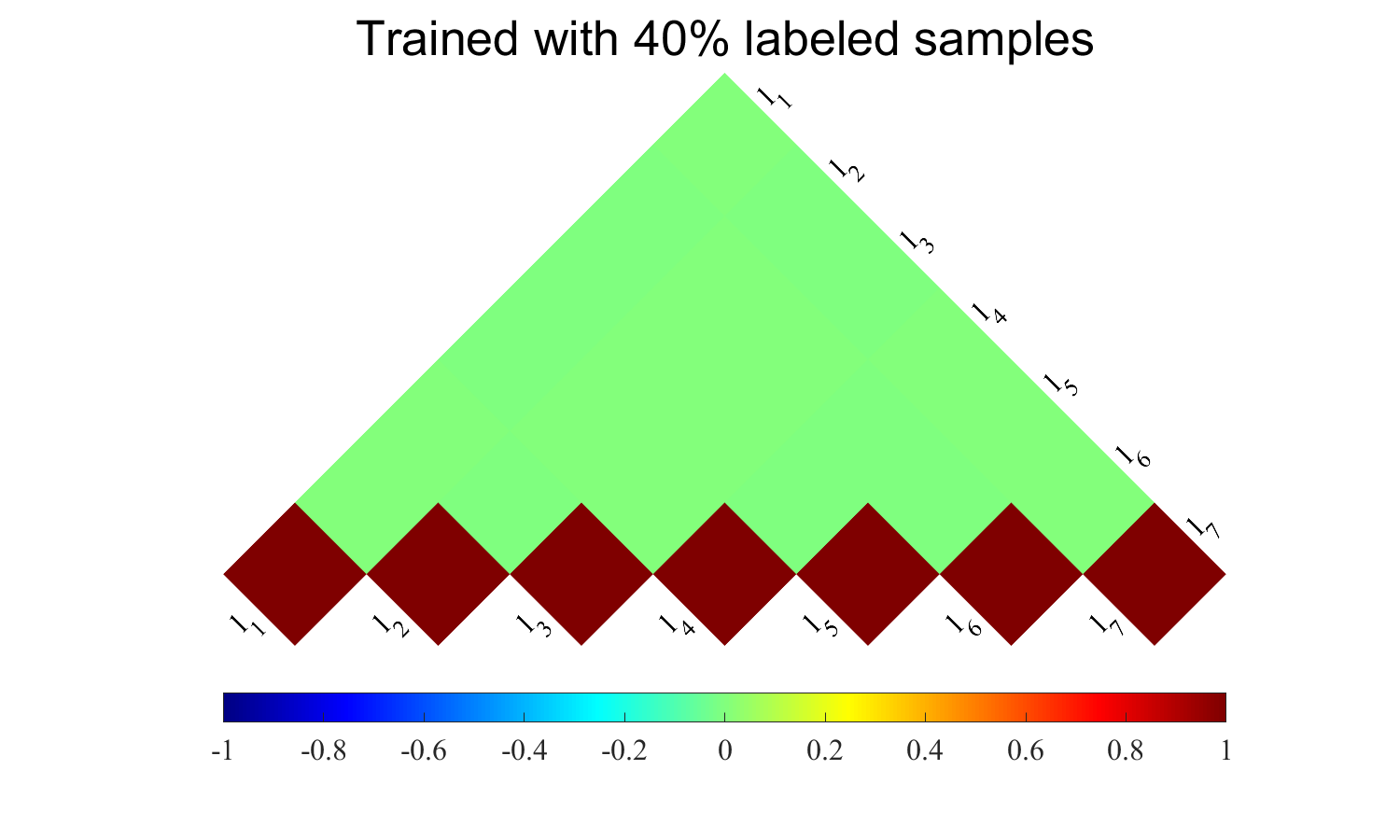}}
	\caption{The triangular heat maps about the distributions of label correlations learned by SGMFS on dataset Emotions with 30\% labeled samples (a, b, c, d) and 40\% labeled samples (e, f, g, h). The label correlations among the predicted label matrix $F$ (a,e), feature weight matrix $W$ (b,f), ground-truth labels $Y$ of labeled samples (c,g), and the shared labels $Q$ (d,h) are compared.
	}
	\label{fig6} 
\end{figure*}

\subsection{Stability Analysis}

To validate the stability of our proposed SGMFS, we extract the best results for $Hamming$ $Loss$, $Ranking$ $Loss$, $Macro$ $Average$, and $Micro$ $Average$ across feature selection methods, using 2\% to 30\% of selected features in each dataset with 20\% labeled samples. The top performance values of $Macro$ $Average$ and $Micro$ $Average$ are then normalized to a common range of $\left[0,0.5\right]$, as the ML-kNN classification results vary significantly across datasets, which could distort the comparison. Since lower values of $Hamming$ $Loss$ and $Ranking$ $Loss$ indicate better performance, we invert the best results of these two metrics before normalization. Finally, we present four spider web diagrams in Fig.\ref{fig4} to illustrate the stability of each feature selection algorithm, where the red polygon represents the performance of SGMFS. From Fig.\ref{fig4}, it is evident that:

(\romannumeral1). SGMFS is closest to standard hexagons on all four metrics, especially for $Ranking$ $Loss$ and $Micro$ $Average$, which illustrates that SGMFS can obtain a much more stable solution on small or large datasets than that of other state-of-the-art feature selection methods.

(\romannumeral2). The hexagons of SGMFS are the biggest, and the vast majority of the metrics of SGMFS can obtain the largest normalized value of 0.5 on most datasets except $Hamming$ $Loss$ and $Micro$ $Average$ on Emotions and $Macro$ $Average$ on NUS-WIDE, which means that SGMFS has the best performance compared with other methods to some extent.

\subsection{Consistency and correlation Analysis}
In order to measure the performance of SGMFS in maintaining feature-label consistency and learning label correlations in the semi-supervised scenario, we assess the label correlations among the predicted label matrix $F$, feature weight matrix $W$, ground-truth labels $Y$ of labeled samples, and the shared labels $Q$. In this regard, each vector in these matrices is normalized from the label dimension, then the correlation values between labels are computed by the cosine similarity of these normalized vectors.

As shown in Fig.~\ref{fig6}, we visualize the distributions of label correlations calculated by each matrix above, and these experiments are conducted on dataset Emotions with 30\% and 40\% labeled samples, from which we can observe that:

(\romannumeral1). The distributions of label correlations computed from $F$ and $Y$ are similar to a certain extent, which proves that SGMFS can learn label dependencies effectively in the semi-supervised scenario. 

(\romannumeral2). The distributions of label correlations computed from $F$ and $W$ are quite similar, which illustrates that the correlation information learned by SGMFS is preserved consistently from feature space and original predicted label space respectively. 

(\romannumeral3). According to the distribution of label correlations obtained by $Q$ in label subspace, each shared label is independent of each other, and this reveals that label subspace learning learns label correlation by reducing dimension and integrating each set of relevant labels to an individual shared label, which is independent to each other shared labels.

\begin{table*}[t]
	\centering
	\renewcommand\arraystretch{1.15}
	\setlength\tabcolsep{2.8pt}
	\footnotesize
	\caption{$Average$ $Precision$ comparison  when 15\% data points are labeled}
    \resizebox{\textwidth}{!}{
	\begin{tabular}{|c|c|c|c|c|c|c|c|c|}
		\hline  
		\multirow{2}*{Datasets}&\multicolumn{8}{|c|}{Algorithms}\\
		\cline{2-9}  
		~&CSFS&SMFS&SSLR&MIFS&SCFS&HSFSG&SFS-BLL&SGMFS (ours)\\
		\hline 
		Emotions&$0.746\pm0.013$&$0.720\pm0.026$&$0.769\pm0.025$&$0.756\pm0.014$&$0.753\pm0.010$&$0.640\pm0.055$&$0.771\pm0.013$&$\bm{0.773}\pm\bm{0.014}$\\
		\hline
		Scene&$0.841\pm0.012$&$0.706\pm0.005$&$0.806\pm0.146$&$0.801\pm0.017$&$0.853\pm0.013$&$0.671\pm0.046$&$0.853\pm0.022$&$\bm{0.856}\pm\bm{0.017}$\\
		\hline
		Yeast&$0.769\pm0.014$&$0.727\pm0.006$&$0.747\pm0.107$&$0.740\pm0.021$&$0.756\pm0.011$&$0.712\pm0.003$&$0.751\pm0.008$&$\bm{0.773}\pm\bm{0.014}$\\
		\hline
		EUR-Lex&$0.611\pm0.012$&$0.604\pm0.012$&$0.617\pm0.031$&$0.627\pm0.022$&$0.587\pm0.006$&$0.582\pm0.017$&$0.598\pm0.024$&$\bm{0.632}\pm\bm{0.033}$\\
		\hline        Mediamill&$0.742\pm0.023$&$0.724\pm0.008$&$0.735\pm0.026$&$0.731\pm0.019$&$\bm{0.744}\pm\bm{0.020}$&$0.740\pm0.033$&$0.736\pm0.015$&$\bm{0.744}\pm\bm{0.022}$\\
		\hline
		NUS-WIDE&$0.519\pm0.022$&$0.494\pm0.014$&$0.507\pm0.132$&$0.513\pm0.012$&$0.526\pm0.009$&$0.522\pm0.024$&$0.514\pm0.018$&$\bm{0.533}\pm\bm{0.102}$\\
		\hline  Plant&$0.485\pm0.011$&$0.516\pm0.006$&$0.482\pm0.045$&$0.524\pm0.029$&$0.519\pm0.015$&$0.449\pm0.025$&$0.467\pm0.020$&$\bm{0.528}\pm\bm{0.021}$\\
		\hline
	\end{tabular}
    }
	\begin{tablenotes}
		\item[1] The optimal results are marked in bold. The number after $\pm$ is the standard deviation.
	\end{tablenotes}
	\label{tab2}
\end{table*}
\begin{table*}[t]
	\centering
	\renewcommand\arraystretch{1.15}
	\setlength\tabcolsep{2.8pt}
	\footnotesize
	\caption{$Average$ $Precision$ comparison when 25\% data points are labeled}
    \resizebox{\textwidth}{!}{
	\begin{tabular}{|c|c|c|c|c|c|c|c|c|}
		\hline  
		\multirow{2}*{Datasets}&\multicolumn{8}{|c|}{Algorithms}\\
		\cline{2-9}  
		
		~&CSFS&SMFS&SSLR&MIFS&SCFS&HSFSG&SFS-BLL&SGMFS (ours)\\
		\hline 
		Emotions&$0.750\pm0.018$&$0.727\pm0.004$&$\bm{0.779}\pm\bm{0.052}$&$0.750\pm0.031$&$0.753\pm0.011$&$0.695\pm0.014$&$0.763\pm0.016$&$0.774\pm0.031$\\
		\hline
		Scene&$0.845\pm0.063$&$0.743\pm0.023$&$0.824\pm0.037$&$0.822\pm0.032$&$0.845\pm0.047$&$0.695\pm0.043$&$0.849\pm0.019$&$\bm{0.862}\pm\bm{0.045}$\\
		\hline
		Yeast&$0.761\pm0.022$&$0.725\pm0.002$&$0.762\pm0.071$&$0.769\pm0.089$&$0.758\pm0.133$&$0.716\pm0.004$&$0.755\pm0.020$&$\bm{0.774}\pm\bm{0.046}$\\
		\hline
		EUR-Lex&$0.610\pm0.185$&$0.617\pm0.007$&$0.621\pm0.053$&$0.627\pm0.065$&$0.579\pm0.125$&$0.604\pm0.013$&$0.606\pm0.011$&$\bm{0.633}\pm\bm{0.043}$\\
		\hline
		Mediamill&$0.744\pm0.075$&$0.732\pm0.012$&$0.736\pm0.035$&$0.726\pm0.021$&$0.731\pm0.023$&$0.742\pm0.029$&$0.739\pm0.013$&$\bm{0.757}\pm\bm{0.034}$\\
		\hline
		NUS-WIDE&$0.522\pm0.019$&$0.511\pm0.006$&$0.523\pm0.040$&$0.499\pm0.033$&$0.528\pm0.031$&$0.531\pm0.015$&$0.516\pm0.021$&$\bm{0.538}\pm\bm{0.117}$\\
        \hline
		Plant&$0.500\pm0.005$&$0.528\pm0.003$&$0.506\pm0.031$&$0.521\pm0.017$&$0.525\pm0.028$&$0.512\pm0.019$&$0.495\pm0.009$&$\bm{0.531}\pm\bm{0.086}$\\
		\hline
	\end{tabular}
    }
	\begin{tablenotes}
		\item[1] The optimal results are marked in bold. The number after $\pm$ is the standard deviation.
	\end{tablenotes}
	\label{tab3}
\end{table*}
\begin{table*}[!t]
	\centering
	\renewcommand\arraystretch{1.15}
	\setlength\tabcolsep{2.8pt}
	\footnotesize
	\caption{$Average$ $Precision$ comparison when 35\% data points are labeled}
    \resizebox{\textwidth}{!}{
	\begin{tabular}{|c|c|c|c|c|c|c|c|c|}
		\hline  
		\multirow{2}*{Datasets}&\multicolumn{8}{|c|}{Algorithms}\\
		\cline{2-9}  
		
		~&CSFS&SMFS&SSLR&MIFS&SCFS&HSFSG&SFS-BLL&SGMFS (ours)\\
		\hline 
		Emotions&$0.758\pm0.022$&$0.756\pm0.009$&$\bm{0.788}\pm\bm{0.029}$&$0.769\pm0.016$&$0.762\pm0.035$&$0.702\pm0.014$&$0.776\pm0.012$&$0.781\pm 0.049$\\
		\hline
		Scene&$0.851\pm0.034$&$0.744\pm0.015$&$0.822\pm0.154$&$0.844\pm0.144$&$0.842\pm0.015$&$0.735\pm0.021$&$0.856\pm0.013$&$\bm{0.859}\pm\bm{0.071}$\\
		\hline
		Yeast&$0.764\pm0.020$&$0.726\pm0.004$&$0.765\pm0.082$&$0.771\pm0.026$&$0.761\pm0.074$&$0.731\pm0.012$&$0.761\pm0.021$&$\bm{0.777}\pm\bm{0.034}$\\
		\hline
		EUR-Lex&$0.603\pm0.054$&$0621\pm0.006$&$0.630\pm0.032$&$0.611\pm0.201$&$0.601\pm0.052$&$0.597\pm0.016$&$0.618\pm0.019$&$\bm{0.642}\pm\bm{0.049}$\\
		\hline
		Mediamill&$0.745\pm0.146$&$0.748\pm0.012$&$0.752\pm0.052$&$0.725\pm0.068$&$0.749\pm0.109$&$0.738\pm0.025$&$0.741\pm0.017$&$\bm{0.758}\pm\bm{0.024}$\\
		\hline
		NUS-WIDE&$0.529\pm0.112$&$0.536\pm0.020$&$0.538\pm0.039$&$0.512\pm0.095$&$0.533\pm0.130$&$0.522\pm0.018$&$0.525\pm0.014$&$\bm{0.541}\pm\bm{0.106}$\\
		\hline
        Plant&$0.511\pm0.104$&$0.539\pm0.036$&$0.513\pm0.036$&$0.527\pm0.114$&$0.520\pm0.032$&$0.523\pm0.015$&$0.513\pm0.022$&$\bm{0.543}\pm\bm{0.074}$\\
		\hline
	\end{tabular}
    }
	\begin{tablenotes}
		\item[1] The optimal results are marked in bold. The number after $\pm$ is the standard deviation.
	\end{tablenotes}
	\label{tab4}
\end{table*}

\subsection{Performance Evaluation with different proportion of labeled data}
Since semi-supervised feature selection methods have different performances with different proportions of labeled data, we use 25\% and 35\% labeled data to compare SGMFS with other feature selection methods, which are shown in Tables~\ref{tab2},~\ref{tab3} and~\ref{tab4} respectively. 
It is noteworthy that in this subsection of experiments, we not only conducted validation using the first six benchmark datasets listed in Table~\ref{tab1}, but also specifically introduced the Plant dataset as a supplementary experimental subject. To thoroughly validate the effectiveness of our proposed method, we compared our method with extra three state-of-the-art semi-supervised multi-label feature selection approaches, including SMFS~\cite{sheikhpour2025robust}, SSLR~\cite{zhao2024sparse}, and HSFSG~\cite{sheikhpour2023hessian}. These extended experimental results can further demonstrate the enhanced effectiveness and robustness of our proposed method.

In order to demonstrate the results more accurately, we use another example-based $Average$ $Precision$ metric.
The larger $Average$ $Precision$ is, the better the corresponding methods are \cite{tang2014feature}.

In our experiments, we compute the average value of $Average$ $Precision$ metric with 2\% to 30\% selected features for each method first, which is calculated for 10 times, and then we obtain the final average value and standard deviation to evaluate the performances. According to  Tables~\ref{tab2},~\ref{tab3} and~\ref{tab4}, we can observe that:

(\romannumeral1). SGMFS has the max $Average$ $Precision$ on all three small datasets with all three different proportions of labeled data and achieves the best performance upon most occasions on the other three large datasets. Overall, SGMFS can persistently get better performances than other state-of-the-art semi-supervised methods.

(\romannumeral2). The $Average$ $Precision$ values of SGMFS have lesser improvements on some datasets such as Emotions, Scene, and NUS-WIDE, while $Average$ $Precision$ increases obviously in other methods with the proportion of labeled data improving. Therefore, SGMFS is superior to other methods when there are only a few training samples are labeled, which illustrates the competitive advantage of sparse graph learning for semi-supervised feature selection.

(\romannumeral3). The effect of SGMFS is gradually improved as the proportion of labeled samples increases, while some other semi-supervised methods cannot support this property. For example, the $Average$ $Precision$ value of SCFS with 35\% labeled samples is less than that with 25\% labeled samples.

In summary, SGMFS has a competitive superiority with different proportions of labeled data.

\begin{table}[t]
	\centering
	\renewcommand\arraystretch{1.15}
	\setlength\tabcolsep{2.8pt}
	\footnotesize
	\caption{Ablation study of $Average$ $Precision$ comparison ($\bm$Standard Deviation) when 35\% data points are labeled}
	\begin{tabular}{|c|c|c|c|}
		\hline  
		\multirow{2}*{Algorithms}&\multicolumn{3}{|c|}{Datasets}\\
		\cline{2-4}  
		~&Emotions&Scene&Yeast \\
		\hline 
		SGMFS\textbackslash$lsc$ &$0.732\pm0.013$&$0.803\pm0.022$&$0.713\pm0.086$\\
		\hline
		SGMFS\textbackslash$sc$ &$0.764\pm0.093$&$0.836\pm0.017$&$0.748\pm0.052$\\
		\hline
		SGMFS\textbackslash$lc$ &$0.771\pm0.042$&$0.828\pm0.109$&$0.762\pm0.074$\\
		\hline
		SGMFS &$\bm{0.781}\pm\bm{0.049}$&$\bm{0.859}\pm\bm{0.071}$&$\bm{0.777}\pm\bm{0.034}$\\
		\hline
	\end{tabular}
	\begin{tablenotes}
		\item[1] The best results are highlighted in bold.
	\end{tablenotes}
	\label{tab5}
\end{table}

\subsection{Ablation Study}
To further investigate whether the effectiveness of our methods is derived from label correlation learning and space consistency preserving, we compare the impact of different components of the proposed SGMFS in this experiment. To be specific, we compare the $Average$ $Precision$ performances of SGMFS\textbackslash$lsc$, SGMFS\textbackslash$lc$, SGMFS\textbackslash$sc$ and SGMFS, where SGMFS\textbackslash$lsc$ denotes training our SGMFS model without considering both the label correlations and space consistency, and SGMFS\textbackslash$lc$ and SGMFS\textbackslash$sc$ mean training SGMFS without label correlations learning and without space consistency preserving respectively. In other words, we set $\alpha = \beta = 0$ and $\Omega\left(\textbf{W},\textbf{M}\right)=\|\textbf{W}\|_{2,1}$ in SGMFS\textbackslash$lsc$, which is the same as Eq.~(\ref{eq4}). $\beta = 0$ and $\Omega\left(\textbf{W},\textbf{M}\right)=\|\textbf{W}\|_{2,1}$ in SGMFS\textbackslash$sc$ and  $\alpha = 0$ in SGMFS\textbackslash$lc$. The results are recorded in Table~\ref{tab5}.

According to Table~\ref{tab5}, we can observe that: SGMFS\textbackslash$lc$ and SGMFS\textbackslash$sc$ perform better than SGMFS\textbackslash$lsc$, yet perform worse than SGMFS, which proves the effectiveness of label correlations learning and space consistency preserving during training with Algorithm~\ref{Alg1}.

\begin{figure*}[p]
    \vspace{-2.5cm}
	\setlength{\abovecaptionskip}{0pt}
	\centering
	\subfigure[Step 25 iteration on Emotions]{
		\label{fig:subfig:a} 
		\includegraphics[width=0.42\columnwidth]{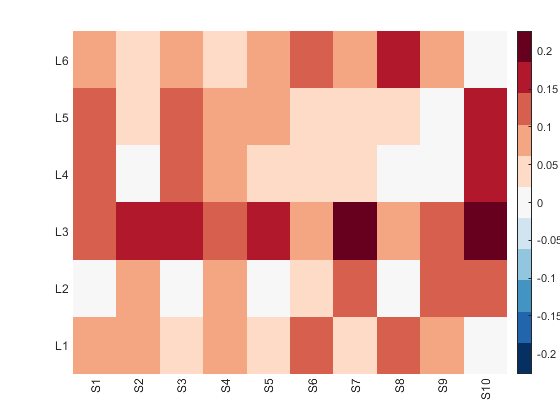}}
	\hspace{3.1mm}
	\subfigure[Step 50 iteration on Emotions]{
		\label{fig:subfig:b} 
		\includegraphics[width=0.42\columnwidth]{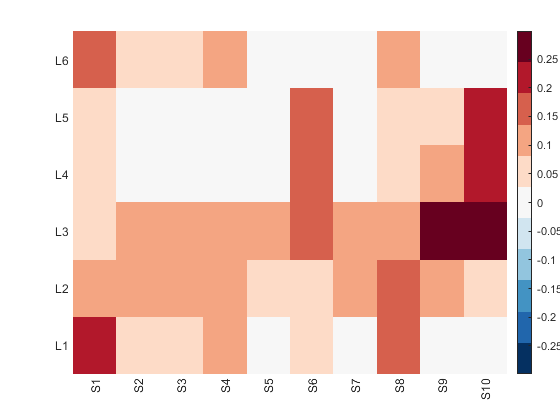}}
	\hspace{3.1mm}
	\subfigure[Step 75 iteration on Emotions]{
		\label{fig:subfig:c} 
		\includegraphics[width=0.42\columnwidth]{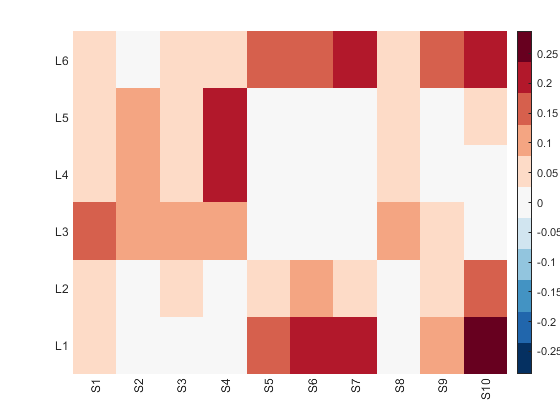}}
	\hspace{3.1mm}
	\subfigure[Step 100 iteration on Emotions]{
		\label{fig:subfig:d} 
		\includegraphics[width=0.42\columnwidth]{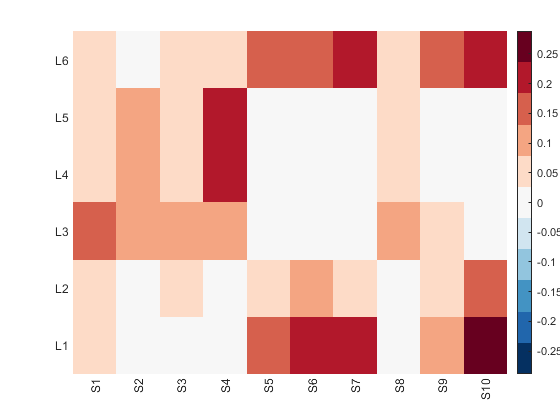}}
    \hspace{3.1mm}
	\subfigure[Step 25 iteration on Scene]{
		\label{fig:subfig:a} 
		\includegraphics[width=0.42\columnwidth]{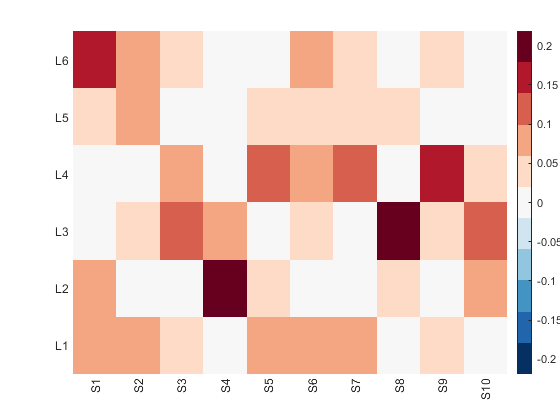}}
	\hspace{3.1mm}
	\subfigure[Step 50 iteration on Scene]{
		\label{fig:subfig:b} 
		\includegraphics[width=0.42\columnwidth]{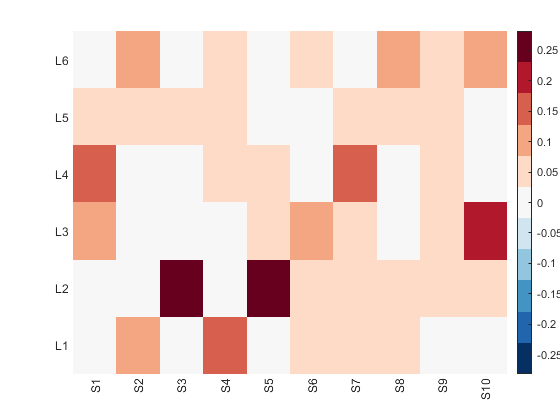}}
	\hspace{3.1mm}
	\subfigure[Step 75 iteration on Scene]{
		\label{fig:subfig:c} 
		\includegraphics[width=0.42\columnwidth]{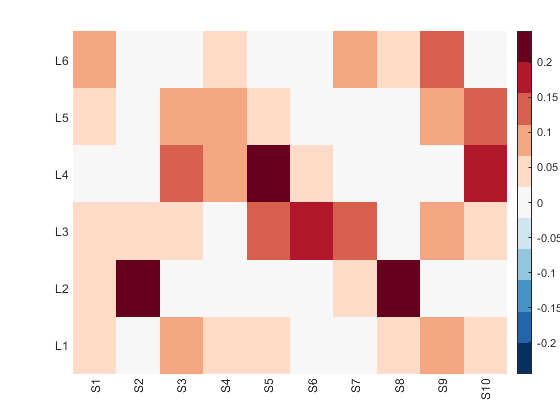}}
	\hspace{3.1mm}
	\subfigure[Step 100 iteration on Scene]{
		\label{fig:subfig:d} 
		\includegraphics[width=0.42\columnwidth]{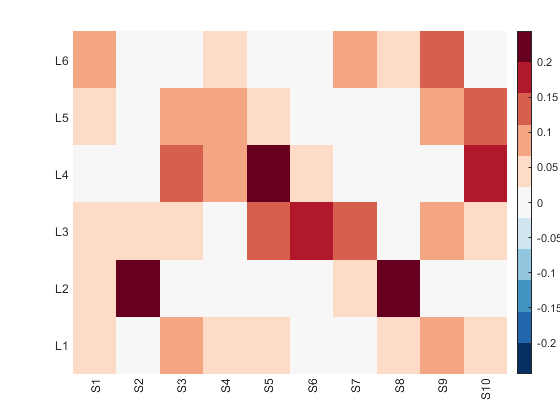}}
	\caption{The visual representation of the predicted label matrix for 10 unlabeled training samples throughout the iterative process of Algorithm~\ref{Alg1} on the Emotions dataset (a-d) and the Scene dataset (e-h).}
	\label{fig9} 
\end{figure*}

\subsection{Visual Analysis of Soft-labels}
Figure \ref{fig9} illustrates the visual representation of predicted label matrices for 10 randomly selected unlabeled training samples throughout the iterative process detailed in Algorithm~\ref{Alg1}, implemented on the Emotions and Scene datasets, respectively. Each cell in the predicted label matrix is displayed as a 6$\times$10 pixel block, where the horizontal axis indicates different samples and the vertical axis represents various labels. Darker colors denote higher weights, while white corresponds to 0 weight. All parameters are set to median values within the tuned range. It is noticeable that in the 25th iteration, the predicted label matrix contains several mixed elements. Subsequent iterations lead to a gradual increase in white cells, indicating a clearer structure. By the 75th iteration, the heatmap stabilizes, signifying convergence.

\begin{figure}[t]
	\setlength{\abovecaptionskip}{0pt}
	\centering
	\subfigure[Emotions ($\alpha$)]{
		\label{fig:subfig:a} 
		\includegraphics[width=0.48\columnwidth]{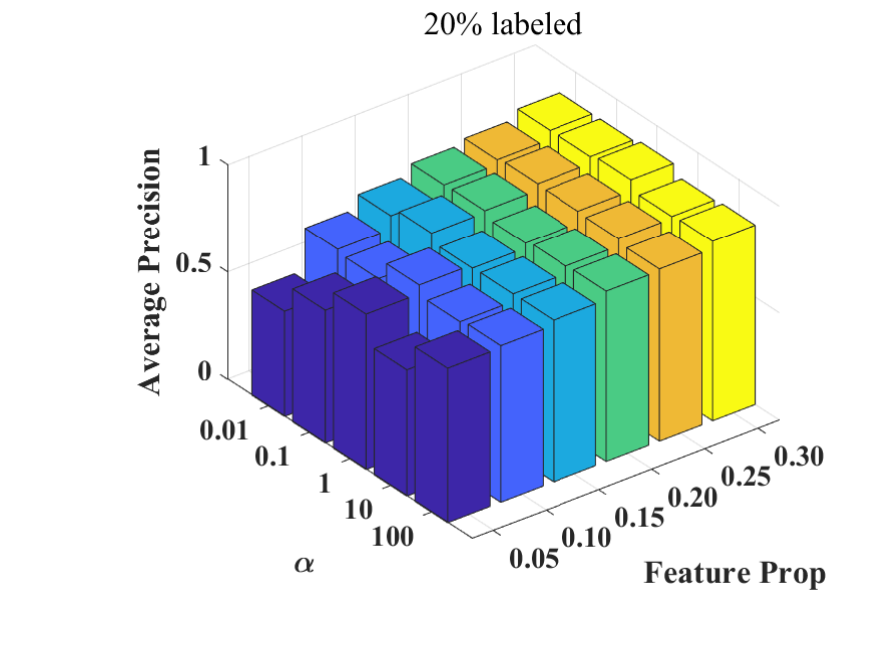}}
	\hspace{-2.1mm}
	\subfigure[Emotions ($\beta$)]{
		\label{fig:subfig:a} 
		\includegraphics[width=0.48\columnwidth]{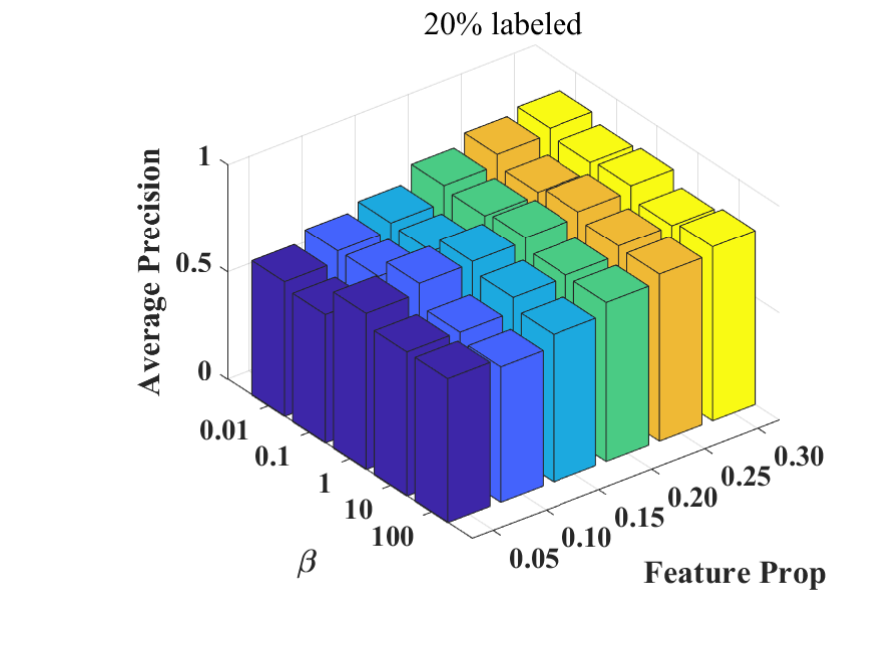}}
	\hspace{-2.1mm}
	\hspace{-2.1mm}
	\subfigure[Emotions ($\gamma$)]{
		\label{fig:subfig:a} 
		\includegraphics[width=0.48\columnwidth]{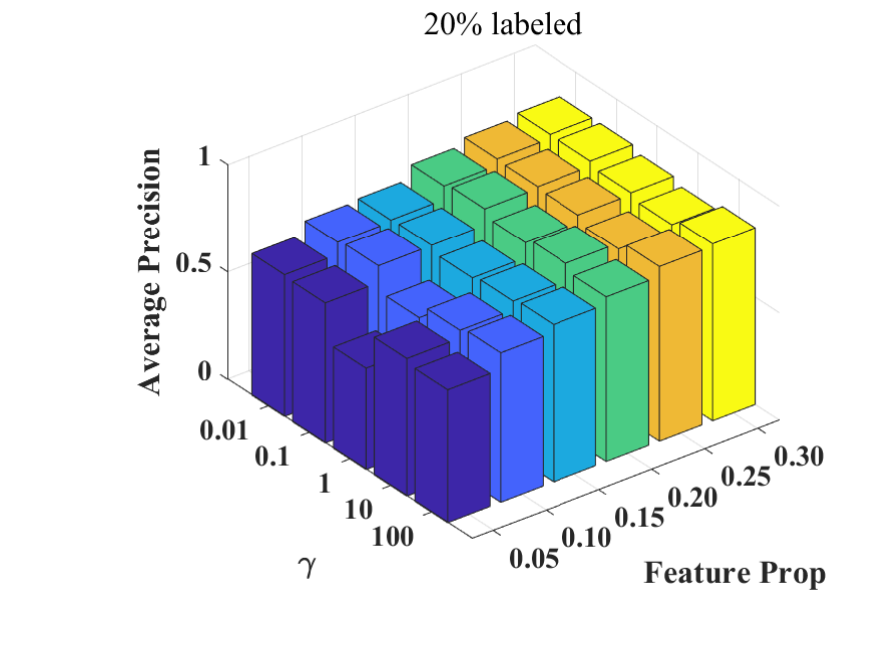}}
	\hspace{-2.1mm}
	\subfigure[Emotions ($lsd$)]{
		\label{fig:subfig:a} 
		\includegraphics[width=0.48\columnwidth]{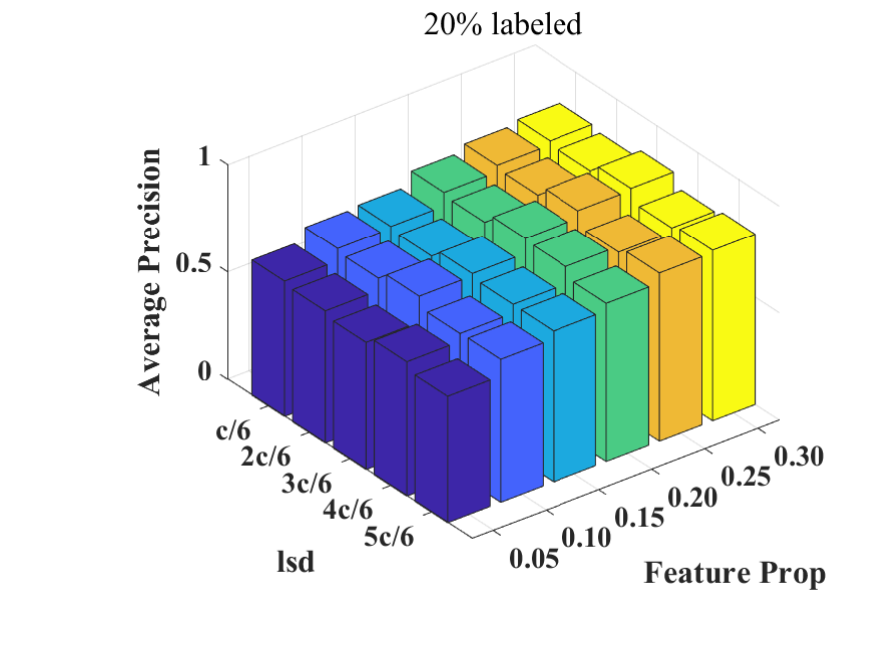}}
	\hspace{-2.1mm}
	\hspace{-2.1mm}
	\caption{Sensitivity analysis about parameters $\alpha$, $\beta$, $\gamma$ and $lsd$ on dataset Emotions (20\% labeled samples).} 
	\label{fig5} 
\end{figure}

\subsection{Sensitivity Analysis about Parameters}

In this subsection, we study the parameter sensitivity of SGMFS through tuning regularization parameters $\alpha$, $\beta$, $\gamma$ among $\left\{0.01,0.1,1,10,100\right\}$ and the dimension of shared label subspace $lsd$ in $\{c/6,2c/6,3c/6,4c/6,5c/6\}$ respectively, where $c$ is the dimension of original label space. For each parameter, we set other regularization parameters being 1 and $lsd$ being $c/2$, then the sensitivity analysis diagrams are drawn to show the variation of $Average$ $Precision$ with feature proportion changed in $\left\{0.05,0.1,0.15,0.20,0.25,0.30\right\}$, which are demonstrated in Fig.~\ref{fig5}. Due to the space limitation, we show the results on datasets Emotions, and results on Scene are shown in \ref{apen4}.

According to Fig.~\ref{fig5}, we can observe that: (\romannumeral1). Regularization parameters $\alpha$, $\beta$ and $\gamma$ are not sensitive in Algorithm~\ref{Alg1}, especially with a higher proportion of selected features, since the percentages of 5\%-10\% for feature selection may lead to some key feature information of datasets missing. (\romannumeral2). From (d) of Fig.~\ref{fig5}, $lsd$ could influence the performance of SGMFS to some extent, which illustrates the non-trivial role of shared label subspace learning and label correlations in multi-label feature selection. And the $Average$ $Precision$ of SGMFS with $lsd=c/2$ is slightly higher than that with $lsd=c/6$ and $lsd=5c/6$. That is the mezzo value is propitious to the dimension of shared label subspace in SGMFS.

In general, the four parameters are not sensitive to SGMFS, and the best dimension for label subspace to explore label correlations around $c/2$, which is different in different datasets. The same conclusions can be drawn from the other four datasets.

\begin{figure}[t]
	\setlength{\abovecaptionskip}{0pt}
	\centering
	\subfigure[20\% labeled on Yeast]{
		\label{fig:subfig:a} 
		\includegraphics[width=0.42\columnwidth]{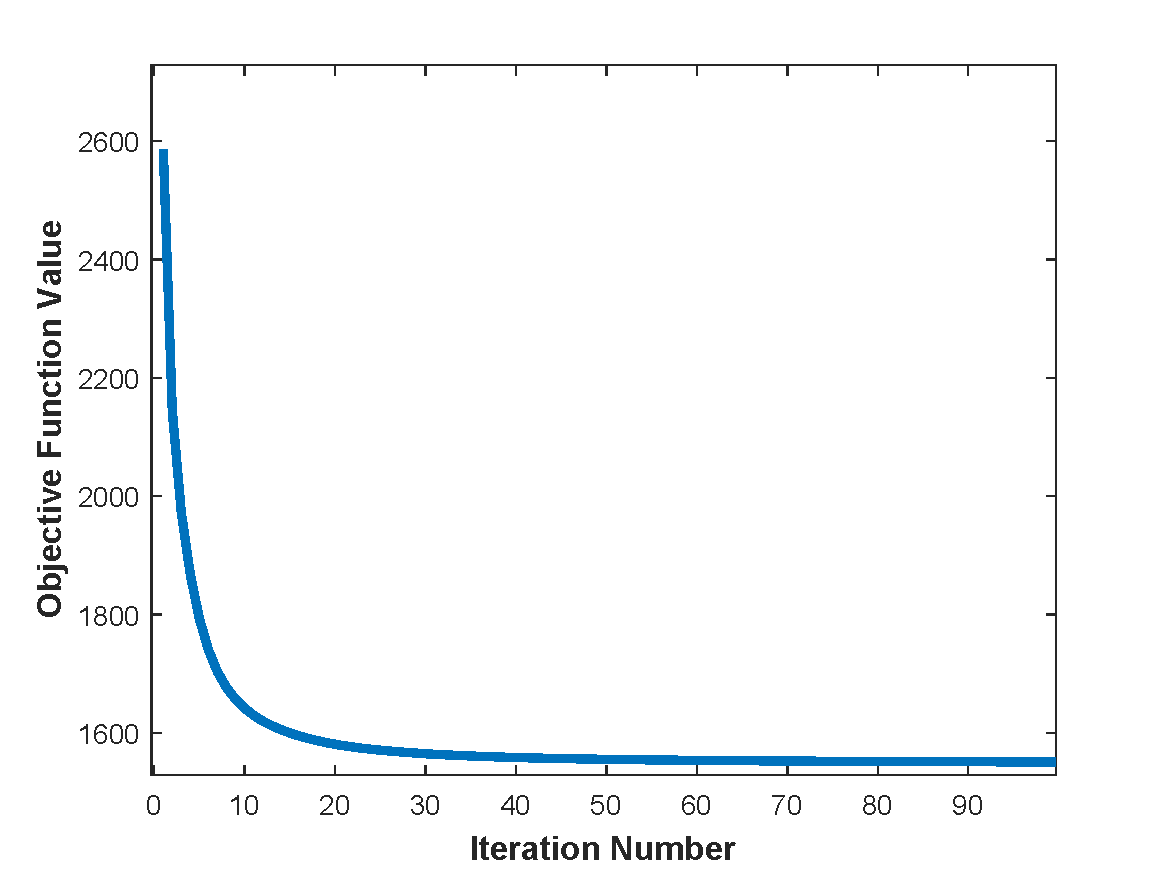}}
	\subfigure[40\% labeled on Yeast]{
		\label{fig:subfig:b} 
		\includegraphics[width=0.42\columnwidth]{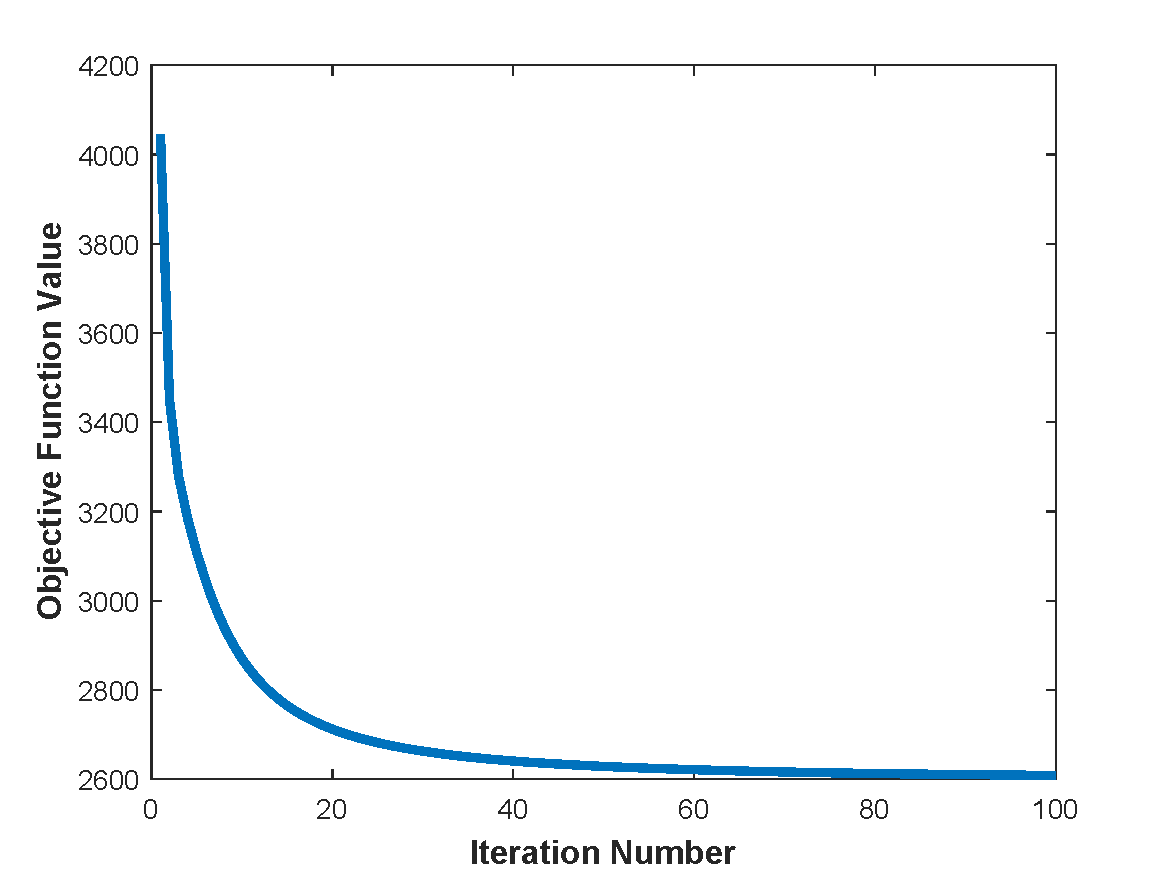}}
	\caption{The convergence curves on dataset Yeast with (a) 20\% labeled data and (b) 40\% labeled data.}
	\label{fig7} 
\end{figure}

\subsection{Convergence Study}

To demonstrate the rapid convergence of SGMFS, we calculate the objective function value from Eq.(\ref{eq11}) using the Yeast dataset as an example, setting $\alpha=1$, $\beta=1$, $\gamma=1$, and $lsd=c/2$. The convergence curves are shown in Fig.\ref{fig7} for (a) 20\% labeled data and (b) 40\% labeled data. The results indicate that the objective function converges in approximately 20 iterations. Similar findings are observed across other datasets with varying proportions of labeled samples, providing visual confirmation of Theorem~\ref{theo2}.

\section{Conclusion}
\label{S6}

This paper integrates semi-supervised feature selection, multi-label learning, and sparse graph learning into a unified framework, proposing an efficient method named SGMFS. The optimization algorithm of SGMFS exhibits rapid convergence, enabling it to handle both small and large datasets effectively. Compared to other semi-supervised methods, SGMFS offers two key advantages for performance enhancement:
(\romannumeral1) It introduces a novel shared label subspace learning approach, which effectively captures label correlations in semi-supervised settings.
(\romannumeral2) It ensures consistency between the label and feature spaces by adaptively identifying sparse neighbors for each sample during soft multi-label learning, without relying on inaccurate prior knowledge, thereby enhancing multi-label feature selection performance.
Comprehensive experiments demonstrate the effectiveness and robustness of SGMFS against seven state-of-the-art methods in supervised and semi-supervised multi-label feature selection. Future research will focus on developing semi-supervised methods capable of automatically determining optimal feature selection proportions and addressing multi-view scenarios.

%



\appendix

\section{Complexity Analysis of Computing $\textbf{W}$}
\label{apen1}
According to  Algorithm~\ref{Alg1}, the optimal $\textbf{W}$ is obtained by:
\begin{equation*}
	\begin{array}{l}
		2\left(\textbf{S}-\alpha \textbf{XQQ}^{T} \textbf{X}^{T}\right) \textbf{W}=2 \textbf{XHF} \\
		\Rightarrow \textbf{W}=\left(\textbf{S}-\alpha \textbf{XQQ}^{T} \textbf{X}^{T}\right)^{-1} \textbf{XHF}\quad\quad
	\end{array}
\end{equation*}
where $\textbf{S}=\textbf{XHX}^{T}+\gamma \textbf{D}+\alpha \textbf{X} \textbf{X}^{T}$. According to Woodbury Matrix Identity~\cite{woodbury1950inverting}, we have:
\begin{equation*}
	\begin{array}{l}
		\left(\textbf{S}-\alpha \textbf{XQQ}^{T} \textbf{X}^{T}\right)^{-1} 
		=\left[\gamma \textbf{D}+\textbf{X}\left(\textbf{H}+\alpha\textbf{I}-\alpha\textbf{QQ}^{T}\textbf{X}^{T}\right)\right]^{-1} \\
		= \frac{1}{\gamma} \textbf{D}^{-1}\!\!\!-\!\!\frac{1}{\gamma^2}\textbf{D}^{-1} \textbf{X}\left[\left(\textbf{H}\!+\!\alpha \textbf{I}\!-\!\alpha \textbf{Q} \textbf{Q}^{T}\right)^{-1}\!\!\!\!+\!\!\frac{1}{\gamma} \textbf{X}^{T} \textbf{D}^{-1} \textbf{X}\right]^{-1}\!\!\!\!\! \textbf{X}^{T} \textbf{D}^{-1}
	\end{array}
\end{equation*}
where $\textbf{D}$ is a diagonal matrix with the size of $d\times d$, thus the complexity of computing $\textbf{D}^{-1}$ is $\mathcal{O}\left(d\right)$. Note that $\left(\textbf{H}+\alpha \textbf{I}-\alpha \textbf{Q} \textbf{Q}^{T}\right)\in\mathcal{R}^{n\times n}$, so calculating the inverse of matrix for it requires $\mathcal{O}\left(d^3\right)$.

Therefore, if $n> d$, then we can consume less computing by updating $\textbf{W}$ with $\left(\textbf{S}-\alpha \textbf{XQQ}^{T} \textbf{X}^{T}\right)^{-1} \textbf{XHF}$, which requires $\mathcal{O}\left(nd^2\right)$ (The process of inverting the matrix requires $\mathcal{O}\left(d^3\right)$, and the process of matrix multiplication requires $\mathcal{O}\left(nd^2\right)$, while $nd^2>d^3$). If $n < d$, then we can consume less computing by updating $\textbf{W}$ with the aforementioned method derived from Woodbury Matrix Identity, which requires $\mathcal{O}\left(dn^2\right)$ (The process of inverting the matrix requires $\mathcal{O}\left(n^3\right)$, and the process of matrix multiplication requires $\mathcal{O}\left(dn^2\right)$, while $dn^2>n^3$). 

Finally, we can conclude that the computational complexity of calculating the inverse of matrices for $\textbf{W}$ is $\mathcal{O}\left(nd*\text{min}\left\{n,d\right\}\right)$.

\section{The Proof of Theorem~\ref{theo1}}
\label{apen2}
\noindent \textbf{Theorem~\ref{theo1}}.
The objective function defined in Eq.~(\ref{eq22}) in the main text satisfy $J\left(\textbf{M}_{t}\right) \geq J\left(\textbf{M}_{t+1}\right)$ under the update rule of Eq.~(\ref{eq23}) in the main text.
\begin{proof}
	We derive inequations for each term of Eq.~(\ref{eq22}) in the main text. Following the method in \cite{ding2008convex}, we can set $\textbf{A} \leftarrow \textbf{I}, \textbf{B} \leftarrow \textbf{A}^{+}$ and obtain:
	\begin{equation}
		Tr\left(\textbf{M}^{T} \textbf{A}^{+} \textbf{M}\right) \leq \sum_{i k} \frac{\left(\textbf{A}^{+} \textbf{M}^{\prime}\right)_{i k} \textbf{M}_{i k}^{2}}{\textbf{M}_{i k}^{\prime}}
		\label{eq30}
	\end{equation}
	Then we can get the following inequation according to $2ab \leq\left(a^{2}+b^{2}\right) $:
	\begin{equation}
		\sum_{i k} \textbf{B}_{i k}^{-} \frac{\textbf{M}_{i k}^{2}+\textbf{M}_{i k}^{\prime 2}}{2 \textbf{M}_{i k}^{\prime}}\ge \sum_{i k} \textbf{M}_{i k} \textbf{B}_{i k}^{-}=Tr\left(\textbf{M}^{T} \textbf{B}^{-}\right)
		\label{eq31}
	\end{equation}
	Based on $z \geq 1+\log z,z>0 $, we can obtain:
	\begin{equation}
		\frac{\textbf{M}_{i k}}{\textbf{M}_{i k}^{\prime}} \geq 1+\log \frac{\textbf{M}_{i k}}{\textbf{M}_{i k}^{\prime}}
		\label{eq32}
	\end{equation}
	and
	\begin{equation}
		\frac{\textbf{M}_{i k} \textbf{M}_{i \ell}}{\textbf{M}_{i k}^{\prime} \textbf{M}_{i \ell}^{\prime}} \geq 1+\log \frac{\textbf{M}_{i k} \textbf{M}_{i \ell}}{\textbf{M}_{i k}^{\prime} \textbf{M}_{i \ell}^{\prime}}
		\label{eq33}
	\end{equation}
	From which the following inequation holds:
	\begin{equation}
		Tr\left(\textbf{M}^{T} \textbf{B}^{+}\right)=\sum_{i k} \textbf{B}_{i k}^{+} \textbf{M}_{i k} \geq \sum_{i k} \textbf{B}_{i k}^{+} \textbf{M}_{i k}^{\prime}\left(1+\log \frac{\textbf{M}_{i k}}{\textbf{M}_{i k}^{\prime}}\right)
		\label{eq34}
	\end{equation}
	In a similar way, we have:
	\begin{equation}
		Tr\left(\textbf{MA}^{-} \textbf{M}^{T}\right) \geq \sum_{i k \ell} \textbf{A}_{k \ell}^{-} \textbf{M}_{i k}^{\prime} \textbf{M}_{i \ell}^{\prime}\left(1+\log \frac{\textbf{M}_{i k} \textbf{M}_{i \ell}}{\textbf{M}_{i k}^{\prime} \textbf{M}_{i \ell}^{\prime}}\right)
		\label{eq35}
	\end{equation}
	Since $\textbf{A}^{+},\textbf{A}^{-},\textbf{B}^{+},\textbf{B}^{-}$ are all nonnegative, $J(\textbf{M}) = \mathscr{P}\left(\textbf{M}, \textbf{M}\right) \leq \mathscr{P}\left(\textbf{M}, \textbf{M}^{\prime}\right)$ holds through the comparison of Eq.~(\ref{eq22}) and Eq.~(\ref{eq24}). Therefore, we can have $J\left(\textbf{M}_{t}\right)=\mathscr{P}\left(\textbf{M}_{t}, \textbf{M}_{t}\right)$ and $\mathscr{P}\left(\textbf{M}_{t+1}, \textbf{M}_{t}\right)\geq J\left(\textbf{M}_{t+1}\right)$. Note that $\mathscr{P}\left(\textbf{M}_{t}, \textbf{M}_{t}\right)\geq \mathscr{P}\left(\textbf{M}_{t+1}, \textbf{M}_{t}\right)$ according to Proposition~\ref{prop3}, thus $J\left(\textbf{M}_{t}\right) \geq J\left(\textbf{M}_{t+1}\right)$.
	
\end{proof}

\begin{figure}[t]
	\setlength{\abovecaptionskip}{0pt}
	\centering
	\subfigure[Scene ($\alpha$)]{
		\label{fig:subfig:b} 
		\includegraphics[width=0.48\columnwidth]{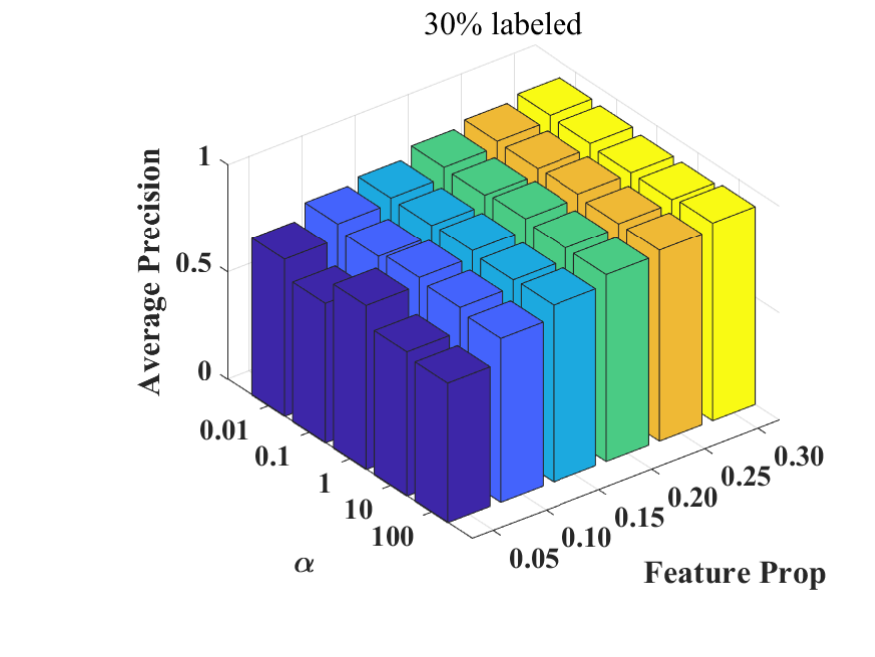}}
        \hspace{-3.1mm}
	\subfigure[Scene ($\beta$)]{
		\label{fig:subfig:a} 
		\includegraphics[width=0.48\columnwidth]{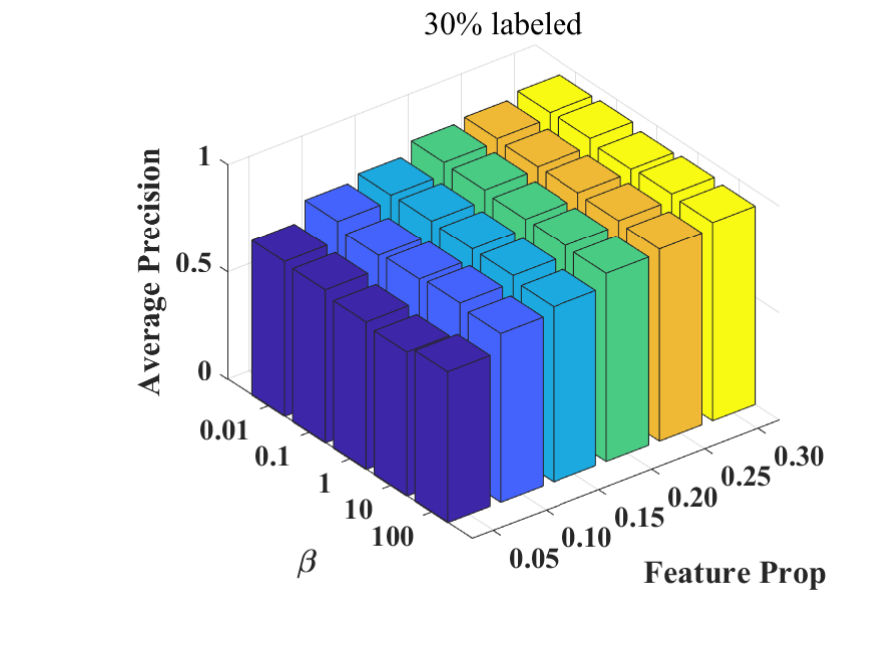}}
		\hspace{-3.1mm}
	\subfigure[Scene ($\gamma$)]{
		\label{fig:subfig:a} 
		\includegraphics[width=0.48\columnwidth]{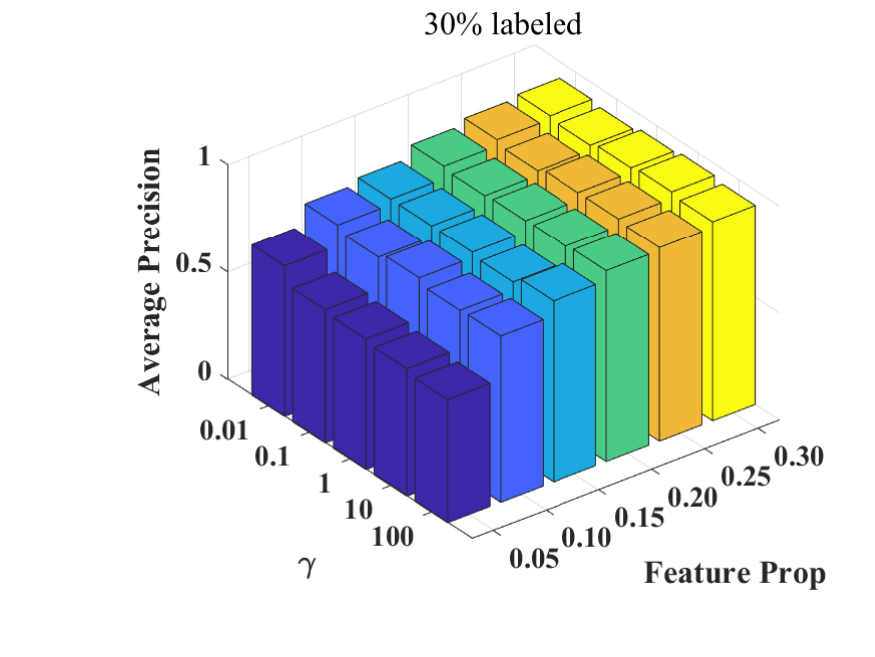}}
        \hspace{-3.1mm}
	\subfigure[Scene ($lsd$)]{
		\label{fig:subfig:a} 
		\includegraphics[width=0.48\columnwidth]{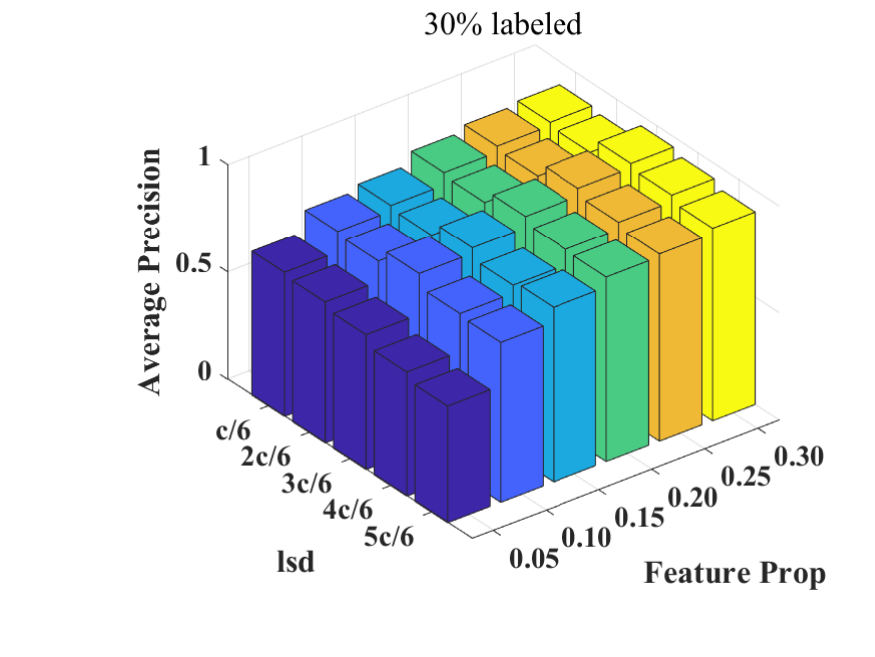}}
	\hspace{-3.1mm}
	\caption{Sensitivity analysis about parameters $\alpha$, $\beta$, $\gamma$ and $lsd$ on dataset Scene (30\% labeled samples).} 
	\label{fig5} 
\end{figure}

\section{The Proof of Theorem~\ref{theo2}}
\label{apen3}
\noindent \textbf{Theorem~\ref{theo2}}.
The objective function of the proposed SGMFS in Eq.~(\ref{eq11}) in the main text decreases monotonically with the optimization process in Algorithm~\ref{Alg1}.

\begin{proof}
	For the sake of convenience, we rewrite the objection function of SGMFS Eq.~(\ref{eq11}) in the main text as:
	\begin{equation}
		\begin{aligned}
			\psi\left(\textbf{M}, \textbf{b}, \textbf{F},\textbf{W},\textbf{Q}\right)=h\left(\textbf{M}, \textbf{b}, \textbf{F},\textbf{W},\textbf{Q}\right)+\gamma\|\textbf{W}\|_{2,1}
		\end{aligned}
		\label{eq39}
	\end{equation}
	where $h\left(\textbf{M}, \textbf{b}, \textbf{F},\textbf{W},\textbf{Q}\right)=\left\|\textbf{X}^{T} \textbf{W}\!+\!\mathbf{1}_{n} \textbf{b}^{T}\!-\!\textbf{F}\right\|_{F}^{2}\!\!+\alpha\|\textbf{X}^{T}\textbf{W}-\textbf{QQ}^{T}\textbf{X}^{T}\textbf{W}\|_{F}^{2}+\beta\|\textbf{MF}-\textbf{F}\|_{F}^{2}+\beta\|\textbf{MQ}-\textbf{Q}\|_{F}^{2}+\gamma\|\textbf{M}\|_{1}$.
	
	According to Eq.~(13) in the main text, we can denote
	\begin{equation}
		\textbf{Q}_{t+1}=\min _{\textbf{Q}^{T}\textbf{Q}=\textbf{I}}\alpha\|\textbf{X}^{T}\textbf{W}_{t}-\textbf{QQ}^{T}\textbf{X}^{T} \textbf{W}_{t}\|_{F}^{2}\!+\beta\|\textbf{M}_{t} \textbf{Q}-\textbf{Q}\|_{F}^{2}
		\label{eq40}
	\end{equation}
	Then we have
	\begin{equation}
		\begin{aligned}
			\psi\left(\textbf{M}_{t}, \textbf{b}_{t}, \textbf{F}_{t},\textbf{W}_{t},\textbf{Q}_{t+1}\right)\leq\psi\left(\textbf{M}_{t}, \textbf{b}_{t}, \textbf{F}_{t},\textbf{W}_{t},\textbf{Q}_{t}\right)
		\end{aligned}
		\label{eq41}
	\end{equation}
	According to fixing $\textbf{Q}=\textbf{Q}_{t+1}$ and denoting that $g\left(\textbf{W}\right)=\left\|\textbf{X}^{T} \textbf{W}\!+\!\mathbf{1}_{n} \textbf{b}^{T}_{t}\!-\!\textbf{F}_{t}\right\|_{F}^{2}\!\!+\!\alpha\|\textbf{X}^{T}\textbf{W}-\textbf{Q}_{t+1} \textbf{Q}^{T}_{t+1}\textbf{X}^{T}\textbf{W}\|_{F}^{2}$, we define
	\begin{equation}
		\textbf{W}_{t+1}=\arg \min _{\textbf{W}} g\left(\textbf{W}_{t}\right)+\alpha Tr\left(\textbf{W}_{t}^{T} \textbf{D}_{t} \textbf{W}_{t}\right)
		\label{eq42}
	\end{equation}
	Then the following inequation holds:
	\begin{equation}
		g\left(\textbf{W}_{t+1}\right)+\alpha Tr\left(\textbf{W}_{t+1}^{T} \textbf{D}_{t} \textbf{W}_{t+1}\right) \leq g\left(\textbf{W}_{t}\right)+\alpha Tr\left(\textbf{W}_{t}^{T} \textbf{D}_{t} \textbf{W}_{t}\right)
		\label{eq43}
	\end{equation}
	Following the same method in \cite{nie2010efficient}, we can obtain:
	\begin{equation}
		\|\textbf{W}_{t+1}\|_{2,1}- Tr\left(\textbf{W}_{t+1}^{T} \textbf{D}_{t} \textbf{W}_{t+1}\right) \leq \|\textbf{W}_{t}\|_{2,1}- Tr\left(\textbf{W}_{t}^{T} \textbf{D}_{t} \textbf{W}_{t}\right)
		\label{eq44}
	\end{equation}
	According to Eq.~(\ref{eq43}) and Eq.~(\ref{eq44}), we have:
	\begin{equation}
		\begin{aligned}
			\psi\left(\textbf{M}_{t},\textbf{ b}_{t}, \textbf{F}_{t},\textbf{W}_{t+1},\textbf{Q}_{t+1}\right)\leq\psi\left(\textbf{M}_{t}, \textbf{b}_{t}, \textbf{F}_{t},\textbf{W}_{t},\textbf{Q}_{t}\right)
		\end{aligned}
		\label{eq45}
	\end{equation}
	Next, we fix $\textbf{W}=\textbf{W}_{t+1}$, and Eq.~(\ref{eq39}) is monotonically decreasing with respected to $\textbf{b}$ according to Eq.~(\ref{eq16}) in the main text, so Eq.~(\ref{eq45}) can be rewritten as
	\begin{equation}
		\begin{aligned}
			\psi\left(\textbf{M}_{t},\textbf{ b}_{t+1}, \textbf{F}_{t},\textbf{W}_{t+1},\textbf{Q}_{t+1}\right)\leq\psi\left(\textbf{M}_{t}, \textbf{b}_{t}, \textbf{F}_{t},\textbf{W}_{t},\textbf{Q}_{t}\right)
		\end{aligned}
		\label{eq46}
	\end{equation}
	Similarly, according to fixing $\textbf{b}=\textbf{b}_{t+1}$ and Eq.~(\ref{eq19}) in the main text, we have
	\begin{equation}
		\begin{aligned}
			\psi\left(\textbf{M}_{t},\textbf{ b}_{t+1}, \textbf{F}_{t+1},\textbf{W}_{t+1},\textbf{Q}_{t+1}\right)\leq\psi\left(\textbf{M}_{t}, \textbf{b}_{t}, \textbf{F}_{t},\textbf{W}_{t},\textbf{Q}_{t}\right)
		\end{aligned}
		\label{eq47}
	\end{equation}
	Eventually, the limiting solution of $M$ updated by rule Eq.~(\ref{eq23}) in the main text is correct since it satisfies the KKT condition~\cite{boyd2004convex} proved in Proposition~\ref{prop4} in the main text. Therefore, we can obtain the final inequation based on $J\left(\textbf{M}_{t}\right) \geq J\left(\textbf{M}_{t+1}\right)$ proved in Theorem~\ref{theo1}.
	\begin{equation}
		\begin{aligned}
			\psi\left(\textbf{M}_{t+1},\textbf{ b}_{t+1}, \textbf{F}_{t+1},\textbf{W}_{t+1},\textbf{Q}_{t+1}\right)\leq\psi\left(\textbf{M}_{t}, \textbf{b}_{t}, \textbf{F}_{t},\textbf{W}_{t},\textbf{Q}_{t}\right)
		\end{aligned}
		\label{eq48}
	\end{equation}
	Therefore, Eq.~(\ref{eq11}) in the main text decreases monotonically with the optimization process in Algorithm~\ref{Alg1} in the main text, and it is easy to know that objective function Eq.~(\ref{eq11}) in the main text has the zero lower bound, so Algorithm~\ref{Alg1} in the main text can reach convergence.
\end{proof}

\section{Additional Experimental Results}
\label{apen4}
Here we show the experimental results of the sensitivity analysis about parameters $\alpha$, $\beta$, $\gamma$ and $lsd$ on dataset Scene (30\% labeled samples in Figure~\ref{fig5}, where the example-based metric $Average$ $Precision$ is used in this experiment.
The larger $Average$ $Precision$ is, the better the corresponding methods are~\cite{zhang2013review}. According to these experimental results, we can obtain the same conclusion illustrated in the main text.

 \bibliographystyle{elsarticle-num} 
 \bibliography{main}





\end{document}